\documentclass{article}

\usepackage[preprint]{neurips_2026}

\usepackage[utf8]{inputenc} 
\usepackage[T1]{fontenc}    
\usepackage{hyperref}       
\usepackage{url}            
\usepackage{booktabs}       
\usepackage{amsfonts}       
\usepackage{nicefrac}       
\usepackage{microtype}      
\usepackage{xcolor}         

\usepackage{subfig}
\usepackage{amsmath}
\usepackage{amsthm}
\usepackage{cleveref}
\usepackage{algorithm}
\usepackage{algorithmic}
\usepackage{graphicx}
\usepackage{multirow}
\usepackage{mathtools}
\usepackage{wrapfig}
\usepackage{enumitem}
\theoremstyle{plain}
\newtheorem{theorem}{Theorem}[section]
\newtheorem{proposition}[theorem]{Proposition}

\newtheorem{corollary}[theorem]{Corollary}
\theoremstyle{definition}
\newtheorem{definition}[theorem]{Definition}

\theoremstyle{remark}
\newtheorem{remark}[theorem]{Remark}

\title{Recovering Physical Dynamics from Discrete Observations via Intrinsic Differential Consistency}

\author{%
  Yuxiang Luo, Andrew Perrault\\
  Department of Computer Science and Engineering\\
  The Ohio State University\\
  \texttt{\{luo.929,perrault.17\}@osu.edu}
}

\begin{document}
\maketitle

\begin{abstract}
  Recovering continuous-time dynamics from discrete observations is difficult because local supervision loses fidelity as the observation interval grows. We replace local supervision with a global structural constraint: any flow representing autonomous dynamics must satisfy the semi-group property under time translation. We train a time-conditioned secant velocity field whose deviation from this property---which we call Symmetry Rupture---serves two roles: as a training regularizer it confines the hypothesis space to flows that compose consistently across temporal scales; as an inference oracle it guides an adaptive solver to select the largest step size that preserves internal consistency, without relying on local truncation error estimates. On the diffusion-reaction and Navier-Stokes benchmarks, our method achieves the lowest rollout RMSE among all evaluated methods using $\leq 3$ function evaluations per step. On the near-conservative shallow water benchmark, our method matches or exceeds competing flow-based methods at far lower computational cost; the strongest Neural ODE baseline achieves lower absolute RMSE but at the expense of $>100$ function evaluations per step. In the direct auto-regressive setting, our solver maintains stable long-horizon rollouts where flow-matching baselines diverge and Neural ODE requires up to $12\times$ more function evaluations. The core contribution is a model that internalizes temporal consistency, removing the need for high-order external solvers to compensate for structural bias.
\end{abstract}

\section{Introduction}
Learning physical dynamics from data constitutes a cornerstone of scientific machine learning (SciML)~\citep{chen2018neural,raissi2019physicsinformed,li2021fourier,helwig2023group,zhang2024sinenet,trenta2026learning}.
While extensive efforts have been dedicated to optimizing network architectures and learning paradigms, a persistent challenge remains: bridging the representational gap between the continuous-time nature of physical dynamics and the real-world discrete-time observations.

This gap manifests in the dichotomy between prevailing approaches.
Discrete-time state regression schemes~\citep{chen2023implicit,zhang2024sinenet,khrabry2025hierarchicalembedding,ghanem2025learning} are bound by fixed time steps, resulting in temporal rigidity that fails to generalize across varying horizons.
Conversely, Neural ODEs~\citep{chen2018neural,huang2020learning,trenta2026learning} can learn dynamics that are numerically plausible but physically inconsistent, due to the absence of structural supervision between observations.
Physics-Informed Neural Networks (PINNs)~\citep{raissi2019physicsinformed} enforce physical compliance through explicit equation-based supervision, but require knowledge of the governing equations, which is often unavailable in practice.

These paradigms are asymptotically equivalent as the sampling interval approaches zero ($\Delta t \to 0$).
However, with the sparse and noisy data typical of real-world settings, this theoretical equivalence breaks down.
The shared limitation is a reliance on local constraints---discrete regression targets, derivative approximations, or point-wise equation residuals---whose effectiveness degrades as the observation interval grows.

Motivated by this observation,
we propose to treat the dynamic evolution not as a sequence of discrete states, but as a continuous flow, and to learn models that recover such flows from discrete data.
We formalize this by constructing a hypothesis space $\mathcal{H}$ of flows satisfying \textbf{time-translation symmetry} (TTS) on an expanded state-temporal manifold ($\psi_{\theta}(s_{t},\Delta t)\to \frac{s_{t+\Delta t}-s_{t}}{\Delta t}$) such that the learned dynamics are consistent across varying temporal scales.
For implementation,
we train a time-conditioned secant velocity field ($\psi_{\theta}(s_{t}, \Delta t)$) with a triangle-consistency loss (Symmetry Rupture $\mathcal{R}_{3}$) that is based on the temporal semi-group property of autonomous system (derived from TTS); and reuse the same loss as an adaptive step-size criterion at inference:

\textbf{1.  Learning Representational Consistency:} We explicitly penalize $\mathcal{R}_3$ during training, such that the hypothesis space $\mathcal{H}$ is constrained to flows that satisfy strict semi-group property. This optimization forces the learned representations across different $\Delta t$ strata to compose a TTS-compatible space---promoting that the learned dynamics remain consistent under sequential composition of flows across varying temporal scales. The Symmetry Rupture combined with the temporal expansion of the velocity field $\psi_{\theta}(s_{t},\Delta t)$ unfolds the geometric structure of the dynamics to a multi-scale bundle of flows, which enables large-step Euler's method for extrapolation while achieving faithful approximation of the true dynamics, hence results in superior efficiency and accuracy compared to baseline models.

\textbf{2. Inference (Consistency-Aware Oracle):}
Existing adaptive solvers dictate step sizes based on \textit{local truncation error}---a metric that assumes the exactness of the governing differential equations and focusing on the extrinsic physical geometry. However, for neural surrogate dynamics, the primary failure mode often stems from representational collapse.
By repurposing our SR as an intrinsic consistency-aware oracle, our solver explicitly monitors the network's internal representational consistency. It dynamically searches for the largest admissible macroscopic step size that safely remains within the inconsistency-suppressed envelope (where the learned representation holds consistent between step sizes). This enables highly efficient long-horizon rollouts while preserving the semi-group property, thereby mitigating the compounding errors characteristic of standard extrapolation.

Our contributions are summarized as follows:
\begin{enumerate}[leftmargin=*, itemsep=0pt, parsep=0pt]
  \item \textbf{Consistency-Based Hypothesis Space:} We propose grounding the learning of physical dynamics on a TTS-guided hypothesis space $\mathcal{H}$, enabling the recovery of continuous dynamics from discrete data without explicit equation supervision.
  \item \textbf{Temporal Expansion for Unbiased Learning and Representational Consistency:} We leverage a temporal expansion of the velocity field $\psi_{\theta}(s_{t},\Delta t)$ to eliminate the systematic bias in tangent velocity approximation and creating the representational capacity for building TTS condition into the model, thus mitigating the conflicting learning objectives between learning the true dynamics and learning a flat surrogate dynamics for extrapolation.
  \item \textbf{Unified Training and Inference Framework:} We show that the same consistency metric serves dual roles: as a training regularization that promotes physically compatible flows, and as a consistency-based error measure for adaptive step-size scheduling during long-horizon inference.
  \item \textbf{Adaptive Consistency Solver:} Using the model's own consistency estimate as a local error indicator, we introduce an adaptive solver that dynamically adjusts integration step sizes, improving both accuracy and efficiency by allocating computation based on the local complexity of the learned dynamics.
\end{enumerate}

By constraining the hypothesis space to TTS-compatible flows, our method incurs a marginal $O(1)$ training overhead yet yields up to a significant reduction in long-horizon prediction error and decrease in function evaluations compared to existing continuous-time baselines. Ablation studies further isolate the contribution of each component---in particular, distinguishing our composition-consistency approach from naive multi-scale regularization.

\section{Hypothesis Space $\mathcal{H}$ and Time-translation Symmetry}
Our method conceives two core principles:
\begin{enumerate}[leftmargin=*, itemsep=0pt, parsep=0pt]
  \item Construct a hypothesis space $\mathcal{H}$ that follows the continuous-time structure of physical dynamics
  \item Leverage discrete observations $\mathcal{D}$ as ground truth to collapse this space towards the true dynamics.
\end{enumerate}
Let $\Phi_{\Delta t}(s)$ be the evolution operator that maps the system state $s$ to its future state after a time duration $\Delta t$, i.e., $\Phi_{\Delta t}(s):\mathcal{M} \to \mathcal{M}$.
We now define $\mathcal{H}$ as the set of all $\Phi$s supporting some natural physical dynamics on $\mathcal{M}$.
Following the principle of constructing a physics-compatible hypothesis space, the purpose of learning shifts from fitting discrete observations to identifying the best-fitting structure ($\Phi^{\star}$) within the hypothesis space $\mathcal{H}$.

\textbf{Time-translation symmetry (TTS)} describes the fundamental property of autonomous physical systems, reflecting that the laws of physics are invariant under shifts in time~\citep{noether1918invariant}.
It is used as a \textbf{geometry-preserving regularizer} that prevents temporal aliasing and enforces multi-scale coherence of the learned flow.
Since $\Phi_{t}$ supports the temporal evolution of such autonomous systems, we discuss the group properties following~\citet{olver1993applications} in~\Cref{theorem:group}, which leads to the generalized time-translation symmetry condition (in short, Symmetry Condition) as~\Cref{eq:groupcondition}:
\begin{equation}
  \forall s\in \mathcal{M}, t_1, t_2 \in \mathbb{R}^{+},\Phi_{(t_{1}+t_{2})}(s) - \Phi_{t_2}(\Phi_{t_1}(s)) = 0
  \label{eq:groupcondition}
\end{equation}
where the semi-group property is explicitly enforced for all autonomous systems, and the inverse map condition is implicitly satisfied for conservative systems but relaxed for dissipative systems.

We define the temporal-evolution-operator-based Symmetry Rupture $\mathcal{R}_{k}^{\Phi}$ over the path-independence residual between the multi-step composed flow and the direct flow prediction:
\begin{equation}
  \label{eq:rupture_error}
  \begin{split}
    \mathcal{R}_{k}^{\Phi} & \coloneqq \underbrace{\Phi_{\frac{\Delta t}{k-1}}\circ\cdots\circ\Phi_{\frac{\Delta t}{k-1}}}_{\text{Integral of Vector Field}} - \underbrace{\Phi_{\Delta t}}_{\text{Direct Flow Prediction}}
  \end{split}
\end{equation}
A Symmetry conditioned flow would yield closure of the loop as $\mathcal{R}_{k}^{\Phi} \equiv 0$.
Therefore, we constrain the hypothesis space $\mathcal{H}$ of our learned model $\psi_{\theta}$ by enforcing this temporal closure:
\begin{equation}
  \label{eq:topo_constraint}
  \mathcal{H} = \left\{ \Phi \mid \|\mathcal{R}_{k}^{\Phi}\|_2 = 0 \right\}
\end{equation}
where the exactness of this constraint is relaxed in practice to allow for optimization and generalization.
This ensures that the learned discrete steps $\{\Phi_{\Delta t}(\cdot)\}$ generate a consistent one-parameter semi-group flow, rather than a disjoint collection of mappings. In practice, this objective is penalized over the training distribution. Consequently, we rely on the model's empirical generalization to maintain consistency across the broader   state-temporal manifold.
We show in \Cref{sec:triangle_consistency_sufficiency} that $k=3$ is sufficient to enforce the Symmetry Condition. Therefore, we formulate the problem of learning continuous dynamics from discrete observations as minimizing Symmetry Rupture of the induced transport along closed temporal loops (particularly over triangular loops $\mathcal{R}_{3}$) to enforce the time-translation symmetry.

\section{Temporal Expansion for Representational Consistency}
\label{sec:cvf}
We follow~\Cref{eq:topo_constraint} to shape the hypothesis space $\mathcal{H}$, however, we show in~\Cref{sec:tangent_velocity_learning_with_nre,sec:learning_secant_velocity_with_nre} that the representational inconsistency between discrete observations and continuous dynamics poses a significant challenge for learning a surrogate model $v_{\theta}$ for tangent velocity field; and the conflicting learning objectives between accurate tangent velocity approximation for the true dynamics and flat surrogate dynamics for extrapolation further exacerbate this challenge.
We resolve this by performing a temporal expansion of the velocity field $\psi_{\theta}(s_{t},\Delta t)$,
\begin{equation}
  \begin{split}
    \psi_{\theta}(s_{t},\Delta t) & \coloneqq \frac{\Phi_{\Delta t}(s_{t})-s_{t}}{\Delta t}
  \end{split}
\end{equation}
where we exploit the ground truth secant velocity field ($v_{\Delta t}$) calculated from discrete observations to provide direct supervision for learning the continuous-time dynamic. With the temporal expansion, the model learns building blocks of the dynamics at multiple temporal scales, which reduces learning of true dynamics as a sub problem ($\Delta t=0$) of learning a consistent cross-temporal-bundle representation, and thus mitigating the conflicting learning objectives and representational inconsistency~(\Cref{sec:tangent_velocity_learning_with_nre}).
Such expansion leads to a reformulation of the Symmetry Rupture $\mathcal{R}_{k}$ over the velocity field and the temporal decompositions ($\Delta t = \sum_{i=1}^{k-1} \Delta t_i$):
\begin{equation}
  \mathcal{R}_{k}(s_{t},\Delta t) =\sum_{i=1}^{k-1} \frac{\Delta t_i}{\Delta t}\psi_{\theta}\left(s_{t}^{i},\Delta t_i\right)-\psi_{\theta}\left(s_{t},\Delta t\right)
  \label{eq:topo_error_cvf}
\end{equation}
where $s_t^{i} = s_t + \sum_{j=1}^{i-1} \Delta t_j\cdot \psi_{\theta}(s_t^{j},\Delta t_j)$. \Cref{appendix:symmetry_rupture_velocity_field,appendix:taylor_expansion} show detail derivation.

The training objective of the velocity field $\psi_{\theta}$ is defined in~\Cref{eq:cvf_loss} to match the secant velocity field ($\psi_{\theta}(s,\Delta t)$) and to minimize the $\mathcal{R}_{3}$ over the randomly sampled decomposition $\Delta t_{i}$s:
\begin{equation}
  \mathcal{J}_{\theta} =  \mathbb{E}\bigg[ \| \psi_{\theta}(s_{t}, \Delta t) - v_{\Delta t} \|_{2}^{2}+ \|\mathcal{R}_{3}(s_t,\Delta t)\|_{2}^{2}\bigg]
  \label{eq:cvf_loss}
\end{equation}
We further show in~\Cref{sec:tangent_supervision_ablation} that our method achieves superior performance compared to methods that learn with spline-based tangent velocity supervision, providing evidence for our structured consistency enforcement as a more accurate and robust supervision signal than the biased tangent velocity from spline approximation.

\textbf{Cascaded Normalization for Stable Training}
Now that the learning process involves processing of original state $s$ and the velocity field $\psi_{\theta}(s,\Delta t)$, we apply a unique cascaded normalization scheme to the velocity field that aligns state and velocity scaling via a pushforward transformation to stabilize the training and inference, which is detailed in~\Cref{appendix:normalization_scheme}. Our ablation studies in~\Cref{fig:rollout_normalization_dtest1} show that this normalization scheme is critical for stabilizing the training and inference, therefore significantly improves the performance of the model.

\subsection{Soft Constraint as a Geometric Probe.}
As rigorously proven via Taylor expansion (\Cref{appendix:taylor_expansion}), the Symmetry Rupture decomposes into two naturally unaligned temporal components: a \textbf{Temporal Representation Mismatch} (Term I) and an \textbf{Induced Convective Perturbation} (Term II). If the representation mismatch were strictly penalized via a hard constraint ($\mathcal{R}_k = 0$), the network would be forced to learn trivial, flat dynamics to strictly satisfy the constraint. Our soft constraint elegantly resolves this by permitting a flexible representation mismatch to capture distinct flow horizons, while utilizing misaligned temporal evaluations (the convective derivative) to explicitly accommodate non-trivial manifold curvature. Crucially, because these two geometric components are plainly summed within the SR objective, the optimizer naturally learns to neutralize them against one another, preserving $\Delta t$-sensitivity while capturing complex geometries.

\subsection{Anti-Aliasing via Cross-Stratum Consistency.}
The $\Delta t$ strata expansion under a continuous TTS constraint acts as a strong regularizer against temporal aliasing. From a Koopman operator perspective~(\Cref{appendix:koopman_aliasing}), any aliased frequency branch ($m \neq 0$) dictates a linearly exploding vector field norm $\|v_\theta\| \propto |m|$, which is heavily penalized by the optimizer's inherent preference for minimum-norm, Lipschitz-smooth solutions (spectral bias). This structural penalization ensures that our framework intrinsically favors the true continuous-time manifold even under challenging sampling regimes.

\section{Representational Consistency-Driven ODE Solver}
\label{sec:adaptive_solver}

One critical difference between equation-based numerical solvers and learning-based neural surrogates is that
the latter is only trained on a finite distribution of time steps $\Delta t$ and thus may not maintain its structural integrity when extrapolating far beyond the training distribution.
The Symmetry Condition~(\Cref{eq:groupcondition}) constructs robust representational blocks within the temporal bounds of $\left[\delta_{min},\mathop{\max}_{\Delta t\sim\mathcal{D}}\left[\Delta t\right]\right]$. Within this temporal envelope, the \textbf{representational inconsistency} (Term I in~\Cref{eq:secant_rupture_decomposed}) is suppressed to create a well-structured region where the model can make large temporal leaps without violating the Symmetry Condition.

Capitalizing on this insight, we repurpose the Symmetry Rupture $\mathcal{R}_{k}$ into a consistency-aware error measure for dynamic step-size scheduling during inference to explicitly monitors the internal integrity of the learned flow.
By dynamically searching for the maximal step size $\Delta t$ that maintains the Symmetry Rupture below a defined tolerance ($\mathcal{R}_{k} \le \epsilon$), the solver balances between inference efficiency via large macroscopic steps and remaining within the representationally consistent regime to mitigate error accumulation.
We follow two principles for the design of the step control mechanism:
\begin{enumerate}[leftmargin=*, itemsep=0pt, parsep=0pt]
  \item \textbf{Minimum Step Size Constraint:} The model cannot meaningfully resolve dynamics below the temporal resolution of its training data, denoted as $\delta_{min}$. The step control mechanism should not reduce the time step below $\delta_{min}$.
  \item \textbf{Dynamic Relative Tolerance:} A static tolerance is ill-suited for the generalization of neural operators. Instead, the error tolerance should adapt to the current operational scale of the solver to ensure meaningful error control across different scales of dynamics.
\end{enumerate}
A fundamental shift from spatial error control to temporal error control is proposed and leads to the step control mechanism (\Cref{eq:step_size_update}) that assumes a second-order error convergence ($p=2$) with respect to the step size:
\begin{equation}
  \tau_{new} = \max\bigg(\delta_{min}, \sqrt{\frac{\delta_{min}\times t_{curr}}{\hat{\mathcal{R}}(s_{t},t_{curr})}}\bigg)
  \label{eq:step_size_update}
\end{equation}
where $\hat{\mathcal{R}}(s_t, t_{curr})$ is the normalized symmetry error~(\Cref{sec:inference_with_normalized_symmetry_error}) acting as a dimensionless measure of representational inconsistency.
The novel step control principles and \Cref{alg:grcs} are shown in \Cref{appendix:step_control_mechanism}.
With the proposed learning objective and adaptive solver, we formally define the Consistent Vector Flow (CVF) model as our main method that adopts the Symmetry Condition as a training regularization and the Symmetry Rupture as a representational inconsistency measure for step-size control during inference.
\section{Evaluations}
For evaluations, we compare our proposed CVF model with previous state-of-the-art methods and on public datasets used in~\citet{takamoto2022pdebench}.

\paragraph{Datasets and Evaluation Settings} The datasets used in the evaluations include: 2D diffusion reaction (DR), 2D Shallow Water (SW), and 2D incompressible Navier-Stokes equation in turbulence (CFD Turb.). We additionally include 2D wave equation (WE) that is strictly conservative in our ablation studies to evaluate the effect of semi-group and full-group constraints in dissipative and conservative systems.
We present the tasks and datasets details in~\Cref{appendix:pdebench_datasets}.

The evaluation metrics includes the root mean squared error of rollout ($\mathcal{L}_\tau$) predictions and the average number of function evaluations (NFE) during rollout predictions to measure the computational efficiency. \Cref{appendix:evaluation_metrics} provides the detailed definitions of these metrics.

We include two different inference modes in evaluation of the baselines. Time-Informed Inference~(\Cref{appendix:time_informed_inference}) setting is widely used in literature; the other setting, Direct Auto-regressive Inference, requires the model to predict a distant future frame without intermediate temporal information (~\Cref{appendix:direct_autoregressive_inference}).

\textbf{Baseline Models}
We consider below baseline models that implement various learning paradigms, with details discussion in~\Cref{appendix:baseline_implementations}:
\begin{enumerate}[leftmargin=*, itemsep=0pt, parsep=0pt,topsep=0em,partopsep=0em]
  \item Local Derivative Matching: Forward Euler Matching (SM)~\citep{liu2025learning}; Spline-Smoothed Gradient Matching (TM)~\citep{hou2025cfo}.
  \item Trajectory Unrolling: Auto-regressive $k$-step Pushforward (PF-$k$)~\citep{brandstetter2022message}.
  \item Integration Matching: Adjoint-based Trajectory Matching (N-ODE)~\citep{chen2018neural}.
  \item Generative Paradigm: Continuous-Time Flow Matching (CFM)~\citep{lipman2023flow}.
\end{enumerate}

\subsection{Time-Informed Setting Analysis}
\label{sec:sota_comparison}
\begin{figure}[htpb]
  \centering
  \subfloat[DR]{\includegraphics[width=.33\linewidth]{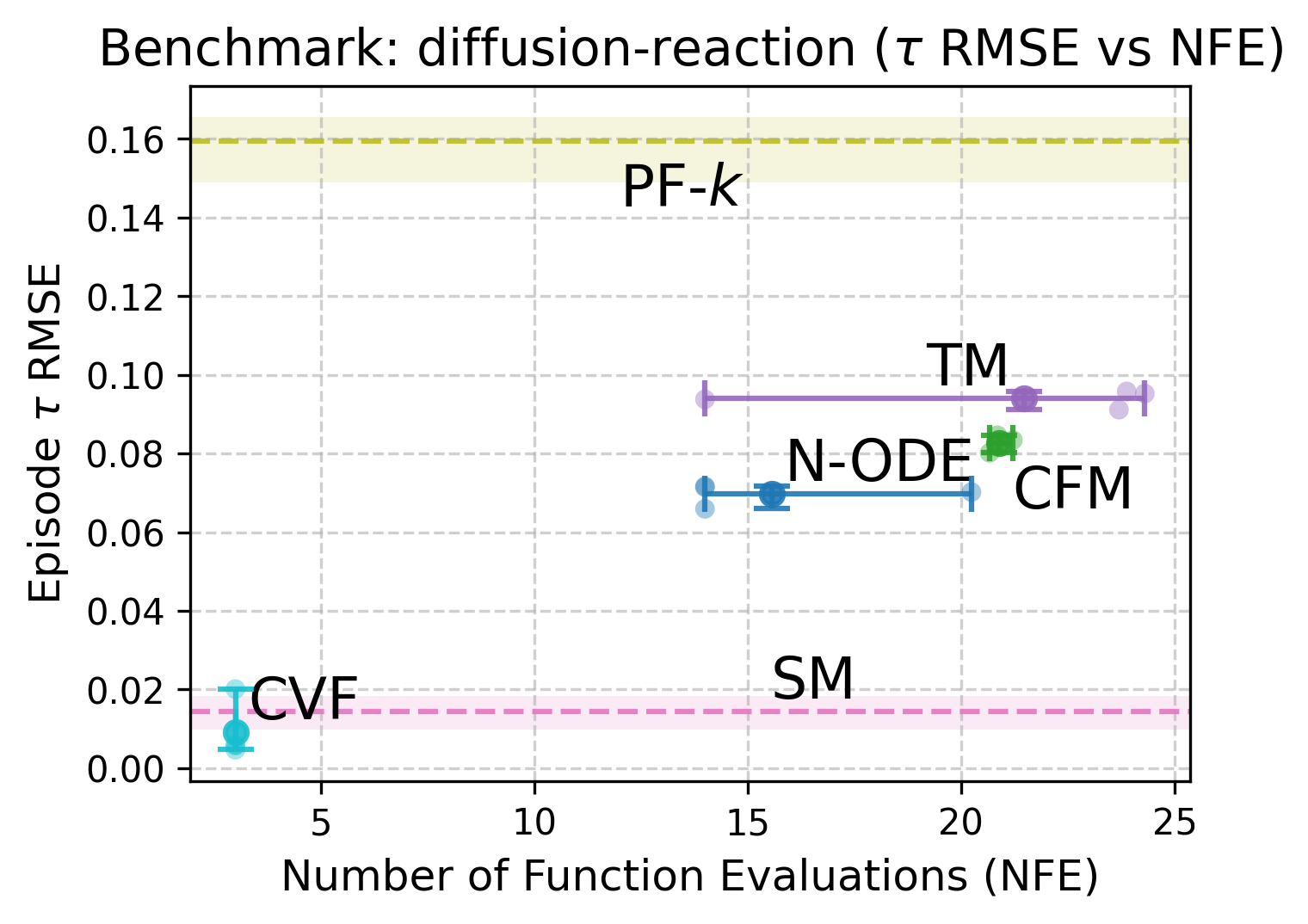}}
  \subfloat[SW]{\includegraphics[width=.33\linewidth]{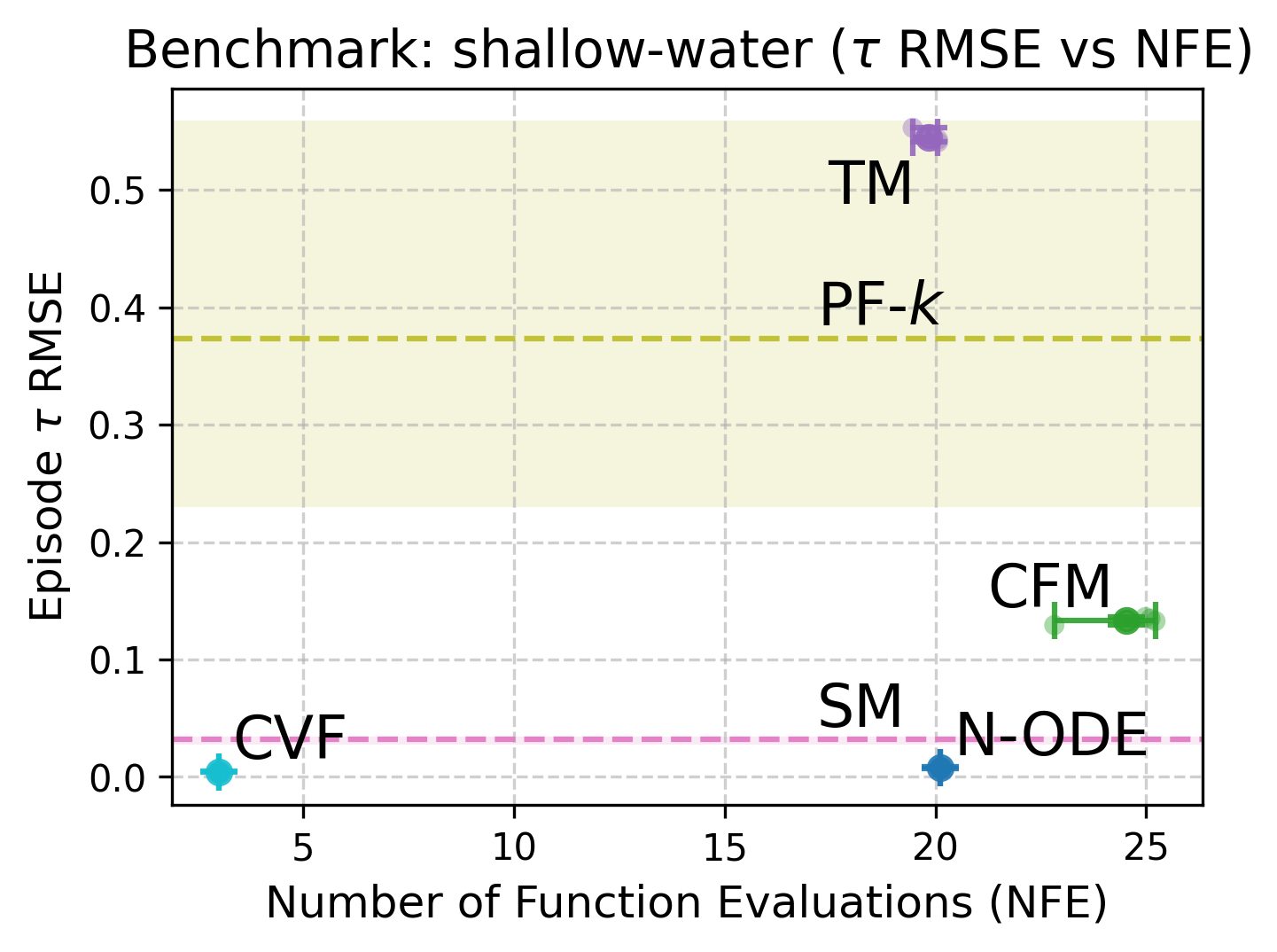}}
  \subfloat[CFD Turb.]{\includegraphics[width=.33\linewidth]{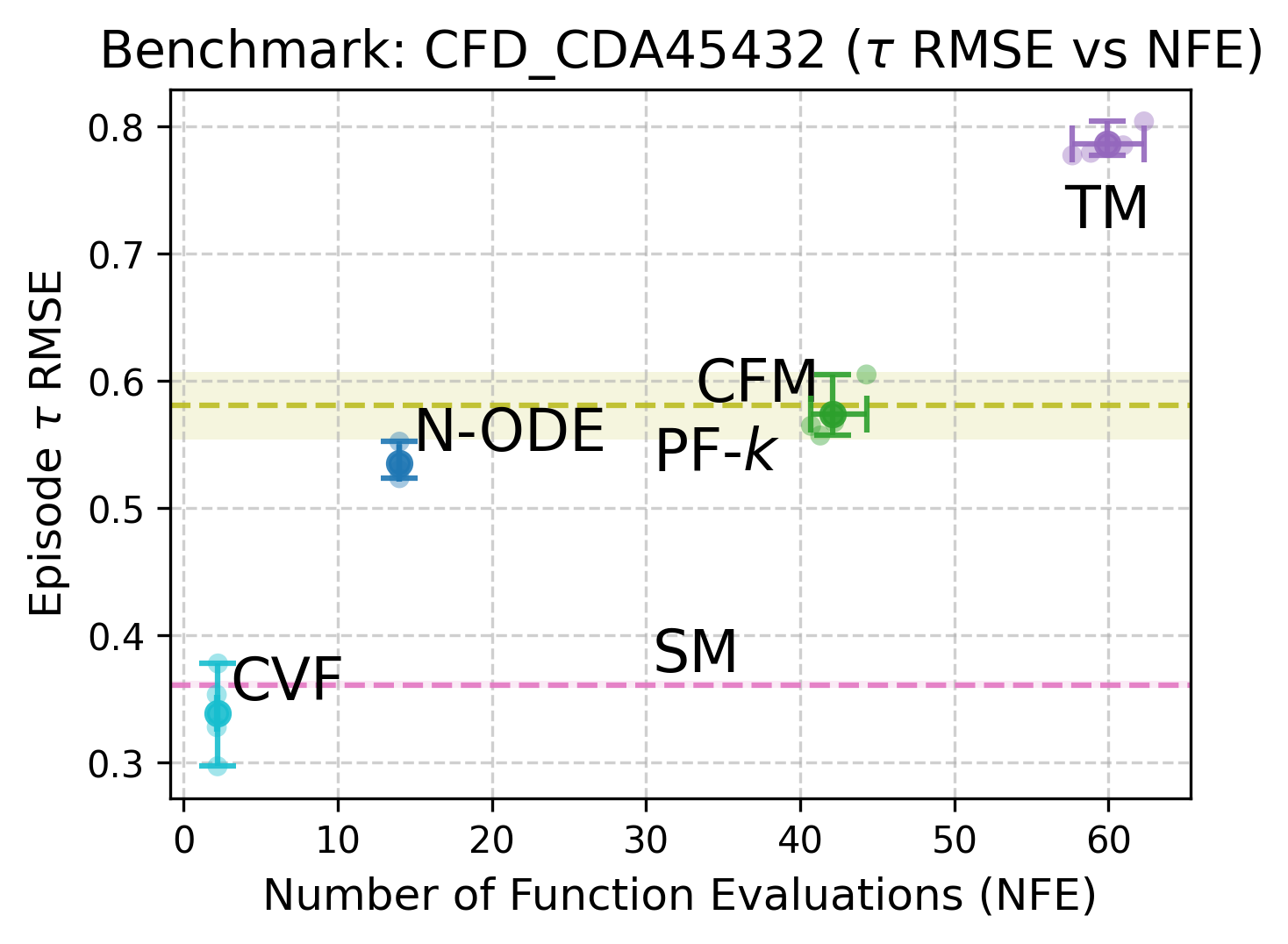}}
  \caption{Accuracy-Efficiency Trade-off on the Time-Informed Inference setting.
    The x-axis is the average NFE during rollout, and the y-axis is the rollout error $\mathcal{L}_{\tau}$. Detail matrices are in \Cref{tab:sota_comparison_pin_1}.
  }
  \label{fig:accuracy_efficiency_tradeoff}
\end{figure}

\textbf{CVF: Accuracy \& Efficiency}, CVF achieves substantially superior long-term accuracy and efficiency compared to all baselines.
Comparing to the CFM and TM baselines, CVF achieving an order of magnitude reduction in RMSE error ($\mathcal{L}_{\tau}$ of $0.009$ vs $0.083$ and $0.094$) in the DR dataset~(\cref{fig:rmse_nfE_scatter_rollout_dr,tab:sota_comparison_pin_1}). The efficiency gain of CVF is also significant--the NFE of CVF is $\leq3$ across all datasets, which signals a rejection-free, single-step rollout behavior in this time-informed setting, while the NFE of CFM and TM are $\geq20$ across all datasets, indicating a heavy reliance on high-order solvers to mitigate their structural bias and numerical inconsistency.

\textbf{Exploiting Training Priors:}
The \textbf{SM and N-ODE} models achieve the second lowest RMSE on the DR (SM vs CVF: $0.014/0.009$) and SW (N-ODE vs CVF: $0.008/0.004$) datasets, their success here is largely attributed to exploiting the training priors:
SM model exploits the temporal grid of data sampling rate ($\Delta t = t_{i+1}-t_{i}$) that perfectly aligns with its training prior, allowing it to perform naive one-step predictions with low error on the DR dataset; but for the SW dataset, SM's Euler state-regression strategy is vulnerable to highly non-linear dynamics.
N-ODE model exploits auto-regressive rollouts during training (RK4 with $\Delta t=0.1\times (t_{k+1}-t_{k})$), which implicitly regularizes the model to be more robust to compounding errors during inference, thus achieving the second lowest RMSE on the SW dataset. But the trade-off between accuracy and efficiency is notable--the NFE of N-ODE increases from $15.6$ in DR to $20.1$ in SW, indicating a trend of increasing reliance on high-order solvers to mitigate the compounding errors in more complex dynamics.
\subsection{Direct Auto-regressive Analysis}
\label{sec:direct_autoregressive_analysis}
\Cref{fig:rollout_accuracy_efficiency_tradeoff,tab:sota_comparison_rollout_1} illustrates the accuracy-efficiency trade-off under the Direct Auto-regressive Inference setting.
\begin{figure}[htpb]
  \centering
  \subfloat[DR]{\includegraphics[width=.33\linewidth]{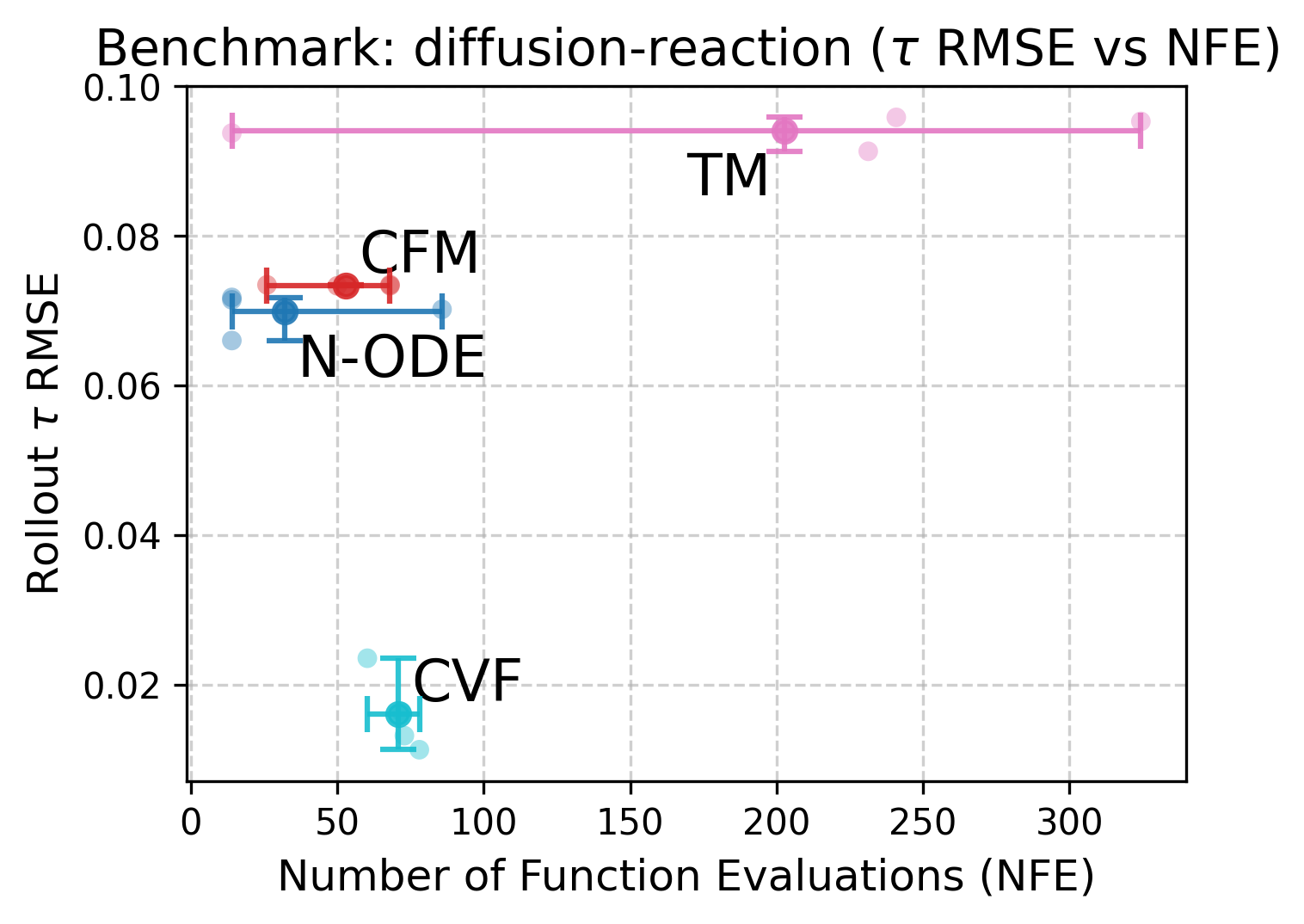}\label{fig:rmse_nfE_scatter_rollout_dr}}
  \subfloat[SW]{\includegraphics[width=.33\linewidth]{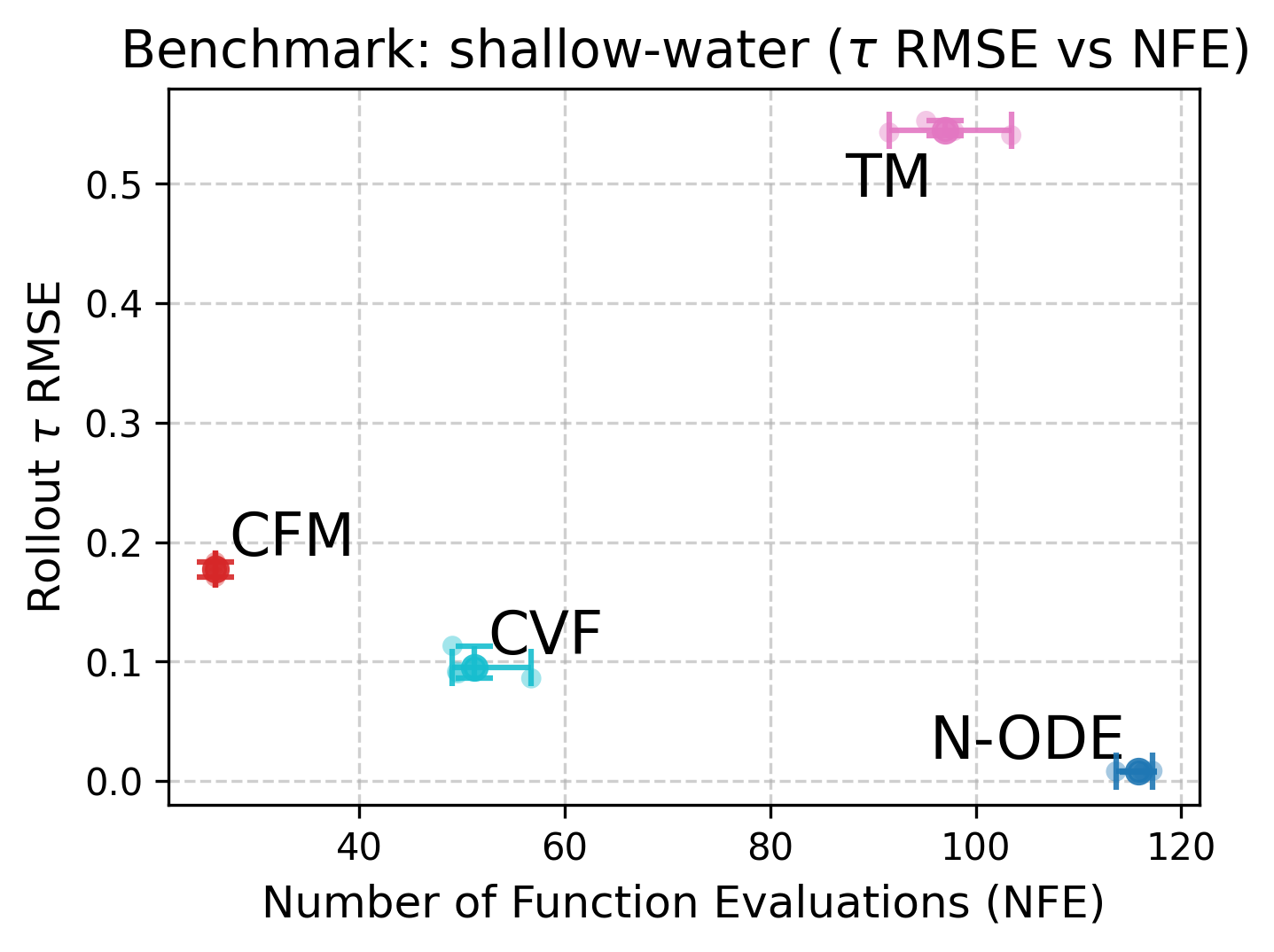}\label{fig:rmse_nfE_scatter_rollout_sw}}
  \subfloat[CFD Turb.]{\includegraphics[width=.33\linewidth]{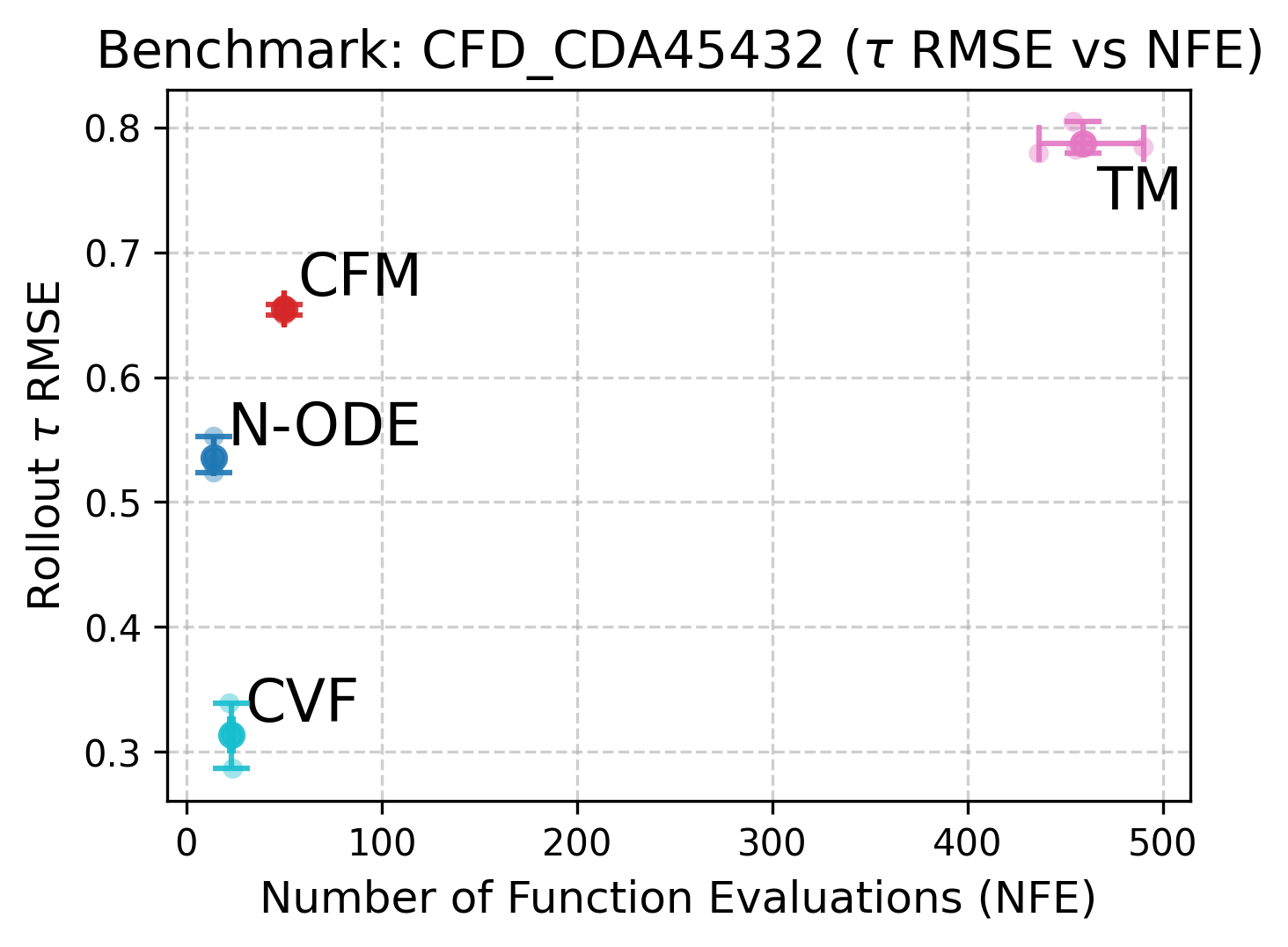}}
  \caption{Accuracy-Efficiency Trade-off on the Direct Auto-regressive Inference setting.}
  \label{fig:rollout_accuracy_efficiency_tradeoff}
\end{figure}
Deprived of the dataset's fine-grain temporal grid, models are forced to navigate the extended integration horizons ($\Delta t = t_{k}-t_{1}$) relying entirely on their own step-size scheduling capabilities.

Under this highly demanding setting, CVF consistently maintains the best accuracy-efficiency trade-off by autonomously dividing the large time span, dynamically scaling its NFE based on the local geometric complexity of the manifold rather than prescribed time-step schedules, and tightly locking the rollouts within its physically accurate trajectory attractor.

\textbf{False Confidence and Pathological Landscapes: N-ODE's Failure Beyond SW.}
The performance of N-ODE is notably better than CVF on the SW dataset under the Direct Auto-regressive setting, yet this comes at an exorbitant computational cost: its NFE skyrockets from $20.1$ in the time-informed setting to $115.9$. This indicates a profound reliance on high-order solvers and the training prior to forcefully mitigate compounding errors during long-horizon extrapolation. \Cref{appendix:n_ode_vs_cvf} provides a detailed analysis of this phenomenon.

\textbf{Structural Limitation of Unidirectional SR on Near-Conservative Dynamics.}
The performance gap on the Shallow Water (SW) dataset under the Direct Auto-regressive setting highlights a structural limitation of unidirectional SR. In dissipative systems, local errors are absorbed by the attractor geometry, allowing our framework to operate efficiently. In near-conservative systems, however, any frequency bias $\varepsilon$ induces a linearly growing rollout error. Because any rotation satisfies the forward semi-group identity, unidirectional SR is structurally blind to frequency bias ($\mathcal{R}_3 \approx 0$ even if $\hat{\omega} \neq \omega^{\star}$). The GCS solver diagnoses this frequency-induced representation mismatch as a numerical truncation error and repeatedly cuts the step size, resulting in the elevated NFE (51.2) observed on the SW dataset. A detailed structural analysis and the connection to Koopman aliasing theory is provided in~\Cref{appendix:n_ode_vs_cvf,appendix:sw_frequency_bias}.

The robust superiority of CVF is fundamentally a two-fold success:

\textbf{Learning Structurally Compatible Dynamics}: Methods like CFM and TM often learn biased or incomplete representations of the underlying dynamics that deviate from the true physical phase space.
Consequently, applying heavy, high-order solvers like dopri5 merely integrates these erroneous vector fields more precisely—resulting in systematically incorrect trajectory rollouts, as evidenced by their high RMSE regardless of high NFE (TM) and low NFE (CFM, benefiting from the Optimal Transport regularization).
By contrast, CVF avoids this by promoting the base vector field to be geometrically consistent with valid physical flows before the solver even steps in.

\textbf{Overcoming Learned Stiffness: N-ODE vs. CVF.}
CVF substantially mitigates the learned stiffness by explicitly embedding semi-group properties, executing massive macroscopic steps safely and preserving the true acceleration purpose of a neural surrogate. This strongly supports our core thesis: a sophisticated solver cannot rescue a structurally flawed flow. Without the rigorous semi-group property constraints naturally embedded in CVF, models construct confidently wrong phase spaces, whereas CVF exclusively restricts learning to the correct physical geometry, enabling both low NFE and state-of-the-art accuracy.

\section{Ablation Study}
\label{sec:tri_ablation}
We conduct ablation studies by varying the downsampling strategy ($k$,~\Cref{subsec:ablation_downsampling}), integral solver \& normalization schema~(\Cref{subsec:ablation_solver_normalization}), and dissipativity-semi-group interaction~(\Cref{subsec:ablation_semi_group}) on the DR dataset.
\subsection{Training Dynamics and Downsampling Strategies.}
\label{subsec:ablation_downsampling}

\textbf{Core Takeaway:} Uniform downsampling is provably beneficial for the model to learn a consistent geometric prior. Furthermore, training with larger uniform intervals (e.g., $k=-16$) equips the model to confidently take larger macroscopic steps, effectively breaking the trade-off between computational cost (NFE) and error accumulation.
We break down these dynamics into three key observations:

\begin{wrapfigure}[14]{r}{0.6\textwidth}
  \centering
  \subfloat[Time-Informed]{\includegraphics[width=.5\linewidth]{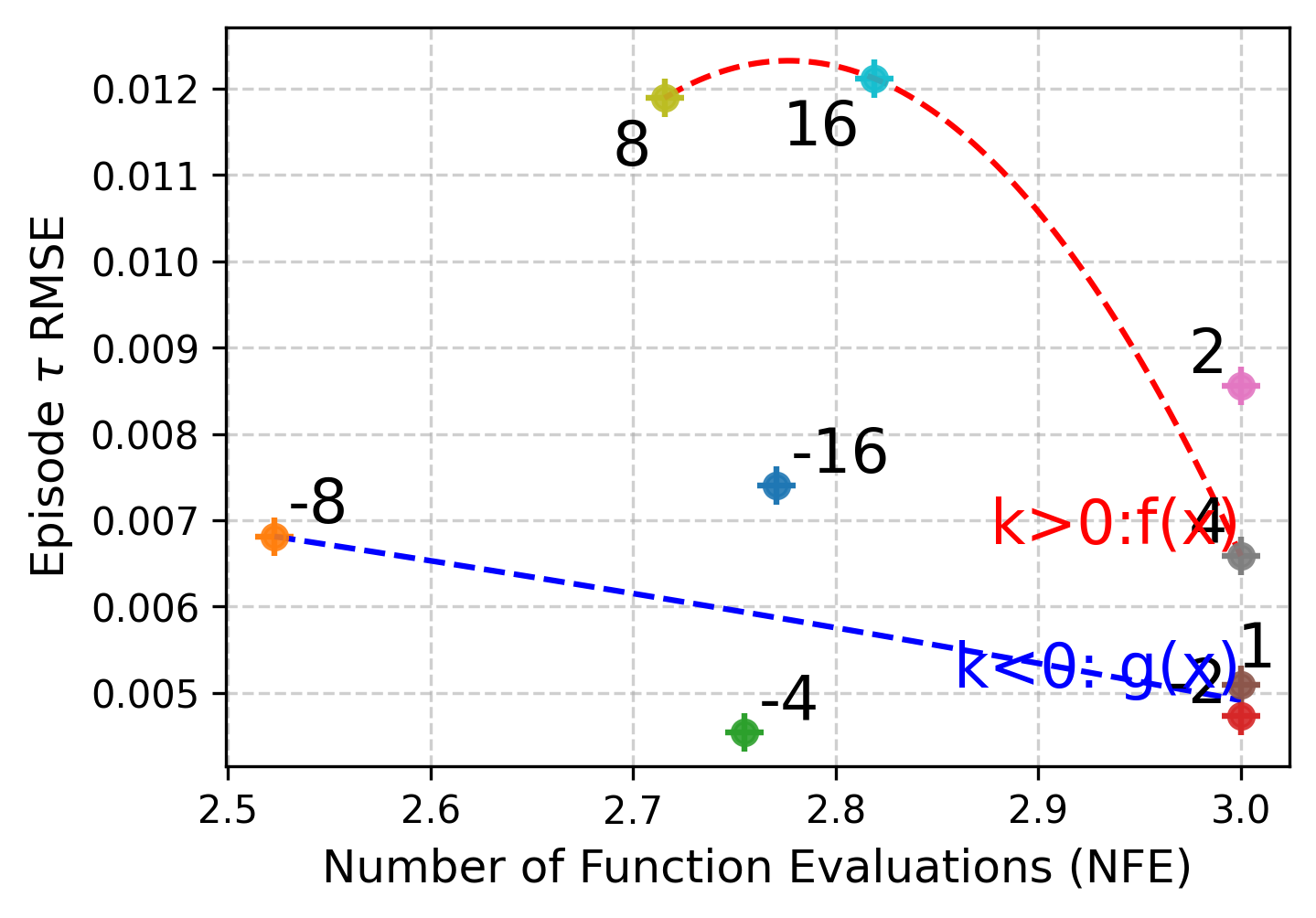}}
  \subfloat[Direct Auto-regressive]{\includegraphics[width=.5\linewidth]{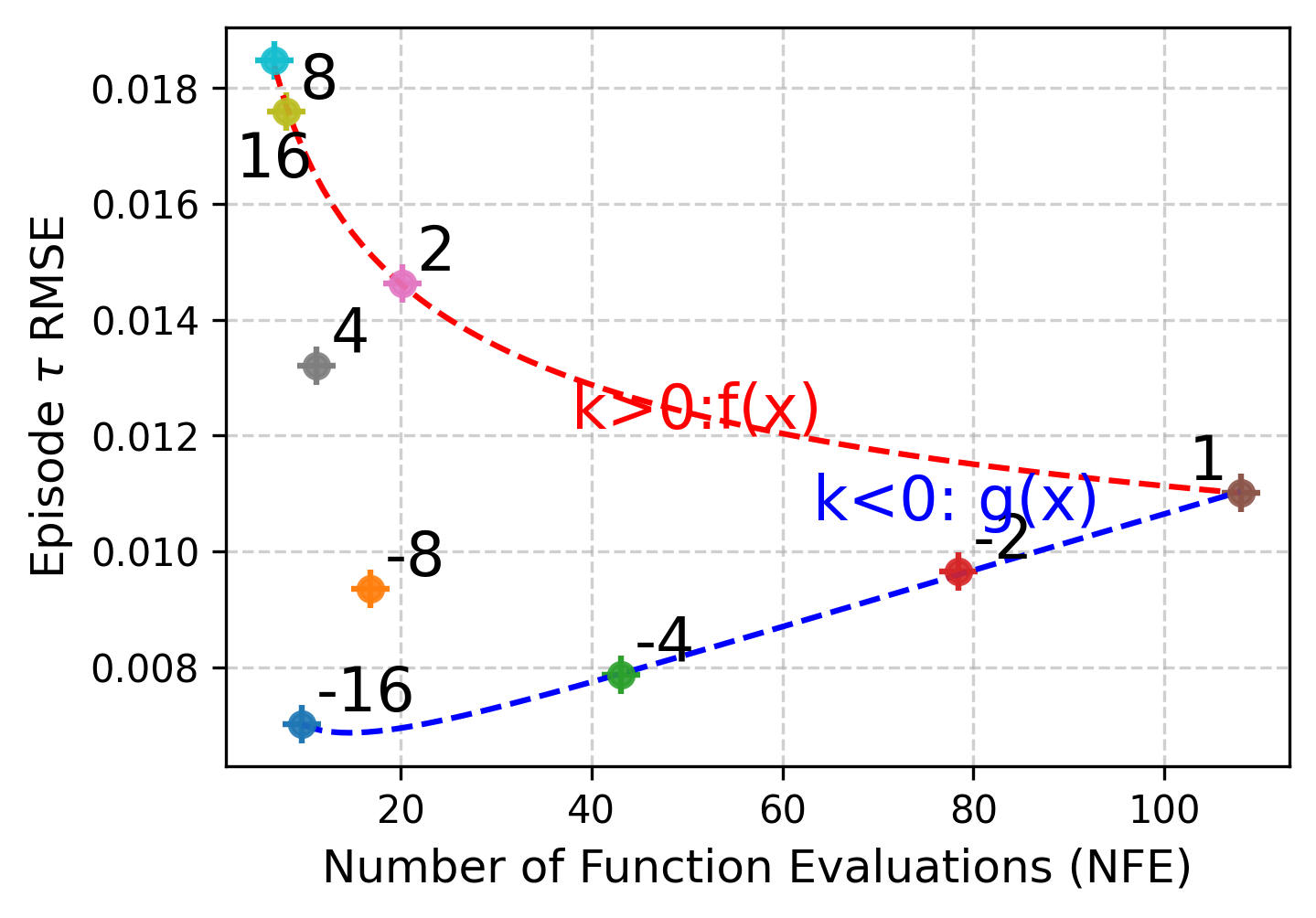}\label{fig:ablation_dynamic_downsampling_rollout}}
  \caption{\small Ablation Study on Uniform Downsampling ($k<0$~\Cref{appendix:even_downsampling}) and Random ($k>0$,~\Cref{appendix:dynamic_downsampling}) Downsampling during Training.}
  \label{fig:ablation_dynamic_downsampling}
\end{wrapfigure}
\textbf{1. The Failure of Dynamic Downsampling ($k>0$):}

Across both evaluation settings, models trained with naive random downsampling exhibit severe overconfidence. They perform almost no step cuts (resulting in low NFE) yet suffer from significantly higher rollout errors compared to their uniformly downsampled counterparts. This starkly distinguishes our explicitly structured framework from simple multi-scale regularization techniques, which mistakenly assume that mere exposure to a diverse range of time steps is sufficient to learn temporal consistency.

The observed performance degradation under random downsampling highlights a profound optimization pathology. Even though the step size $\Delta t$ is explicitly provided as a conditioning input to the network $\psi_{\theta}(s_t,\Delta t)$, the highly stochastic nature of random intervals subjects the model to severely conflicting gradient signals. To navigate this chaotic loss landscape, the network resorts to a degenerate solution: it undergoes \textit{temporal feature collapse}. Rather than establishing robust \textbf{representational consistency} on the expanded temporal bundles, the model renders its output largely insensitive to the temporal scale.
Therefore, it collapses the rich, $\Delta t$-dependent trajectory bundle into a single, overly smooth, linearized vector field (i.e., forcing $\psi_\theta(s_t,\Delta t)$ to be proportional to $\Delta t$).

Crucially, this temporal insensitivity perfectly explains the phenomenon of ``overconfident extrapolation" during inference. Because the predicted vector field behaves strictly linearly with respect to $\Delta t$, the representation inconsistency (Term I,~\Cref{appendix:rupture_theory}) is artificially masked. As a result, the Symmetry Rupture estimator ($\hat{\mathcal{R}}_3$) naturally evaluates to near zero, falsely suggesting to the solver that the learned representations are structurally consistent and the local geometry is flat. Deceived by this false indicator of \textbf{representational integrity}, the adaptive solver waives necessary step cuts and executes massive macro-steps. These steps become completely untethered from the true non-linear dynamics on the base manifold, ultimately compounding into increased rollout errors.

\textbf{2. Constrained Rollout: The $\delta_{min}$ Boundary:}
The time-informed setting reveals an intriguing contrasting behavior. When evaluated on this fine-grained temporal grid, models trained with larger $|k|$ behave more conservatively, resulting in a higher NFE without yielding a proportional drop in rollout error (compared to smaller $|k|$ like $k=-4$). This plateau occurs because the time-informed step duration is already at $\delta_{min}$---the minimum fundamental resolution of the training data. The model cannot reduce its step size beyond this physical limit to further smooth out errors, demonstrating that macro-trained models attempting to reconcile with micro-scale observations will naturally exhaust their adjustment capacity against the data's inherent noise floor.

\textbf{3. Unconstrained Rollout with Macro-Stepping:}
The true advantage of uniform downsampling emerges in the direct auto-regressive setting. Here, a clear linear correlation exists: increasing the training downsampling factor $|k|$ simultaneously reduces both NFE and rollout error. By training on a larger macro-scale (e.g., $k=-16$), the model successfully embeds the long-horizon trajectory into its latent space. Consequently, during inference, it confidently executes large temporal leaps, directly bypassing the compounding local truncation errors that plague smaller-step models, thus mitigating the error accumulation driven by a high number of auto-regressive iterations.
This demonstrates that the model has effectively learned a consistent geometric prior that allows it to break the traditional trade-off between computational cost and error accumulation, if it is not constrained by the fundamental resolution of the data ($\delta_{min}$).
\begin{figure}[htpb]
  \centering
  \subfloat[Normalization Ablation]{\includegraphics[width=.245\linewidth]{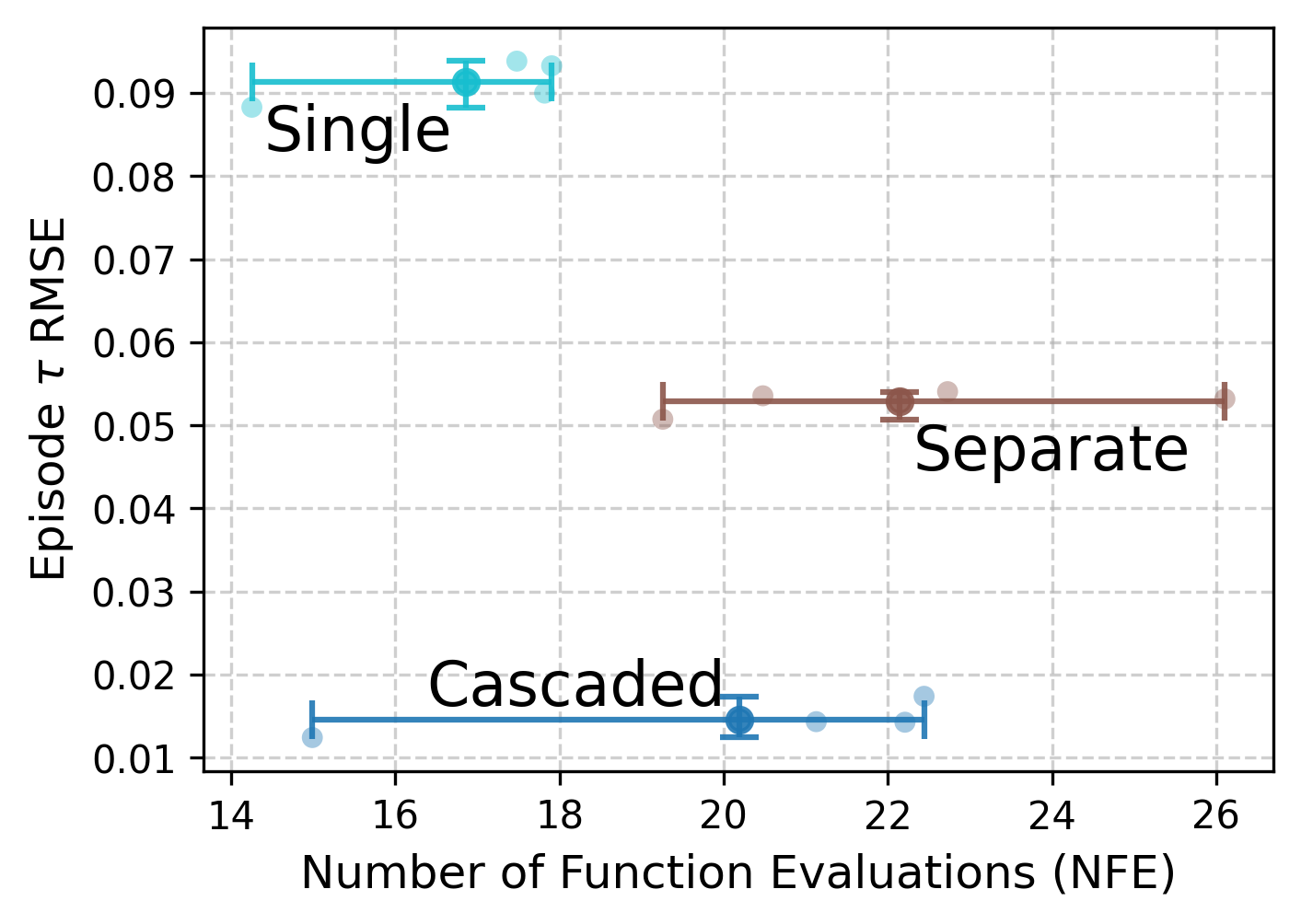}\label{fig:rollout_normalization_dtest1}}
  \subfloat[Solver Ablation]{\includegraphics[width=.245\linewidth]{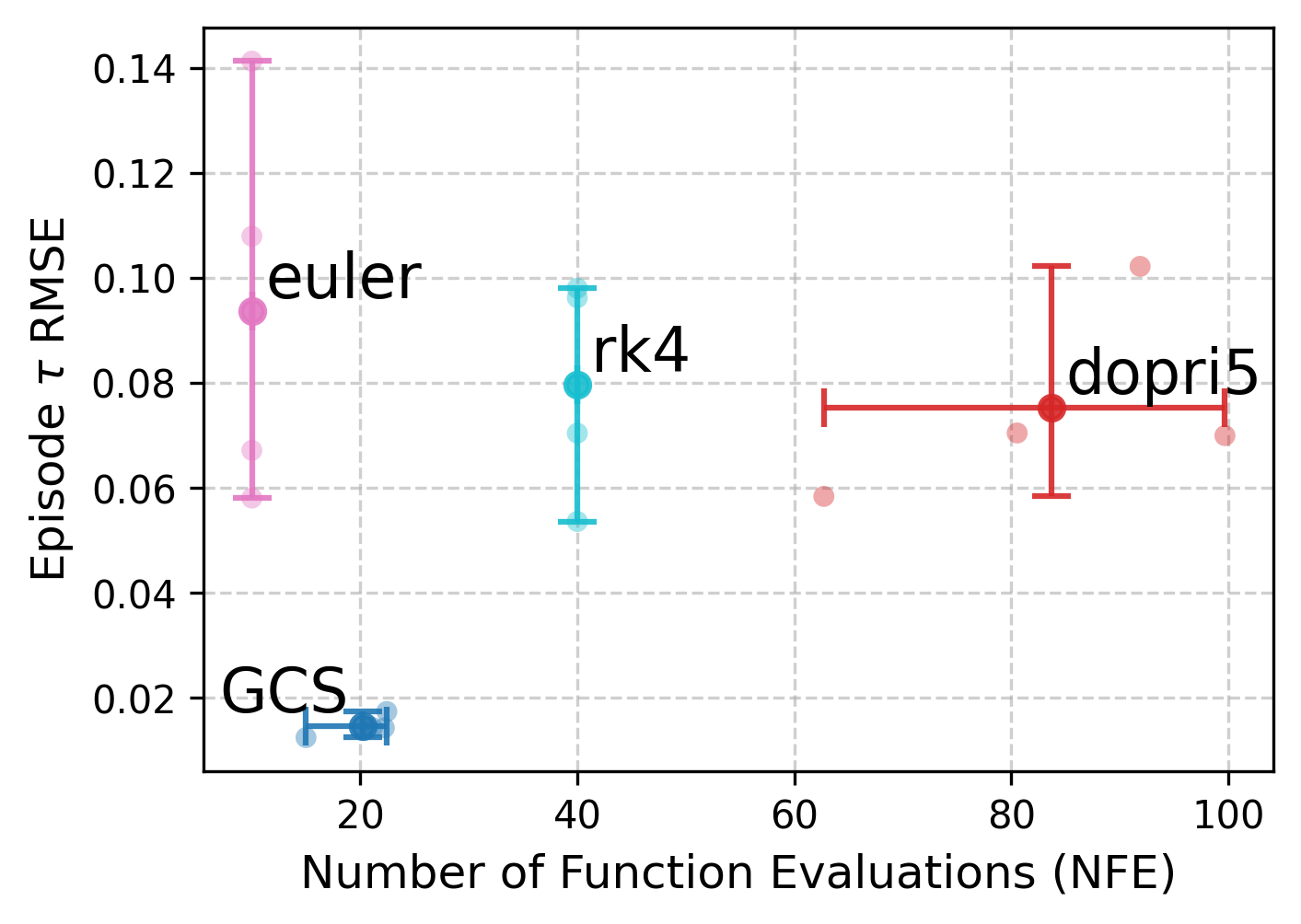}\label{fig:rollout_solver_dtest1}}
  \subfloat[Dissipative System]{\includegraphics[width=.245\linewidth]{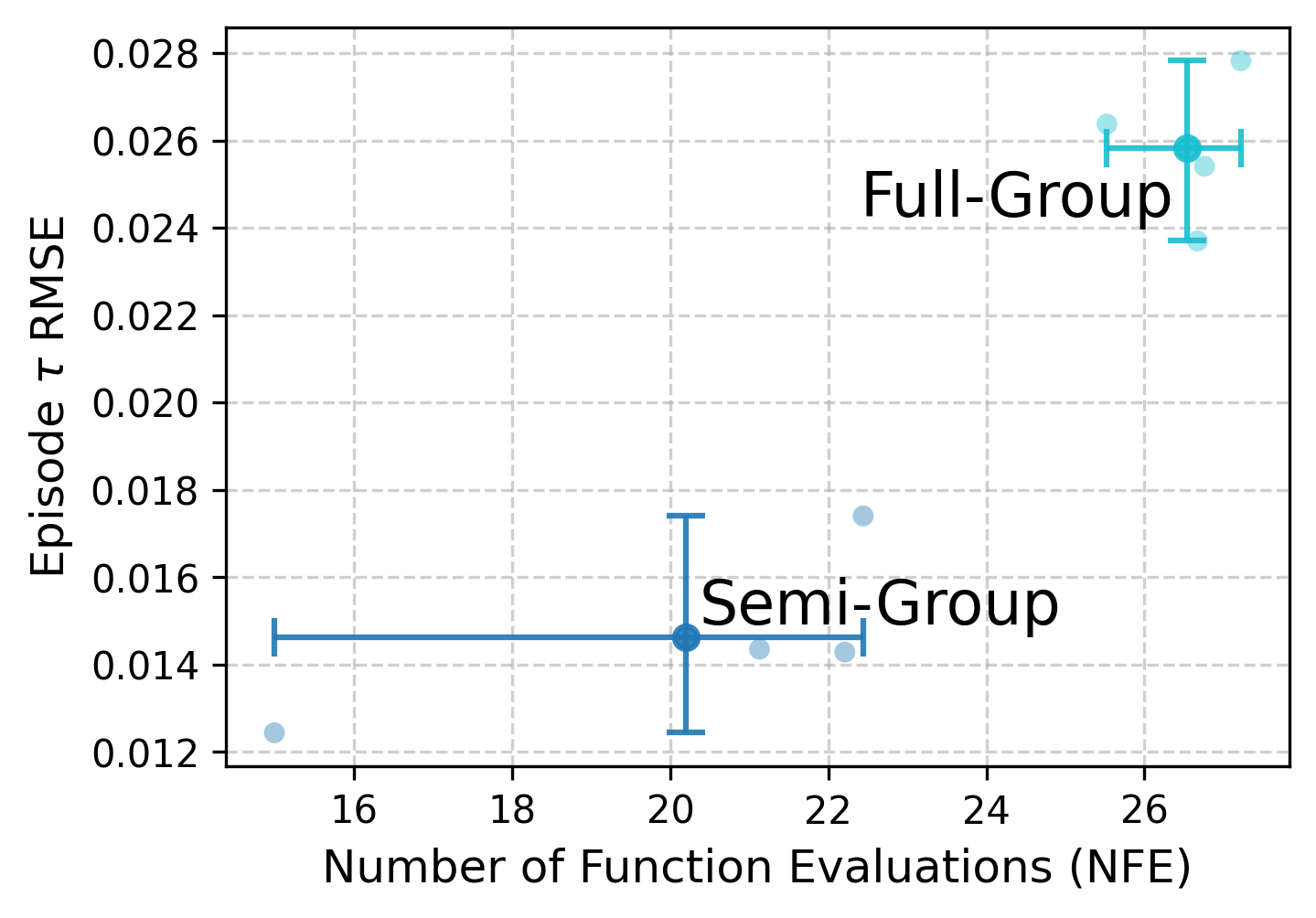}\label{fig:rollout_bidirectional_dtest1}}
  \subfloat[Conservative System]{\includegraphics[width=.245\linewidth]{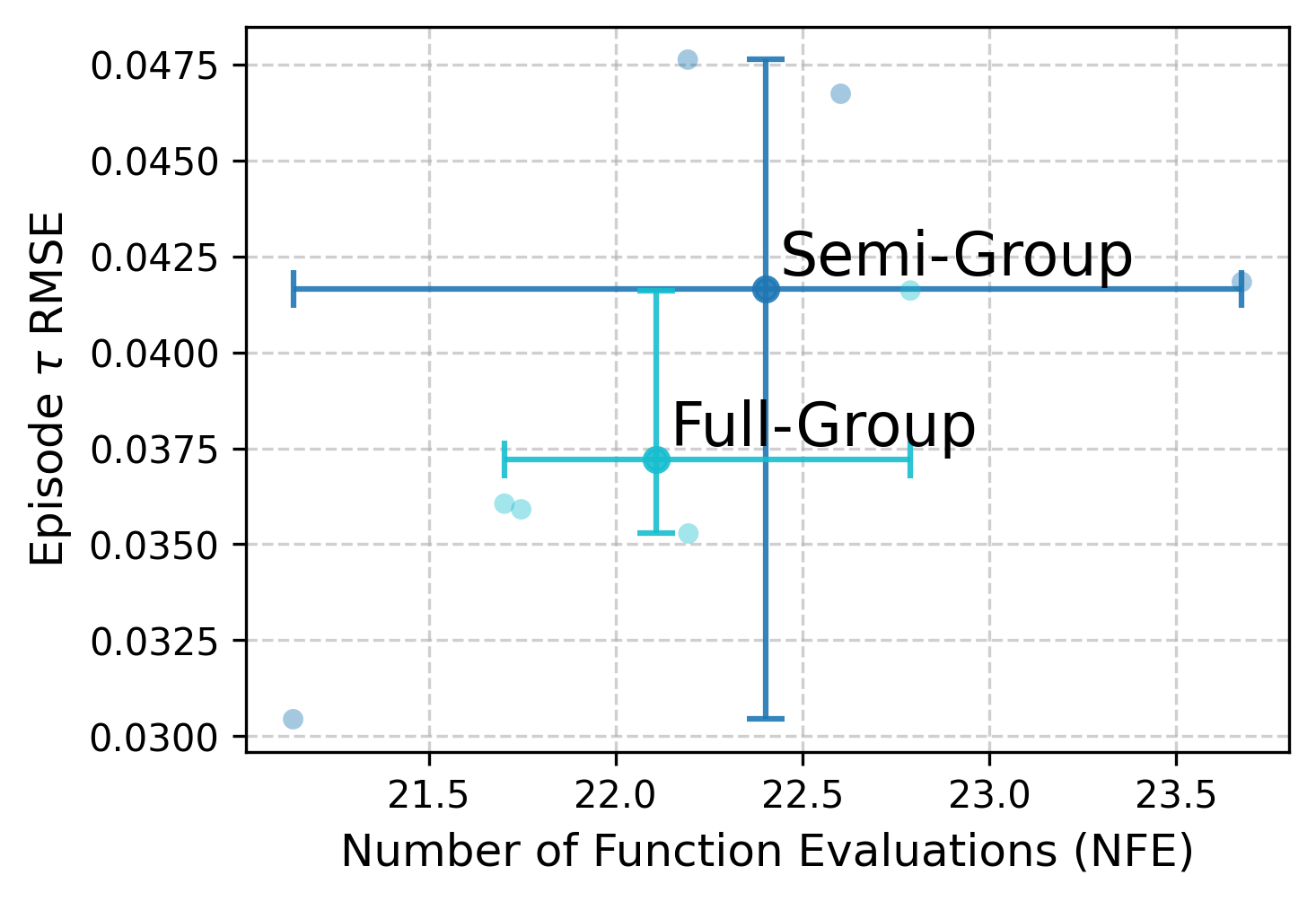}\label{fig:rollout_bidirectional_conservative_dtest1}}
  \caption{Accuracy-Efficiency Trade-off of Border Ablations under the Direct Auto-regressive Inference setting.}
\end{figure}
\subsection{Solver Design and Normalization Scheme}
\label{subsec:ablation_solver_normalization}
Using a uniform downsampling factor of $k=-2$ during training and original step size for the direct auto-regressive setting for testing, we further ablate the solver design and normalization scheme in our framework.
\Cref{fig:rollout_solver_dtest1,fig:rollout_normalization_dtest1} compare the normalization methods used in training, and the different solvers used in inference on the Diffusion-Reaction dataset.
Our approach (Cascaded Normalization + GCS) beats other normalization schemes (Single and Separated normalization) and solvers (Dopri5, RK4, Euler's) in terms of the prediction accuracy. As the cascaded normalization helps in stable learning and GCS helps in faithfully integrating the learned vector field, their contributions are complementary and both are necessary to achieve the best performance.
\subsection{Dissipativity and Semi-Group Property}
\label{subsec:ablation_semi_group}
\Cref{fig:rollout_bidirectional_dtest1,fig:rollout_bidirectional_conservative_dtest1} contrast the effects of enforcing full-group~(\Cref{eq:bi_rupture_loss}) versus semi-group properties~(\Cref{eq:semi_group_rupture}) in dissipative and conservative systems, respectively.
In the dissipative system~(\Cref{fig:rollout_bidirectional_dtest1}), the unidirectional Rupture loss (semi-group~\Cref{eq:semi_group_rupture}) significantly outperforms the bidirectional variant (full-group~\Cref{eq:bi_rupture_loss}).
This confirms that enforcing time-reversibility on an intrinsically irreversible process is detrimental.
In contrast, for the conservative system~(\Cref{fig:rollout_bidirectional_conservative_dtest1}), both settings yield comparable performance.
We attribute this to the capacity of the unidirectional loss to implicitly capture the target's time-reversal symmetry ($-\Phi_{\Delta t}$).
Consequently, while bidirectional consistency mirrors the Hamiltonian nature of conservative dynamics, the forward semi-group property remains practically sufficient.

Above all, these findings highlight the necessity of physics-aware regularization.
Imposing the correct temporal constraint is critical: bidirectional consistency acts as the preferred prior for conservative systems, yet a fatal over-constraint for dissipative ones.
Our framework effectively exemplifies ``inductive bias engineering"-embedding physical domain knowledge directly into the loss landscape—as a promising direction for enhancing model generalization beyond standard architecture or optimization design.
\section{Conclusion}
In this work, we present Consistent Vector Flow (CVF), a novel framework for learning consistent physical dynamics from discrete-time observations.
The CVF framework utilizes TTS as a fundamental inductive bias to learn consistent dynamics; cascaded normalization and a consistency-guided adaptive solver are further introduced to ensure stable training and improve inference accuracy and efficiency.
Our work provides a new perspective on learning consistent dynamics from discrete observations and highlights the importance of correct inductive bias in learning physical systems. Based on the evaluations, we exemplify the critical role of learning structurally compatible dynamics and consistency-guided step-size control in achieving robust long-term predictions, which are often overlooked in existing frameworks.

\textbf{Limitations and Future Work}
While our framework demonstrates significant improvements in learning consistent dynamics, there are several limitations that warrant further investigation. First, our current implementation focuses on autonomous systems, and extending the framework to handle systems with external forcing is an important direction for future work.
Second, while our consistency-guided adaptive solver shows promising results in improving inference efficiency, revoking the reliance on absolute truncation error thresholds causes less inference consistency across time step scales than dopri5 solver.
Finally, the TTS prior is not the only inductive bias that can be leveraged to learn consistent dynamics, and exploring other forms of inductive bias, such as symplectic structure or conservation laws, could be a promising direction for future research in learning physical systems.
\section{Related Works}
\paragraph{Learning Paradigms}

\begin{enumerate}[leftmargin=*, itemsep=0pt, parsep=0pt]
  \item \textbf{Discrete-Time State Regression}
        One common practice in neural surrogate modeling of physical systems learns discrete mapping from the current state to the future state or their residual~\citep{hao2023gnot,li2023geometryinformed,li2020neural,pfaff2021learninga}, assuming that the dynamics operates in a constant time step.
        Despite the two different learning objectives in the discrete-time mapping paradigm, they are essentially equivalent in that the discrete time step $\Delta t$ is assumed constant.
  \item \textbf{Neural ODEs} Neural ordinary differential equation models learn the continuous transformation as a velocity field ordinary differential equations (ODEs) with neural surrogate models~\citep{chen2018neural,sanchez-gonzalez2020learning}.
        The inference process of neural ODEs only relies on external numerical solvers~\citep{dormand2018numerical},
        \citet{chen2018neural,kidger2020neurala,mei2024controlsynth} address the challenges related to such training and inference process~; \citet{norcliffe2020second} application of high-order neural ODEs.
  \item \textbf{Generative Paradigm}
        Recently, generative models~\citep{lipman2023flow,li2025generative,wang2025fourierflow,souveton2025hamiltonian} became the efficient alternative to learn continuous-time flows that follow the optimal transport between the current and future state distributions.
        Works~\citep{zhang2026geoworld,qiao2026lie} have proposed to learn the underlying geometric structure of the dynamics in generative tasks. \citet{song2023consistency} enforce self-consistency along a probability flow to facilitate single-step generation by mapping any point on the trajectory back to its origin, which is different from our approach that enforces semi-group consistency on the learned vector field to ensure the geometric compatibility of the learned dynamics.
\end{enumerate}
Our evaluation operator $\Phi_{\theta}$ that is used in theoretical discussions adopts the same formation as these above paradigms, but our CVF moves beyond this paradigm with the TTS prior, temporal expansion and adaptive solver design.
\citet{xu2025timeseries,matinkia2024learning} also learns a semi-group enforcement from discrete observations in biological systems, with the core idea of multi-scale regularization, which coincidentally shares a similar formulation but is overall a subset of our approach.

\textbf{Auxiliary Works}
Operator learning methods~\citep{li2020neural,lu2021learning,li2021fourier}, encoding methods~\citep{mukhopadhyay2025controllable} and physics-informed neural networks (PINNs)~\citep{raissi2019physicsinformed,lau2024pinnacle} are the orthogonal approaches to the above learning paradigms.

\textbf{Numerical Solvers}
Models that implement fixed-step explicit methods like Euler's method~\citep{lipman2023flow,mei2024controlsynth} extrapolate linearly along the tangent space, inherently ignoring the manifold's curvature and diverging rapidly in non-linear regions.
Traditional adaptive solvers that follow Runge-Kutta methods~\citep{dormand2018numerical} rely on monitoring local truncation errors that
assume a smooth and well-behaved manifold, which violates the neural surrogate modeling setting where the learned velocity field is dominated by the training data distribution and low-level noise as time step deviates from the training time step~(\Cref{appendix:step_control_mechanism}).

\textbf{Post-Hoc Regularization of Lie Symmetries.}
Recent literature has increasingly focused on embedding continuous Lie group symmetries into neural surrogates for dynamic systems (e.g., \citep{liu2026liedynnet,ko2024learning}). A dominant paradigm within this thread treats symmetry discovery as a \textit{posterior} regularization problem. These methods employ generic neural network estimators and attempt to force physical compliance by augmenting the training objective with complex algebraic penalties—such as the infinitesimal invariance condition (IIC) or computationally heavy ``validity scores'' derived from trajectory numerical errors.

This stands in stark contrast to our approach. Rather than relying on heavy regularization to retroactively ``recycle'' valid Lie symmetry pieces from chaotic, unconstrained functional spaces, we directly construct a heavily constrained hypothesis space \textit{a priori}. Because our approach embeds differential consistency and ground truth anchors naturally, it mitigates the computational bottlenecks often encountered in a posterior symmetry discovery, strongly promoting the preservation of the underlying semi-group geometry.

\newpage
\bibliography{neurips2026}
\bibliographystyle{icml2026}

\newpage
\appendix
\section{Notations and Preliminaries}
\label{appendix:notations}

\subsection{Flow, Trajectory, and Manifold}
Let $T_{s}(t)$ be the temporal trajectory of the system state $s$ over time in an unknown system $\mathcal{M}$, i.e., $T_{s}(t): \mathbb{R} \to \mathcal{M}$ where $\mathcal{M}$ can also be viewed as the manifold of all possible system states and their trajectories.
The goal is to recover the TTS compatible dynamics $\mathcal{T}^{\star} = \{ T_{s}^{\star}(t) \mid s \in \mathcal{M} \}$ from a set of discrete observations $\tilde{s} \in \mathcal{D}$ sampled from these trajectories at discrete time points.
Let $\Phi^{\star}$ support the true $\mathcal{T}^{\star}$,
the learning objective is then to identify the optimal $\Phi^{\star}$ within the space $\mathcal{H}$.
\[
  \Phi^{\star} = \arg\min_{\Phi \in \mathcal{H}} \mathbb{E}_{s\sim\mathcal{D}} \left[ \mathcal{L}(\Phi,s) \right]
\]

\begin{table}[htbp]
  \centering
  \begin{tabular}{c|l}
    \hline
    $\mathcal{M}$          & The manifold of all possible system states and their trajectories                                              \\
    $\Phi$                 & The arbitrary evolution operator used in theoretic analysis                                                    \\
    $T_{s}(t)$             & The temporal trajectory of the system state $s$ over time                                                      \\
    $\mathcal{T}^{\star}$  & The true TTS compatible dynamics, i.e., the set of all trajectories $T_{s}^{\star}(t)$ for $s \in \mathcal{M}$ \\
    $\Phi^{\star}$         & The true flow map that supports the true TTS compatible dynamics $\mathcal{T}^{\star}$                         \\
    $\mathcal{H}$          & The hypothesis space of candidate flow maps                                                                    \\
    $\mathcal{D}$          & The dataset of discrete observations sampled from the trajectories in $\mathcal{T}^{\star}$                    \\
    $\mathcal{L}(\Phi,s)$  & The measure of the discrepancy between the flow map $\Phi$ and the true trajectory at state $s$                \\
    $\psi_{\theta}$        & The learned flow map parameterized by neural network parameters $\theta$                                       \\
    $\mathcal{R}_k^{\Phi}$ & The $\Phi$-based Symmetry Rupture used in theoretic analysis                                                   \\
    $\mathcal{R}_k$        & The Symmetry Rupture estimator used for training model                                                         \\
    $\hat{\mathcal{R}}_k$  & The normalized Symmetry error used for step control during inference                                           \\
    $\theta$               & The parameters of the learned flow map $\psi_{\theta}$                                                         \\
    $\mathcal{J}_{\theta}$ & The training objective for learning the flow map $\psi_{\theta}$                                               \\
    $\delta_{min}$         & The minimum fundamental resolution of the training data                                                        \\
    $\Delta t$             & The time step used for training and inference                                                                  \\\hline
    NFE                    & Number of Function Evaluations during inference, a measure of computational cost                               \\
    TTS                    & Time-Translation Symmetry                                                                                      \\
    SR                     & Symmetry Rupture                                                                                               \\\hline
    SM                     & Baseline model trained to learn the secant velocity field                                                      \\
    TM                     & Baseline model trained to learn the tangential velocity field                                                  \\
    N-ODE                  & Baseline model trained with the Neural ODE paradigm                                                            \\
    CFM                    & Baseline model trained with the Continuous Flow Matching paradigm                                              \\
    CVF                    & Our proposed Consistent Vector Flow framework                                                                  \\
    \hline
  \end{tabular}
  \caption{Summary of Notations and Abbreviations}
\end{table}
\section{Time-Translation Symmetry and Group Structure}
\label{appendix:group_structure}

Time-translation symmetry (TTS) is a fundamental property of many physical systems, reflecting the invariance of the system's behavior under shifts in time.
In the context of learning dynamics from discrete observations, TTS implies that the evolution of the system can be described by a flow that is consistent as time progresses, regardless of the specific time points at which observations are made. Note that TTS is a significantly weaker constraint than the PDE-based regularization used in PINNs, as it does not require explicit knowledge of the underlying equations governing the dynamics, but rather requires only that the physical system follows a consistent rule of temporal evolution regardless of when the observations are made.

Mathematically, this manifests as a time-independent ordinary differential equation (ODE) as $\frac{ds}{dt} = v(s)$,
where the rate of change is governed by a time-invariant velocity field $v(s)$.
\begin{theorem}[Group Property of Temporal Evolution]
  \[
    \forall s\in \mathcal{M}, t_1, t_2 \in \mathbb{R}^{*},\begin{cases}
      \Phi_{t_2}(\Phi_{t_1}(s)) =  \Phi_{(t_{1}+t_{2})}(s) \\
      (\Phi_{t})^{-1} = \Phi_{-t}
    \end{cases}
  \]
  \textbf{1. Path Independence}: The composition of flows satisfies the semi-group property, i.e., the flow over a combined time duration is equivalent to the sequential application of flows over individual durations.
  \textbf{2. Inverse Map}: The inverse of $\Phi_{t}$ is given by flowing backward in time.
  \label{theorem:group}
\end{theorem}
\section{Flows and Diffeomorphisms}
\label{appendix:diffeomorphisms}

According to the \textbf{Picard-Lindelöf Theorem} \citep{bensoussan1992ordinary, hastings1992differential}, if the vector field $f(t, s)$ is continuously differentiable ($C^1$) on an open set $D \subset \mathbb{R} \times \mathbb{R}^n$, then for any $(t_0, s_0) \in D$, there exists a unique solution $s(t)$.

\begin{proposition}[Flow as a Diffeomorphism]
  \label{prop:flow_diffeo}
  Let $\Phi_{\Delta t}^t: s(t) \mapsto s(t+\Delta t)$ be the flow generated by the ODE $\frac{ds}{dt} = f(s,dt)$.
  For any finite time interval where solutions are defined, the map $\Phi_{\Delta t}$ is a \textbf{diffeomorphism} (a smooth bijection with a smooth inverse).
  The inverse map is simply given by flowing backwards along the time-reversed vector field: $(\Phi_{\Delta t})^{-1} = \Phi_{-\Delta t}$.
  \begin{proof}
    See \citet{bensoussan1992ordinary}, Chapter 2, Section 2.
    The smooth dependence on initial conditions implies differentiability of $\Phi$, and uniqueness implies bijection.
  \end{proof}
\end{proposition}

\textbf{Relation to Divergence (Liouville's Formula).}
The geometric properties of the flow are governed by the divergence of the field $\nabla \cdot f$:
\begin{equation}
  \frac{d}{dt} \text{Vol}(\Omega_t) = \int_{\Omega_t} (\nabla \cdot f) \, dV
\end{equation}
This allows us to categorize physical systems into two classes relevant to our modeling assumptions:

\begin{itemize}
  \item \textbf{Conservative Systems (Groups):}
        When $\nabla \cdot f = 0$ (e.g., Hamiltonian dynamics), the phase volume is preserved.
        The flow $\Psi_{\Delta t}$ forms a \textit{one-parameter group of diffeomorphisms} defined for all $\Delta t \in \mathbb{R}$.
        The map is strictly reversible, aligning perfectly with the group structure.

  \item \textbf{Dissipative Systems (Semi-groups):}
        When $\nabla \cdot f < 0$ (e.g., Diffusion Reaction), phase volume contracts exponentially.
        Mathematically, the flow forms a \textit{semi-group} ($t \ge 0$) and is not globally a diffeomorphism as $t \to \infty$ due to information loss (merging into attractors)~\citep{mortensen1991infinitedimensional}.

\end{itemize}

As we approach the ideal map of the true underlying dynamics $\Phi_{\Delta t}$ with our learned flow $\psi_{\theta}$, the Lie group structure of diffeomorphisms in the conservative case and the semi-group structure in the dissipative case are both naturally embedded in the learning process, and the proposed Symmetry Rupture serves as a general consistency constraint that promotes the preservation of these geometric properties in the learned dynamics, which is critical for achieving accurate long-term predictions and stable rollouts.

\subsection{Learning Homomorphisms of Flows for Both Conservative and Dissipative Systems}
While the conservative systems and dissipative systems differ in their global geometric properties, the Symmetry Rupture~$\mathcal{R}_{k}$ is compatible with both classes of systems.
For conservative systems, the group property enforces that $-\Phi_{t}=\Phi_{-t}$,  therefore $\mathcal{R}_{k}$ naturally forms a closed loop that measures the Symmetry Rupture.
For dissipative systems, the semi-group property enforces that $\Phi_{t_1+t_2}=\Phi_{t_1}\circ \Phi_{t_2}$ for $t_1,t_2\ge 0$, $\mathcal{R}_{k}$ measures the semi-group inconsistency, which is the only meaningful temporal constraint for dissipative systems.
Therefore, the proposed Symmetry Rupture is a general consistency measurement that can be applied to both conservative and dissipative system, and does not require any modification to accommodate the differences in their geometric properties.
\begin{subequations}
  \begin{align}
    \mathcal{R}_{k}(s) & = \left\| \left( \circ_{i=1}^k \Phi_{\Delta t_i} \right)(s_0) - \Phi_{T}(s_{0}) \right\|_2
    \label{eq:semi_group_rupture}
  \end{align}
  where $s_{0}=s$, $s_{i} = \Phi_{\Delta t_i}(s_{i-1})$.
\end{subequations}
\subsection{Learning Diffeomorphisms for Conservative Systems with Bidirectional Consistency}
However, we explicitly define the bidirectional Rupture loss~$\mathcal{R}_{k}^{bi}$ for conservative systems to further enforce the reverse-time consistency, which is a necessary condition for learning consistent dynamics in conservative systems but not in dissipative systems.
Together, the semi-group-based forward, and the full-group-based bidirectional constraints ($\mathcal{R}_{k}$ and $\mathcal{R}_{k}^{bi}$) form a comprehensive consistency measurement that can effectively regularize the learning of dynamics in both classes of systems.
\begin{subequations}
  \begin{align}
    \mathcal{R}^{\text{bi}}_{k}(s) & = \left\| \left( \circ_{i=1}^{\frac{k}{2}} \Phi_{\Delta t_i} \right)(s_0) - \left( \circ_{i=0}^{\frac{k}{2}} \Phi_{-\Delta t_{k-i}} \right)(s_k) \right\|_2
    \label{eq:bi_rupture_loss}
  \end{align}
\end{subequations}
\section{Reduction: Sufficiency of the Basic Triangle ($k=3$)}
\label{sec:triangle_consistency_sufficiency}

A key insight from geometric integration theory is that any such $k$-polygon error can be decomposed into a sum of local triangle errors via \textit{triangulation} \citep{hairer2006geometric}.
Thus, by strictly enforcing the vanishing of the local 2-step error (Symmetry Condition, $k=3$), we implicitly bound the global error $\mathcal{R}_{k}$ for any $k > 3$.

Directly optimizing \Cref{eq:rupture_error} is computationally intractable due to the infinite sum.
However, we propose that explicitly enforcing consistency on the smallest non-trivial loop ($k=3$) is sufficient to bound the error for any $k$.

\begin{proposition}[Triangle Sufficiency via Composition]
  If the flow map $\Phi$ satisfies the Symmetry Condition (i.e., vanishing rupture for $k=3$) for all $t_1, t_2$:
  \begin{equation}
    \Phi_{t_2} \circ \Phi_{t_1} = \Phi_{t_1 + t_2}
  \end{equation}
  then the flow satisfies the global path independence condition for any finite $k \ge 3$.
\end{proposition}

\begin{proof}
  We proceed by induction on $k$, the number of discrete steps.

  \textbf{Base Case ($k=3$):} The condition is given by hypothesis: $\Phi_{\Delta t_2}(\Phi_{\Delta t_1}(s)) = \Phi_{\Delta t_1 + \Delta t_2}(s)$.

  \textbf{Inductive Step:} Assume consistency holds for any sequence of $k$ steps.
  Consider a sequence of $k+1$ steps with total time $T_{k+1} = \sum_{i=1}^{k+1} \Delta t_i$.
  The composition is:
  \begin{align}
    S_{k+1} & = \Phi_{\Delta t_{k+1}} \circ \left( \Phi_{\Delta t_k} \circ \dots \circ \Phi_{\Delta t_1} \right)                 \\
            & = \Phi_{\Delta t_{k+1}} \circ \left( \Phi_{\sum_{i=1}^k \Delta t_i} \right) \quad \text{(by Inductive Hypothesis)}
  \end{align}
  Now, let $T_k = \sum_{i=1}^k \Delta t_i$.
  We are left with the composition of two terms: $\Phi_{\Delta t_{k+1}} \circ \Phi_{T_k}$.
  Applying the \textbf{triangle} Symmetry Condition (Base Case) to these two terms:
  \begin{equation}
    \Phi_{\Delta t_{k+1}} \circ \Phi_{T_k} = \Phi_{T_k + \Delta t_{k+1}} = \Phi_{T_{k+1}}
  \end{equation}
  Thus, consistency for $k=3$ implies consistency for all $k \in \mathbb{N}$.
\end{proof}

\textbf{Geometric Interpretation (Triangulation).}
This algebraic proof corresponds to the geometric intuition of \textit{Polygon Triangulation}.
Any closed loop formed by $k$ time steps (a $k$-gon) can be decomposed into $k-2$ triangles.

If the curvature vanishes on every fundamental triangle, Stokes' theorem (in the discrete sense) implies that the total curvature over the entire polygon surface must also vanish.
Therefore, constraining the model with the basic $k=3$ loss is structurally sufficient to enforce the global group property required for the $k \to \infty$ geodesic limit.

\subsection{Unified Bidirectional Isomorphism}
Standard approaches to promoting reversibility in dynamic learning typically rely on dual-model frameworks (e.g., CycleGANs using separate forward/backward generators) or architecturally constrained invertible flows.
The former doubles the parameter space and risks learning disjoint dynamics, while the latter sacrifices expressivity.
In contrast, we leverage the Lie group structure~\Cref{eq:bi_rupture_loss} of the temporal flow to unify the forward and backward dynamics within a single model $\psi_{\theta}$.

\paragraph{Isomorphic Consistency at k=3} This unification is most potent when applied to the simplest non-trivial topological loop: the triangle ($k=3$).
By exploiting the group inverse property $\Phi_{-t}=\Phi_{t}^{-1}$, we can structurally fold the loop closure constraint into a single forward pass.
Let the three edges of the temporal triangle be $\Phi_{\Delta t}$ (hypotenuse), $\Phi_{\Delta t_1}$ (partial forward), and $\Phi_{-\Delta t_2}$ (partial backward, where $\Delta t_1+\Delta t_2=\Delta t$).
The closure condition implies that the state reached by evolving $\Delta t_1$ forward from $s_t$ must be identical to the state reached by evolving $\Delta t_2$ \textit{backward} from $s_{t+\Delta t}$.
We operate this by constructing a composite query batch $(S,T)$ that simultaneously samples the forward, backward, and hypotenuse bundles:
\begin{equation}
  \mathbf{S} = [s_t, s_t, s_{t+\Delta t}], \quad \mathbf{T} = [\Delta t_1, -\Delta t_2, \Delta t]
\end{equation}
By feeding this batch into $\psi_{\theta}$, the model is forced to explicitly learn the inverse map logic:
\begin{equation}
  \frac{\Delta t_1}{\Delta t} \psi_{\theta}(s_t, \Delta t_1) \approx \psi_{\theta}(s_t, \Delta t)-\frac{\Delta t_2}{\Delta t}\psi_{\theta}(s_{t+\Delta t}, -\Delta t_2)
\end{equation}
Crucially, this design does not merely ``accelerate" training.
It enforces that the learned operator is a true element of the one-parameter semi-group flow, possessing a valid inverse.
While this bidirectional folding reduces the computational complexity of the consistency check from $\mathbb{O}(k)$ to $\mathbb{O}(\lceil k/2 \rceil)$ for general polygons, it achieves a singular \textbf{computational closure} at $k=3$.
In this special case, the entire topological constraint is resolved in $\mathcal{O}(1)$ (parallel) steps without any recursive rollout.

\section{Unified Tangent-Secant Velocity Fields}

The training objective~(\Cref{eq:cvf_loss_appendix}) of the Consistent Vector Flow (CVF) is to capture the secant velocity field ($\psi_{\theta}(s,\Delta t)$) and the Symmetry Rupture ($\mathcal{R}_{k}$) in a unified framework:
\begin{multline}
  \mathcal{J}_{\theta} =  \mathbb{E}\bigg[ \| \psi_{\theta}(s_{t}, \Delta t) - v_{\Delta t} \|_{2}^{2}+\| r\psi_{\theta}(s_{t}, r\Delta t)+r^{\prime}\psi_{\theta}(s_{t+r\Delta t},r^{\prime}\Delta t) - \psi_{\theta}(s_{t},\Delta t) \|_{2}^{2}\bigg]
  \label{eq:cvf_loss_appendix}
\end{multline}
where $v_{\Delta t}$ is the time-difference secant velocity calculated from discrete-time observations, and $r$ and $r^{\prime}$ are two random scalars satisfying $r+r^{\prime}=1, r\sim \mathcal{U}[0,1]$.

\subsection{Symmetry Rupture in Velocity Field}
\label{appendix:symmetry_rupture_velocity_field}
Learning of tangent and secant velocity fields can find solid ground on the observed data points (knots of the spline interpolation), however, the velocity field in between the observed data points remains unobserved and unconstrained.
One straightforward solution is to sample from the spline trajectory for both tangent and secant velocity supervision, however, it would introduce a strong bias towards the spline-approximated velocity field instead of the true underlying physical dynamics.
The Symmetry Rupture $\mathcal{R}_{k}$ is used in this paper to support the unobserved region of the flow.

Following~\Cref{eq:rupture_error}, it is easy to find the correspondence in velocity fields:
\begin{equation}
  \begin{split}
    \forall [s_{t_1},                 & s_{t_2}, s_{t_3}]  \subset\tau \land t_1\leq t_2 \leq t_3,                                                                                \\
    s_{t_3}                           & =                             s_{t_1} + (s_{t_2}-s_{t_1}) + (s_{t_3}-s_{t_2})                                                             \\
    s_{t_3}-s_{t_1}                   & =                      (s_{t_2}-s_{t_1}) + (s_{t_3}-s_{t_2})                                                                              \\
    \frac{s_{t_3}-s_{t_1}}{t_3 - t_1} & =    \frac{t_2 - t_1}{t_3 - t_1}\frac{s_{t_2}-s_{t_1}}{t_2 - t_1} + \frac{t_3 - t_2}{t_3 - t_1}\frac{s_{t_3}-s_{t_2}}{t_3 - t_2}          \\
    \frac{s_{t_3}-s_{t_1}}{\Delta t}  & =    \frac{\Delta t_1}{\Delta t}\frac{s_{t_2}-s_{t_1}}{t_2 - t_1} + \frac{\Delta t_2}{\Delta t}\frac{s_{t_3}-s_{t_2}}{t_3 - t_2}          \\
    \psi_{\theta}(s_{t}, \Delta t)    & = \frac{\Delta t_1}{\Delta t}\psi_{\theta}(s_{t}, \Delta t_{1}) + \frac{\Delta t_2}{\Delta t}\psi_{\theta}(s_{t+r\Delta t}, \Delta t_{2})
  \end{split}
  \label{eq:state_additivity}
\end{equation}
where $\tau$ is the unknown trajectory of the physical system.
The Symmetry Rupture $\mathcal{R}_{k}$ in velocity field is then defined as the mean squared error between the two sides of \Cref{eq:state_additivity}:
\begin{equation}
  \begin{split}
    \|  r \psi_{\theta}(s_{t}, r\Delta t) + r^{\prime} \psi_{\theta}(s_{t+r\Delta t}, r^{\prime}\Delta t)-\psi_{\theta}(s_{t}, \Delta t)  \|_{2}^{2}
  \end{split}
\end{equation}
\subsection{United Velocity Field as Path Integration}
\label{sec:path_integration_as_cvf}
Ideally, we seek a learned model $\psi_\theta$ that satisfies path independence for any arbitrary discretizations $k \to \infty$.
Taking the limit as $\max(\Delta t_i) \to 0$, the composition of discrete steps converges to the integral of the vector field (the geodesic flow), while the direct term remains the neural operator prediction.
The condition for a consistent physical model is:
\begin{equation}
  \label{eq:geodesic_limit}
  \lim_{k \to \infty} \mathcal{R}_{k}(s; \Delta t) = \frac{1}{\Delta t}\underbrace{\int_0^{\Delta t} v_\theta(s_t) \, dt}_{\text{Geodesic Integral}} - \underbrace{\psi_\theta(s_t, \Delta t)}_{\text{Direct Jump}} \equiv 0
\end{equation}
Satisfying \Cref{eq:geodesic_limit} implies that the learned operator $\psi_\theta$ is the exact exponential map of the learned vector field $v_\theta$.
\section{Tangent-Bundle Consistent Normalization}
\label{appendix:normalization_scheme}

Standard normalization schemes treat state variables $s$ and velocity variables $v$ as independent statistical distributions, applying separate channel-wise scaling (e.g., $s\to \mathcal{N}(\bar{s}),v\to\mathcal{N}(\bar{v})$).
However, from a differential geometric perspective, this independence assumption violates the intrinsic structure of the Tangent Bundle $T\mathcal{M}$.
Velocity $v$ is not a free vector but lives in the tangent space anchored at $s$.
A coordinate transformation on the manifold (normalization of $s$) induces a precise Pushforward transformation on the tangent vectors.
Ignoring this coupling forces the neural network to implicitly learn the scaling ratios between independent normalization statistics, creating a ``conditioning barrier" that hinders convergence.

To resolve this, we propose \textbf{Cascaded Normalization}, a scheme designed to preserve the differential consistency between the base manifold and its tangent space during data pre-processing.

\paragraph{Preserving the Pushforward Map} Let $\Phi:\mathcal{M}\to\tilde{\mathcal{M}}$ be the normalization diffeomorphism acting on the state space, defined as $\tilde{s}=\Phi(s)=\frac{(s-\mu_{s})}{\sigma_{s}}$.
By the chain rule, the induced transformation on the velocity $v=\frac{dt}{ds}$ must be the pushforward $d\Phi$:
\begin{equation} \frac{d\tilde{s}}{dt} = d\Phi \cdot \frac{ds}{dt} = \frac{1}{\sigma_{s}} v.
\end{equation}
Standard independent normalization scales v by $\frac{1}{\sigma_{v}}$, which creates a dimensional mismatch if $\sigma_{v}\neq\sigma_{s}$.
Our Cascaded Normalization respects this geometric linkage by first applying the pushforward scaling ($\frac{1}{\sigma_{s}}$) and then applying a secondary affine standardization for numerical stability:
\begin{equation}
  \begin{split} \tilde{s} & = \frac{s - \mu_{s}}{\sigma_{s}}                                                                                                                                       \\
                \tilde{v} & = \frac{\frac{v}{\sigma_s} - \mu_{v}}{\sigma_{v}}, \quad \tilde{v}_{\Delta t} = \frac{\left(\frac{s(t+\Delta t)-s(t)}{\sigma_s \Delta t}\right) - \mu_{v}}{\sigma_{v}}
  \end{split}
\end{equation}
Here, $\mu_{v}$,$\sigma_{v}$ are the statistics of the pre-scaled velocity $\frac{v}{\sigma_{s}}$.
Crucially, this scheme ensures that the neural velocity field $\psi_{\theta}$ operates in a dimensionless latent space consistent with the state dynamics.
The reconstruction strictly follows the inverse map:
\begin{equation} \hat{s}_{t+\Delta t} = s_{t} + \Delta t \cdot \underbrace{\sigma_{s} \cdot \left[ \sigma_{v} \psi_{\theta}(\tilde{s}, \Delta t) + \mu_{v} \right]}_{\text{Inverse Pushforward} (d\Phi)^{-1}}.
\end{equation}

\paragraph{Analysis of Gradient Consistency} The advantage of this scheme becomes evident when analyzing the derivative limit.
Consider the independent normalization baseline where $\tilde{s}=\frac{(s-\mu_{s})}{\sigma_{s}}$ and $\tilde{v}=\frac{v-\mu_{v}^{\prime}}{\sigma_{v}^{\prime}}$.
The numerical derivative of the normalized state yields:
\begin{equation}
  \lim_{\Delta t \to 0}\frac{\tilde{s}(t+\Delta t) - \tilde{s}(t)}{\Delta t} = \frac{v}{\sigma_s}.
\end{equation} However, the network is trained to output $\tilde{v}=\frac{v}{\sigma_{v}^{\prime}}$.
This forces the network $\psi_{\theta}$ to effectively learn a linear gain factor $k=\frac{\sigma_{s}}{\sigma_{v}^{\prime}}$ to reconcile the physical derivative with the target label.
In contrast, our Cascaded Normalization ensures that the target $v$ is structurally aligned with $\frac{dt}{ds}$ up to the affine shift $\mu_{v}$,$\sigma_{v}$, decoupling the physical scaling $\sigma_{s}$ from the numerical conditioning.

\Cref{fig:rollout_normalization_dtest1} empirically validates that this geometry-aware normalization stabilizes the velocity field distribution, preventing the ``vanishing gradient" effect caused by mismatched scale factors in stiff dynamic systems.

\section{From Semi-Group Property to Symmetry Rupture}
\label{appendix:rupture_theory}
Let $\mathcal{R}_{k}^{\star}$ represent both the evolution-operator-based Symmetry Rupture $\mathcal{R}_{k}^{\Phi}$ and the Symmetry Rupture in the velocity field $\mathcal{R}_{k}$.
In this section, we discuss the geometric interpretation of the Symmetry Rupture~($\mathcal{R}_{k}^{\star}$) based on the semi-group property and the necessary and sufficient conditions for a non-vanishing Symmetry Rupture~($\mathcal{R}_{k}^{\star}> 0$).

We proceed by first defining the notions of the semi-group property and curvature in the context of dynamic system and then demonstrating that trajectory curvature is only a necessary, but not sufficient, condition for a non-vanishing Symmetry Rupture.
By showing the insufficiency of trajectory curvature for a non-zero $\mathcal{R}_{k}^{\star}$, we further establish the true necessary and sufficient condition for a non-vanishing Symmetry Rupture, which is the violation of the representational consistency of the learned flow across different time bundles.
Finally, we incorporate an expanded manifold-temporal space $\mathcal{E}=\mathcal{M}\times \mathbb{R}^{+}$ to show that $\mathcal{R}_{k}^{\star}$ is calculated across various time bundles. The curvature detected in $\mathcal{E}$ is the inconsistency of the representation across these different time bundles, and the vanishing of $\mathcal{R}_{k}^{\star}$ implies the flatness of the learned flow in $\mathcal{E}$, which is a stronger condition than the flatness of the trajectory in $\mathcal{M}$.

\subsection{Definitions: Semi-Group Property and Curvature on the Trajectory Manifold}
\label{sec:rupture_curvature_definitions}

Consider a discrete decomposition of a time horizon $\Delta t$ into a sequence of $k$ intervals $\{ \Delta t_i \}_{i=1}^k$ such that $\sum \Delta t_i = \Delta t$.
The two forms of Symmetry Rupture are shown below:
\begin{equation}
  \begin{split}
    \mathcal{R}_{k}^{\Phi}(s_t; \Delta t) & \coloneqq \left( \circ_{i=1}^k \Phi_{\Delta t_i} \right)(s_t) - \Phi_{\Delta t}(s_t)        \\
    \lim_{k\to\infty}\mathcal{R}_{k}      & \coloneqq \frac{1}{\Delta t}\int_{0}^{\Delta t} v_{\theta}(\cdot) dt - \psi(s_{t},\Delta t)
  \end{split}
\end{equation}
where $\circ$ denotes operator composition.

\begin{definition}[Temporal Semi-Group Property]
  The semi-group property of the flow requires that sequential composition of temporal transitions must be equivalent to the direct transition of equal duration. (i.e., $\Phi_{t_1} \circ \Phi_{t_2} \neq \Phi_{t_1+t_2}$)
\end{definition}
The semi-group property is translated into a vanishing condition of the Symmetry Rupture $\mathcal{R}_{k}^{\star}$.

\begin{corollary}[Symmetry Rupture measures Curvature]
  The evolution-operator-based Symmetry Rupture $\mathcal{R}_{k}^{\Phi}(s; \Delta t)$ serves as a discrete measure of the geodesic curvature induced by the learned dynamics.
  \begin{proof}
    \begin{subequations}
      We show that $\mathcal{R}_{k}^{\Phi}(s_t; \Delta t)$ is equivalent to the covariant derivative of the flow map along the triangle loop, which is a direct measure of curvature.

      Consider the triangle formed by three time steps $r\Delta t, (1-r)\Delta t, \Delta t$ where $r\in (0,1)$.
      $\mathcal{R}_{k}^{\Phi}(s_t; \Delta t)$ around this triangle is:
      \begin{align}
        \mathcal{R}_{k}^{\Phi}(s_t; \Delta t) & = \Phi_{(1-r)\Delta t}(\Phi_{r\Delta t}(s_t)) - \Phi_{\Delta t}(s_t)                     \\
                                              & = s_t+r\Delta t v(s_t) + (1-r)\Delta t v(s_t+r\Delta t v(s_t)) - (s_t + \Delta t v(s_t))
      \end{align}
      We can expand the velocity field in a Taylor series:
      \begin{align}
        v(s_t+r\Delta t v(s_t)) & = v(s_t) + r\Delta t \nabla_{s} v(s_t) v(s_t) + O(\Delta t^2)
      \end{align}
      Substituting back into the Symmetry Rupture expression:
      \begin{align}
        \mathcal{R}_{k}^{\Phi}(s_t; r\Delta t, (1-r)\Delta t) & \approx r(1-r)\Delta t^2 \nabla_{s} v(s_t) v(s_t) + O(\Delta t^3)
      \end{align}
      The term $\nabla_{s} v(s_t) v(s_t)$ is precisely the covariant derivative of the velocity field along the flow, which is a direct measure of the curvature over the local trajectory segment.

      Extending this argument to $\Delta t>0$:
      \begin{align}
        \mathcal{R}_{k}^{\Phi}(s_t; \Delta t) & = \sum_{poly} \mathcal{R}_{k}^{^{\Phi}}(poly)
      \end{align}
      where $\text{poly}$ are the non-overlapping polygon loops that compose the outer $k$-polygon, and $\mathcal{R}_{k}^{\Phi}{\text{poly}_k}$ is the Symmetry Rupture around each polygon loop.
      Because each polygon loop can be further decomposed into triangle loops, we can express the Symmetry Rupture over Riemann integral~($\mathcal{R}_{k}(s_t; \Delta t)$) as a sum of triangle Symmetry Ruptures:
      \begin{align}
        \mathcal{R}_{k}^{\Phi}(s_t; \Delta t) & =\sum_{\Delta \in Triangulation(polygon)}  \mathcal{R}_{k}^{^{\Phi}}(\Delta)
      \end{align}
      where $\Delta$ are the triangle loops that compose the polygon loops.
      Let $\text{C}_{\Delta}$ denote the covariant derivative term for each triangle loop:
      \begin{align}
        \text{C}_{\Delta}                     & = \nabla_{s} v(s_t) v(s_t)                                                \\
        \mathcal{R}_{k}^{\Phi}(\Delta)        & \approx r_{\Delta}(1-r_{\Delta})\Delta t^2 \text{C}_{\Delta}              \\
        \mathcal{R}_{k}^{\Phi}(s_t; \Delta t) & \approx\sum_{\Delta} r_{\Delta}(1-r_{\Delta})\Delta t^2 \text{C}_{\Delta}
      \end{align}
      where $r_{\Delta}$ is the ratio of the triangle loop's time steps to the total time step $\Delta t$. This shows that the Symmetry Rupture measures the accumulated inconsistency effects of the curvature of the flow on the manifold over $\Delta t$ from the initial state $s_t$.
      And minimizing it corresponds to flattening the manifold to achieve path independence.
    \end{subequations}
  \end{proof}
\end{corollary}

\begin{remark}[Connection to Curvature]
  In the language of differential geometry, a non-vanishing Symmetry Rupture around a closed loop implies the existence of a non-zero \textbf{Curvature} in the connection associated with the time-evolution bundle.
  Specifically, the failure of the commutative diagram  can be interpreted as a ``curvature" that prevents the discrete steps from forming a flat, consistent coordinate system on the manifold.
\end{remark}

\subsection{Curvature Flux by Stokes' Theorem}
\label{sec:curvature_flux}
By Stokes' theorem, such Symmetry Rupture can be expressed as the flux of curvature through the area enclosed by the loop, where we denote the unknown curvature flux as $\text{C}_{\Delta}$.
When $k=3, r=0.5$, we have Rup$(s,\{\Delta t_{i}\})$ as $\mathcal{R}_{3}(s_t)$:
\begin{align}
  \mathcal{R}_{3}^{\Phi}(s_t) & \approx \frac{1}{4}\Delta t^{2} C_{\Delta}
\end{align}
This derivation is analogous to the Ambrose-Singer Theorem in differential geometry, which states that the holonomy around a closed loop is directly related to the curvature of the connection. The curvature flux $\text{C}_{\Delta}$, following the same principle, quantifies the distortion of the flow induced by the curvature of the state manifold over a macroscopic time step $\Delta t$.

\subsection{Infeasible Learning of Tangent Velocity with Symmetry Rupture}
\label{sec:tangent_velocity_learning_with_nre}
Replacing the true velocity $v$ with a tangent velocity surrogate $v_{\theta}$ gives the learning objective for training:
\begin{equation}
  \mathcal{J}_{\theta}^{\Phi}
  = \mathbb{E}\bigg[ \|\frac{s_{t+1}-s_{t}}{\Delta t}-v_{\theta}(s_t)\|_2^2 + \|\mathcal{R}_{3}^{\Phi}(s_t,\Delta t)\|_{2}^{2}\bigg]
  = \mathbb{E}\bigg[ \|\frac{s_{t+1}-s_{t}}{\Delta t}-v_{\theta}(s_t)\|_2^2 + \frac{\Delta t^{4}}{16}\|\nabla_{s}v_\theta(s_t) v_\theta(s_t)\|_{2}^{2}\bigg]
\end{equation}

However, there are a three critical caveats to this formulation.

\textbf{First}, the spatial derivative of learned tangent velocity $\nabla_s v_\theta(s_t)$ and velocity field $v_\theta$ are the intrinsic properties of the learned flow that are subjective to the underlying true manifold. Given the fact that we hope to learn a flow that is consistent with the true manifold, learning to match the observed data and minimizing the Symmetry Rupture are two incompatible objectives that cannot be optimized simultaneously.

\textbf{Second}, the estimation of $\nabla_s v_\theta(s_t)$ (captured by $\mathcal{R}_{3}^{\Phi}$) and $v_\theta$ (represented by $\frac{s_{t+1}-s_{t}}{\Delta t}$) are not accurate due to the fundamental mismatch between the discrete-time observations and the continuous-time tangent velocity field learning setting, which imports systematic bias into the estimation both terms in $\mathcal{J}_{\theta}^{\Phi}$ and make it unreliable as a training signal.

\textbf{Third}, the Symmetry Rupture $\hat{\mathcal{R}}_{3}^{\Phi}$ can never be reduced to zero for a curved manifold, which means that the learning process will be trapped in a dilemma of either learning a flow that is consistent with the observed data but has non-zero Symmetry Rupture, or learning a flow that has zero Symmetry Rupture but is inconsistent with the observed data. The former case could diverge the learning process when the Symmetry Rupture becomes the dominant term in the learning process, while the latter case is a trivial solution that does not learn anything about the true dynamics.

\subsection{Learning Secant Velocity With Symmetry Rupture}
\label{sec:learning_secant_velocity_with_nre}
Our method effectively bypasses these issues by directly learning secant velocity field $\psi_{\theta}(s,\Delta t)$:
\begin{equation}
  \begin{split}
    \mathcal{J}_{\theta} & = \mathbb{E}\bigg[ \|\frac{s_{t+1}-s_{t}}{\Delta t},\psi_{\theta}(s_t,\Delta t)\|^2 + \|\mathcal{R}_{3}(s_t,\Delta t)\|_{2}^{2}\bigg]                                                                     \\
                         & = \mathbb{E}\bigg[ \cdots + \|r\psi_{\theta}(s_{t},r\Delta t) + (1-r)\psi_{\theta}\bigg(s_{t}+r\Delta t\psi_{\theta}(s_{t},r\Delta t),(1-r)\Delta t\bigg) - \psi_{\theta}(s_{t},\Delta t)\|_{2}^{2}\bigg]
  \end{split}
\end{equation}

Our method resolves the problems in~\Cref{sec:tangent_velocity_learning_with_nre} by expanding the original manifold space $\mathcal{M}$ to include a temporal dimension $\mathcal{E}=\mathcal{M}\times \mathcal{R}^{+}$, where the secant velocity field $\psi_{\theta}(s_{t},\Delta t)$ is directly learned as a function of both state and time step, and the Symmetry Rupture $\hat{\mathcal{R}}_{3}(s_{t},\Delta t)$ is calculated across different time bundles to measure the representational consistency of the learned flow across these bundles.

By learning the secant velocity field $\psi_{\theta}(s_{t},\Delta t)$, we align the learning objective with the discrete-time observations, eliminating the systematic bias in estimating the velocity field. Expanding to $\mathcal{E}$ contains the learning of true dynamics as a special case when $\Delta t = 0$, while also allowing the model to build TTS across different time bundles, which is a stronger condition than the flatness of the trajectory in $\mathcal{M}$ and can be optimized without conflicting with the data-fitting objective.

We furtherly show in~\Cref{appendix:taylor_expansion} that the expanded space $\mathcal{E}$ brings in representation learning semantics to the learned flow, which transforms the learning of a consistent flow into the learning of a consistent representation across different time bundles, where the curvature detected in $\mathcal{E}$ is the inconsistency of representation across these bundles, and the vanishing of $\hat{\mathcal{R}}_{3}$ implies the flatness of the learned flow in $\mathcal{E}$, which is a stronger and learnable condition than the flatness of the trajectory in $\mathcal{M}$.

\subsection{Decoupling representation inconsistency and Spatial Curvature via Taylor Expansion}
\label{appendix:taylor_expansion}
\begin{subequations}
  Recall the formulation of the Symmetry Rupture over velocity field $\mathcal{R}_{3}$ for a two-step partition parameterized by $r \in (0, 1)$:
  \begin{align}
    \mathcal{R}_{3}(s_{t},\Delta t) = r\psi_{\theta}(s_{t},r\Delta t) + (1-r)\psi_{\theta}\bigg(s_{t}+r\Delta t\psi_{\theta}(s_{t},r\Delta t),(1-r)\Delta t\bigg) - \psi_{\theta}(s_{t},\Delta t)
    \label{eq:rk_original}
  \end{align}

  The nested state transition in the second term obfuscates the analysis. To isolate the effects, we perform a first-order Taylor expansion on the intermediate state transition with respect to the spatial variable $s$. Let the spatial displacement be $\delta s = r\Delta t\psi_{\theta}(s_{t},r\Delta t)$. Expanding the vector field around the anchor state $s_{t}$ yields:
  \begin{align}
    \psi_{\theta}\bigg(s_{t}+\delta s, (1-r)\Delta t\bigg) = \psi_{\theta}\big(s_{t}, (1-r)\Delta t\big) + \nabla_{s} \psi_{\theta}\bigg(s_{t}, (1-r)\Delta t\bigg) \delta s + \mathcal{O}(\|\delta s\|^2)
  \end{align}
  where $\nabla_{s} \psi_{\theta}$ denotes the Jacobian matrix of the velocity field with respect to the state space. Substituting $\delta s$ explicitly into the expansion gives:
  \begin{align}
     & \psi_{\theta}\bigg(s_{t}+r\Delta t\psi_{\theta}(s_{t},r\Delta t), (1-r)\Delta t\bigg)                                                                                         \\
     & = \psi_{\theta}\bigg(s_{t}, (1-r)\Delta t\bigg) + r\Delta t \nabla_{s} \psi_{\theta}\bigg(s_{t}, (1-r)\Delta t\bigg) \psi_{\theta}(s_{t},r\Delta t) + \mathcal{O}(\Delta t^2)
    \label{eq:taylor_step}
  \end{align}

  By plugging Eq.~\ref{eq:taylor_step} back into the original formulation in Eq.~\ref{eq:rk_original}, we can elegantly regroup the Symmetry Rupture into two distinct geometric components:
  \begin{align}
    \mathcal{R}_{3}(s_{t},\Delta t) & = \underbrace{\Big[ r\psi_{\theta}(s_{t},r\Delta t) + (1-r)\psi_{\theta}\bigg(s_{t}, (1-r)\Delta t\bigg) - \psi_{\theta}(s_{t},\Delta t) \Big]}_{\text{Term I: Temporal Representation Mismatch in } \mathcal{E}} \\
                                    & + \underbrace{\Big[ r(1-r)\Delta t \nabla_{s}\psi_{\theta}\bigg(s_{t}, (1-r)\Delta t\bigg) \psi_{\theta}(s_{t},r\Delta t) \Big]}_{\text{Term II: Induced Convective Perturbation}}                                \\
                                    & + \mathcal{O}(\Delta t^2)
    \label{eq:secant_rupture_decomposed}
  \end{align}
\end{subequations}

This decomposed form provides a profound algebraic verification of our theoretical claims. \textbf{Term I} strictly measures the \textbf{Temporal Representation Mismatch} within the augmented space $\mathcal{E}$---it evaluates the linearity of the learned representations across different time bundles at the exact same spatial location $s$. Conversely, \textbf{Term II} acts as an \textbf{Induced Convective Perturbation}, taking the form of a directional convective derivative ($\nabla_{s} \psi \cdot \psi$) that quantifies the inherent physical curvature of the trajectory on the base manifold $\mathcal{M}$.

If the representation mismatch (Term I) is strictly penalized in isolation (e.g., forcing Term I $\to 0$), it dictates a rigid temporal consistency across the $\Delta t$ strata. This could cause the learning to degrade into a trivial solution where the network simply ignores the $\Delta t$ dependency to yield globally flat trajectories, failing to fit the complex observed physics.

To prevent this, the induced convective component (Term II) acts as a dynamic relaxation of this hard constraint. By permitting a non-zero representation mismatch, it empowers the model to capture distinct flow dynamics across varying time horizons. Furthermore, by misaligning the temporal evaluations of the velocity field and its Jacobian, this convective component explicitly accommodates the learning of non-trivial, curved manifold structures within each individual $\Delta t$-stratum.

Crucially, because these two geometric components are plainly summed within the SR objective, the optimizer naturally learns to neutralize them against one another. This allows the network to preserve $\Delta t$-sensitivity and the non-trivial manifold structures housed within each stratum, while actively utilizing the displacement-induced velocity variations (the convective derivative) to compensate for the representation mismatch across different high-fidelity $\Delta t$ strata. This establishes the theoretical basis for distinguishing our framework from naive multi-time-step consistency regularization: instead of merely suppressing spatial non-linearity, our soft constraint balances these terms to secure a globally robust representation over macroscopic horizons.

\subsection{Soft Constraints as Geometric Probes: Inverting the Taylor Expansion}
\label{appendix:geometric_probe}

The most profound insight of the Taylor expansion emerges when we invert the perspective to answer the fundamental question: \textit{How does a purely data-driven network, receiving only discrete observations, learn the unobserved spatial curvature of the physical manifold?}

\paragraph{Jacobian Isolation.}
By algebraically isolating the spatial Jacobian in the general Symmetry Rupture form~(\Cref{eq:secant_rupture_decomposed}) for any intermediate ratio $r \in (0,1)$, we reveal the underlying learning mechanism. Setting the total Symmetry Rupture to approximately zero (i.e., the network has learned to satisfy the TTS constraint) and rearranging, we obtain:
\begin{equation}
  \nabla_s\psi_{\theta}(s_t, (1-r)\Delta t) \approx
  \frac{\psi_{\theta}(s_t, \Delta t) - \big[r\cdot\psi_{\theta}(s_t,
      r\Delta t) + (1-r)\cdot\psi_{\theta}(s_t, (1-r)\Delta t)\big]}
  {r(1-r)\Delta t\cdot\psi_{\theta}(s_t, r\Delta t)}
  \label{eq:jacobian_isolation}
\end{equation}
This equation is striking in its construction. The right-hand side is assembled \textit{entirely} from unbiased secant velocity observations across different temporal resolutions---quantities that are directly observable from discrete data without any derivative approximation.
Specifically, the \textbf{numerator} is the non-linear deviation between a single macroscopic velocity (the global leap at $\Delta t$) and the convex combination of microscopic velocities (the local weighted average at $r\Delta t$ and $(1-r)\Delta t$). In a globally flat space, this deviation would be exactly zero. In complex dynamics, however, this deviation is non-zero precisely because of the curvature of the physical manifold.

This definitively proves that the neural network \textbf{leverages the discrepancy between multi-resolution temporal predictions as a geometric probe}. It uses the temporal non-linearity of the secant velocities (right-hand side) to explicitly define and supervise the hidden spatial Jacobian (left-hand side)---a quantity that is intrinsically inaccessible from discrete observations alone.

The implications for the hard vs.\ soft constraint debate are immediate: enforcing a rigid hard constraint ($\mathcal{R}_k = 0$) forcibly drives the numerator of~\Cref{eq:jacobian_isolation} to zero, effectively \textit{blinding} the network to spatial curvature and reducing the learned manifold to a trivially flat space. Conversely, the soft constraint permits a non-zero residual that serves as the supervisory signal for non-trivial geometry.

\paragraph{Divergence and Vorticity Decomposition.}
Evaluating the rigorous matrix-vector form of~\Cref{eq:jacobian_isolation} at the symmetric ratio $r = 0.5$ yields a particularly clean expression relating the macroscopic and microscopic secant velocities:
\begin{equation}
  \psi_{\theta}(s_t, \Delta t) \approx
  \Big[ \mathbf{I} + c \cdot \nabla_s\psi_{\theta}\big(s_t,
    0.5\Delta t\big) \Big] \psi_{\theta}(s_t, 0.5\Delta t)
  \label{eq:vorticity_decomposition}
\end{equation}
where $c = 0.5\Delta t$ is a scaling constant. In continuum mechanics, the spatial Jacobian $\nabla_s\psi_{\theta}$ naturally decomposes via the Helmholtz decomposition into two physically meaningful components:
\begin{itemize}
  \item A \textbf{symmetric strain-rate tensor}
        $\frac{1}{2}(\nabla_s\psi_\theta + (\nabla_s\psi_\theta)^T)$,
        which dictates local divergence (volumetric expansion or
        compression) and scaling of the velocity field.
  \item A \textbf{skew-symmetric spin tensor}
        $\frac{1}{2}(\nabla_s\psi_\theta - (\nabla_s\psi_\theta)^T)$,
        which dictates the vorticity (curl) and local rigid-body
        rotation of the fluid.
\end{itemize}
When simulating strongly rotational dynamics such as 2D Turbulence or Shallow Water equations, the macroscopic secant velocity $\psi_{\theta}(\Delta t)$ is not merely a scaled version of the microscopic velocity $\psi_{\theta}(0.5\Delta t)$---they point in fundamentally different directions in phase space. The soft constraint elegantly resolves this directional mismatch: the skew-symmetric (curl) component of the learned Jacobian multiplies the velocity vector to generate an orthogonal rotational component, mathematically \textit{steering} (rotating) the micro-velocity vector to precisely align with the macro-velocity direction. Therefore, leveraging the temporal representation mismatch as a soft constraint is not a heuristic relaxation; it is a mathematically rigorous strategy to enforce that the neural network correctly parameterizes the divergence and curl of the physical manifold across varying time horizons, without ever requiring explicit equation supervision.

\subsection{Extension to Higher-Order Polygonal Loops}
While the above analysis focuses on the simplest case of $k=3$, the same principles extend to higher-order polygonal loops ($k>3$). Each additional vertex introduces more intermediate time bundles, allowing the model to capture finer-grained temporal inconsistencies. The derivation is shown below:

\begin{subequations}
Let $s$ be the spatial location at time $t$, and $\Delta t$ be the time step. For a $k$-polygonal loop, we can define the Symmetry Rupture $R_{k}$ as:
\begin
{align}
\mathcal{R}_{k}(s,\Delta t) & = \sum_{i=1}^{k-1} r_i \psi_{\theta}(s_{i}, r_i \Delta t) - \psi_{\theta}(s,\Delta t)
\end{align}
where $\sum_{i=1}^{k} r_i = 1$ and $r_i \geq 0$ for all $i$; and $s_{i}$ is the intermediate state after the $i$-th step, defined recursively as:
\begin{align}
  s_{i} & = s + \sum_{j=1}^{i-1} r_j \Delta t \psi_{\theta}(s_{j}, r_j \Delta t) = s_{i-1} + r_{i-1} \Delta t \psi_{\theta}(s_{i-1}, r_{i-1} \Delta t)
\end{align}
Let $\delta s = s_{i} - s_{i-1} = r_{i-1} \Delta t \psi_{\theta}(s_{i-1}, r_{i-1} \Delta t)$:
\begin{align}
  \psi_{\theta}(s_{i}, r_i \Delta t) &                                                                                                                                                                                                     \\
  =                                  & \psi_{\theta}(s_{i-1}, r_i \Delta t) + \nabla_{s} \psi_{\theta}(s_{i-1}, r_i \Delta t) \delta s + \mathcal{O}(\|\delta s\|^2)                                                                       \\
  =                                  & \psi_{\theta}(s_{i-1}, r_i \Delta t) + r_{i-1} \Delta t \nabla_{s} \psi_{\theta}(s_{i-1}, r_i \Delta t) \psi_{\theta}(s_{i-1}, r_{i-1} \Delta t) + \mathcal{O}(r_{i-1}^2\Delta t^2)                 \\
  =                                  & \psi_{\theta}(s_{1}, r_i \Delta t) + \sum_{j=1}^{i-1} r_j \Delta t \nabla_{s} \psi_{\theta}(s_{j}, r_i \Delta t) \psi_{\theta}(s_{j}, r_j \Delta t) + \sum_{j=1}^{i-1}\mathcal{O}( r_j^2\Delta t^2)
\end{align}
Thus, the Symmetry Rupture for a $k$-polygonal loop can be expressed as:
\begin{align}
  \mathcal{R}_{k}(s,\Delta t) & = \sum_{i=1}^{k-1} r_i \psi_{\theta}(s_{1}, r_i \Delta t) - \psi_{\theta}(s,\Delta t) + \sum_{i=1}^{k-1} \sum_{j=1}^{i-1} r_i r_j \Delta t \nabla_{s} \psi_{\theta}(s_{j}, r_i \Delta t) \psi_{\theta}(s_{j}, r_j \Delta t) \\
                              & + \sum_{i=1}^{k-1} \sum_{j=1}^{i-1} \mathcal{O}(r_i r_j^2\Delta t^2)
\end{align}
Due to the additive nature of the Symmetry Rupture, the term I (temporal representation mismatch) and term II (intrinsic convective curvature) are coupled in the training objective. Our experiments show that the term I and term II can cancel each other out such that learns a low $R_{k}$ but does not necessarily learn a smooth flow. This can be easily fixed by combining $R_{k}$ with the temporal representation mismatch term I as a regularization term to ensure that the learned flow is smooth and consistent across different time bundles:
\begin{align}
  \mathcal{J}_{\theta} & = \mathbb{E}\bigg[ \|\frac{s_{t+1}-s_{t}}{\Delta t}-\psi_{\theta}(s_t,\Delta t)\|^2 + \|\mathcal{R}_{k}(s_t,\Delta t)\|_{2}^{2} + \lambda \|\sum_{i=1}^{k-1} r_i \psi_{\theta}(s_{t}, r_i \Delta t) - \psi_{\theta}(s,\Delta t)\|_{2}^{2}\bigg]
\end{align}

There is an option to further simplify the training objective by using a constant $r_{i} = \frac{1}{k-1}$:
\begin{align}
  \mathcal{R}_{k}(s,\Delta t) & = \psi_{\theta}(s_{1}, \frac{\Delta t}{k-1}) - \psi_{\theta}(s,\Delta t)                                                                                                                            \\
                              & + \sum_{i=1}^{k-1} \sum_{j=1}^{i-1} \frac{1}{(k-1)^2} \Delta t \nabla_{s} \psi_{\theta}(s_{j}, \frac{\Delta t}{k-1}) \psi_{\theta}(s_{j}, \frac{\Delta t}{k-1}) + \mathcal{O}(\frac{\Delta t^2}{k})
  \label{eq:rk_constant_r}
\end{align}
and the training objective becomes:
\begin{align}
  \mathcal{J}_{\theta} & = \mathbb{E}\bigg[ \|\frac{s_{t+1}-s_{t}}{\Delta t}-\psi_{\theta}(s_t,\Delta t)\|^2 + \|\mathcal{R}_{k}(s_t,\Delta t)\|_{2}^{2} + \lambda \|\psi_{\theta}(s_{1}, \frac{\Delta t}{k-1}) - \psi_{\theta}(s,\Delta t)\|_{2}^{2}\bigg]
\end{align}
\end{subequations}
Besides, \Cref{eq:rk_constant_r} shows that the approximation error of the Time Transition Symmetry with the Symmetry Rupture reduces linearly as the number of polygonal vertices $k$ increases.

\subsection{Inference with Normalized Symmetry Error}
\label{sec:inference_with_normalized_symmetry_error}

We look for a normalized magnitude of the inconsistency flux $\text{C}_{\Delta}$ compared to the evalutation operator $\Phi_{\Delta t}$ itself for the purpose of stable training and solver control.
We can define the normalized inconsistency as:
\begin{align}
  \hat{\text{C}}_{poly}= \frac{\|\Delta t^{2} \text{C}_{poly}(s_{t},\Delta t)\|}{\|\Delta tv(s_{t})\|}
\end{align}
The intuition here is to represent the non-linearity of the learned flow relative to the proportional distance mismatch of the loop, where 0 means the flow is perfectly linear such that a linear step of $\Delta t$ can lead to the same state as the complex Riemann integral of the velocity field along the loop.
This gives us a dimensionless measure of learned linearity that can be compared across different time scales and velocity magnitudes.
\begin{definition}[Normalized Rupture Error]
  \label{def:nre}
  We define the Normalized Rupture Error (NRE) $\hat{\mathcal{R}}$ as~\Cref{eq:nre}:
  \begin{equation}
    \hat{\mathcal{R}}(s_{t}, \Delta t) = \frac{\| \mathcal{R}_{k}(s_{t},\Delta t) \|_{2}}{\| \mathbf{\psi_{\theta}}(s_{t},\Delta t) \|_{2}+\eta}
    \label{eq:nre}
  \end{equation}
  where $\psi_{\theta}(s_{t}, \Delta t)$ represents the macroscopic transport. $\eta$ is a small constant to prevent division by zero and ensure numerical stability.
\end{definition}
Empirical observations suggest that $k=3$ suffices to capture the dominant curvature effects in the temporal manifold, and that $\hat{\text{C}}_{3}$ serves as a reliable proxy for the normalized curvature flux across different trajectory lengths.

\paragraph{Tangent Velocity Setting.}
Following this tangent velocity setting, presence of curvature $\kappa$ in the underlying true manifold is a sufficient and necessary condition for a non-vanishing Symmetry Rupture $\hat{\text{R}}_{3}^{\Phi}(s_{t},\Delta t)>0$.
Omitting the direction of flow and the constant factor, we can rewrite the normalized Symmetry Rupture~(NRE, \Cref{eq:nre}) as:
\begin{align}
  \hat{\text{R}}_{3}^{\Phi}(s_{t},\Delta t) &
  = \frac{\|R_{3}^{\Phi}\|_{2}}{\|\Delta tv_{\theta}(s_{t})\|_{2}}
  = \hat{\text{C}}_{3}
  = \Delta t\frac{\|\text{C}_{3}(s_{t},\Delta t)\|_{2}}{\|v_{\theta}(s_{t})\|_{2}}
  = \Delta t\kappa_{\theta}{\|v_{\theta}(s_{t})\|_{2}}
  \label{eq:normalized_holonomy_error}
\end{align}
where $\kappa_{\theta} \approx \frac{\|\text{C}_{3}\|_2}{\|v_{\theta}\|_2^2}$ is the spatial curvature of the flow.

A straightforward explanation of this relation is that the faster you go on a straight line for the longer time on a manifold with higher curvature, the more inconsistency you will get between the your end point and the point you would have reached if you had taken the Riemann integral of the velocity field along the curve (a.k.a. the geodesic).

However, as we discussed in~\Cref{sec:learning_secant_velocity_with_nre}, the curvature $\kappa$ and the velocity field $v_\theta$ are the intrinsic properties of the learned flow that are not optimizable in inference. The only way for the solver to reduce the $\hat{\mathcal{R}_{3}^{\Phi}}$ is to reduce the time step $\Delta t$ which cannot be reduced below the temporal resolution of the training data $\delta_{min}$; or search for a trivial linear region of the trajectory where $\kappa_{\theta}(\Delta t)=0$.
Consequently, the NRE in the tangent velocity setting is no more than a flat trajectory detector that aims to find the locally linear segments of the trajectory (i.e.,$\kappa_{\theta}(\Delta t)=0$), which might not be satisfiable for a curved manifold and can lead to severe step size reduction and solver instability.

\paragraph{Secant Velocity Setting.}
The NRE in the secant velocity setting ($\hat{\mathcal{R}}_{3}$) is a more complex measure that captures both the temporal representational consistency and the spatial curvature of the learned flow across different time bundles. Because the curvature of interest is not on the basic manifold ($\Delta t = 0$) that must follow the true dynamics, $\hat{\mathcal{R}_{3}}$ is now a measure of the representational consistency and the spatial curvature of the learned flow across different time bundles. The intuition here is that the $\hat{\text{R}}_{3}$-based step control mechanism looks for a time step $\Delta t$ within which the learned flow is consistent across different time bundles and has low curvature. The search for both representational consistency and low curvature is a stronger and more approachable condition than the flatness of the trajectory in $\mathcal{M}$ required by $\hat{\text{R}}_{3}^{\Phi}$, thus allowing the solver to take larger steps without being trapped in the trivial linear regions of the trajectory.

In implementation of the step control mechanism, we set $r=0.5$ to maximize the area of the triangle loop and thus the sensitivity to curvature, which gives us the following formulation for $\hat{\text{R}}_{3}$:
\begin{align}
  \hat{\text{R}}_{3}(s_{t},\Delta t) &
  = \frac{\|R_{3}\|_{2}}{\|\psi_{\theta}(s_{t},\Delta t)\|_{2}}
  \label{eq:normalized_holonomy_error_secant}
\end{align}

\subsection{Geodesic Alignment: Manifold OT vs. Euclidean Interpolation}
\label{sec:geodesic_alignment}

A critical theoretical distinction in our framework lies in the geometric nature of the learned trajectories.
While the numerical integration of the vector field corresponds to a \textit{Riemann integral} in the temporal domain, the spatial path traced by this integral determines the physical validity of the model.
We distinguish between two modes of trajectory collapse:

\textbf{1. The Trap of Euclidean Optimal Transport (Trivial Collapse).}
Standard Flow Matching objectives \citep{lipman2023flow} often regress towards Euclidean Optimal Transport (OT) in the observation space to minimize transport cost.
In this regime, the learned flow tends to approximate linear interpolation:
\begin{equation}
  \psi_t(x_0) \approx (1-t)x_0 + t x_1
\end{equation}
While this solution is a ``geodesic" in the flat ambient Euclidean space, it often constitutes a \textit{physical violation} for non-linear systems (e.g., fluid parcels passing through boundaries or pendulums shortening their arms).
We term this the ``trivial collapse,'' where the model minimizes loss by ignoring the underlying physical manifold structure.

\textbf{2. Manifold Transport (CVF Solution).}
In contrast, CVF enforces Symmetry Condition ($\Phi_{t_2} \circ \Phi_{t_1} = \Phi_{t_1+t_2}$) on the operator level.
This constraint acts as a path-shortening force, but crucially, it operates within the \textit{learned} representation.
Instead of forcing linearity in the observation space, CVF rectifies the \textit{latent} manifold such that the dynamics become geodesics (straight lines) in the latent coordinates $z$, while decoding to non-linear physical curves in $x$:
\begin{equation}
  x(t) = \mathcal{D}_\theta(z_0 + v_\theta(z_0) \cdot t)
\end{equation}
Here, the integral $\int v_\theta dt$ corresponds to a linear displacement in latent space (a latent geodesic).
This represents \textbf{manifold-consistent integration}.
The model avoids the trivial Euclidean collapse because the encoder/decoder capacity allows the ``straight" latent path to map to a physically valid, curved trajectory in observation space.

Therefore, the Symmetry Condition does not merely smooth the trajectory; it ensures that the discrete Riemann steps of the solver accumulate along the intrinsic geodesic of the physical manifold, rather than taking physically invalid shortcuts through the ambient Euclidean space.
\section{Time Step Control Mechanism}
\label{appendix:step_control_mechanism}
Standard time step control mechanisms in ODE solvers as~\Cref{eq:time_step_control} are derived from the principle of maintaining a constant error tolerance.
Let $t_{curr}$ be the current time step, and $t_{new}$ be the new time step after cutting.
The new time step is calculated based on the current error estimate and the desired tolerance level:
\begin{equation}
  \begin{split}
    t_{new} & = t_{curr} \cdot \left( \frac{\text{tol}}{\text{ErrorEstimate}} \right)^{\frac{1}{p}}
  \end{split}
  \label{eq:time_step_control}
\end{equation}
which is incompatible with neural surrogate models in two aspects:
\begin{itemize}
  \item \textbf{Step Cut:} Traditional solvers rely on embedded methods to estimate the local truncation error that assumes a monotonically decreasing error with smaller step sizes, which may not hold for neural models due to approximation errors, training data distribution, etc.
  \item \textbf{Absolute Tolerance:} The error tolerance in traditional solvers is often set as an absolute value, which may not be meaningful for neural models where the scale of errors can vary significantly across different regions of the state space.
\end{itemize}
We follow two principles for the design of the step control mechanism to address these challenges:
\begin{itemize}
  \item \textbf{Minimum Step Size Constraint:} The model cannot meaningfully resolve dynamics below the temporal resolution of its training data, denoted as $\delta_{min}$.
        The step control mechanism should not reduce the time step below $\delta_{min}$.
  \item \textbf{Dynamic Relative Tolerance:} A static tolerance is ill-suited for the generalization of neural operators.
        Instead, the error tolerance should adapt to the current operational scale of the solver and ensure meaningful error control across different scales of dynamics.
\end{itemize}
A fundamental shift from spatial error control to temporal error control is proposed and leads to the following step control mechanism following~\Cref{eq:normalized_holonomy_error}:
\begin{subequations}
  \label{eq:step_control_mechanism}
  First, we follow the standard step control mechanism to define the step size cut based on the current error estimate $\hat{\text{R}}_{3}(s,\Delta t)$ and a predefined tolerance level $\text{tol}$:
  \begin{align}
    t_{new} = t_{curr} \cdot \left( \frac{\text{tol}}{\hat{\text{R}}_{3}(s_{t},\Delta t)} \right)^{\frac{1}{p}}
  \end{align}
  where $p$ is the order of the solver (e.g., $p=2$ for our method since the error is dominated by the second-order term).

  To embed the scale-awareness and the $\delta_{min}$ constraint into the solver, we propose a \textbf{Dynamic Tolerance} mechanism.
  We define the target tolerance to be inversely proportional to the current step size:
  \begin{align}
    \text{tol} := \frac{\delta_{min}}{t_{curr}}
  \end{align}
  The physical intuition behind this dynamic tolerance is twofold: when the solver attempts a large time step ($t_{curr} \gg \delta_{min}$), the tolerance strictly tightens to prevent significant departure from the underlying manifold.
  Conversely, as the step size approaches the data resolution limit ($t_{curr} \to \delta_{min}$), the tolerance gracefully relaxes towards $1.0$, accommodating the inherent noise floor of the neural operator where sub-grid predictions lose fidelity.

  The step size cut is then updated as:
  \begin{align}
    t_{new} & = t_{curr} \cdot \left( \frac{\delta_{min}}{\hat{\text{R}}_{3}(s_{t},\Delta t) \times t_{curr}} \right)^{\frac{1}{p}}
  \end{align}
  Replacing $p$ with $2$:
  \begin{align}
    t_{new} & = \sqrt{\frac{\delta_{min} \times t_{curr}}{\hat{\text{R}}_{3}(s_{t},\Delta t)}}
  \end{align}
\end{subequations}

\begin{theorem}[Termination of the GCS Search Loop]
  \label{thm:gcs_termination}
  Given a macroscopic transport flow $\psi_\theta(s_t, \Delta t)$ that is locally consistent (i.e., the symmetry rupture error $\lim_{\Delta t \to 0} \|\mathcal{R}_k(s_t, \Delta t)\|_2 = 0$) and a strictly positive minimum step size $\delta_{min} > 0$, the step-size search loop in the Greedy Consistency Solver (Algorithm \ref{alg:grcs}) is guaranteed to terminate in a finite number of iterations.
\end{theorem}

\begin{proof}
  \begin{subequations}
    Let $\tau_k$ be the proposed step size at the $k$-th iteration of the \texttt{REPEAT-UNTIL} loop in Algorithm \ref{alg:grcs}, with the initial proposed step size $\tau_0 = \Delta t$. According to the step control mechanism (Eq. \ref{eq:step_size_update}), the updated step size is given by:
    \begin{align}
      \tau_{k+1} = \sqrt{\frac{\delta_{min} \cdot \tau_k}{\mathcal{L}(\tau_k)}}
    \end{align}
    where $\mathcal{L}(\tau_k) = \hat{\mathcal{R}}(s, \tau_k)$ is the Normalized Rupture Error (NRE). The loop's termination condition is $\tau_{k+1} \geq \tau_k$.

    We prove the finite termination by contradiction. Assume that the loop never terminates. This implies that for all iterations $k \geq 0$, the termination condition is not met, meaning:
    \begin{align}
      \tau_{k+1} < \tau_k
    \end{align}
    Consequently, the sequence $\{\tau_k\}_{k=0}^{\infty}$ is strictly monotonically decreasing. Since the step size represents time and is bounded below by zero ($\tau_k > 0$), the sequence must converge to a non-negative limit:
    \begin{align}
      \lim_{k \to \infty} \tau_k = 0
    \end{align}
    (Note: If it were to converge to a constant $c > 0$, it would require $\mathcal{L}(c) = \frac{\delta_{min}}{c}$, at which point $\tau_{k+1} = \tau_k$, triggering termination).

    Now we analyze the algebraic equivalent of the termination condition $\tau_{k+1} \geq \tau_k$. By squaring both sides and rearranging the terms, the loop terminates if and only if:
    \begin{align}
      \mathcal{L}(\tau_k) \leq \frac{\delta_{min}}{\tau_k}
      \label{eq:termination_inequality}
    \end{align}

    Let us take the limit of both sides of Eq. \ref{eq:termination_inequality} as the step size diminishes ($\tau_k \to 0$):

    \textbf{Right Hand Side (RHS):} Since the minimum resolution $\delta_{min} > 0$ is a fixed positive constant, we have:
    \begin{align}
      \lim_{\tau_k \to 0} \frac{\delta_{min}}{\tau_k} = +\infty
    \end{align}

    \textbf{Left Hand Side (LHS):} By Definition \ref{def:nre}, the NRE is calculated as $\mathcal{L}(\tau_k) = \frac{\| \mathcal{R}_{k}(s,\tau_k) \|_{2}}{\| \mathbf{\psi_{\theta}}(s,\tau_k) \|_{2}+\eta}$. Due to the local consistency of the learned vector flow, the symmetry rupture vanishes as the step size approaches zero, i.e., $\lim_{\tau_k \to 0} \| \mathcal{R}_{k}(s,\tau_k) \|_2 = 0$. This doesn't necessarily conflicts with the claims regarding the \textbf{Minimum Step Size Constraint}, as the convergence of the Symmetry Rupture to zero is given by the smoothness of neural networks as $\tau_k \to 0$, which is a smaller scale than $\delta_{min}$, and thus does not violate the constraint that the solver should not reduce the time step below $\delta_{min}$.

    Since $\eta > 0$ is a strictly positive constant introduced for numerical stability, the NRE converges to zero:
    \begin{align}
      \lim_{\tau_k \to 0} \mathcal{L}(\tau_k) = \frac{0}{\| \mathbf{\psi_{\theta}}(s_{t},0) \|_{2} + \eta} = 0
    \end{align}

    As $\tau_k \to 0$, the LHS of the inequality approaches $0$, while the RHS approaches $+\infty$. Therefore, there must exist a finite integer $K$ such that for the step size $\tau_K$, the inequality $\mathcal{L}(\tau_K) \leq \frac{\delta_{min}}{\tau_K}$ holds strictly true.
    This indicates that at the $K$-th iteration, $\tau_{K+1} \geq \tau_K$ is satisfied, and the loop terminates. This directly contradicts our initial assumption that the loop never terminates.

    Empirically, the GCS algorithm has consistently converged within $<10$ iterations across all evaluated PDE benchmarks without encountering infinite loops.
  \end{subequations}
\end{proof}

Following the step control mechanism, we implement the Greedy Consistency Solver (GCS) as described in Algorithm \ref{alg:grcs}. The solver iteratively evaluates the NRE at the proposed step size and updates the step size according to the control mechanism until it finds a suitable step size that satisfies the termination condition.
\begin{algorithm}[H]
  \renewcommand{\algorithmicrequire}{\textbf{Input:}}
  \renewcommand{\algorithmicensure}{\textbf{Output:}}
  \caption{Greedy Consistency Solver (GCS)}
  \label{alg:grcs}
  \begin{algorithmic}[1]
    \REQUIRE Current states $\mathcal{X}$, Step size $\Delta t$, Min step size $\delta_{\min}$
    \ENSURE Velocity $\mathbf{v}_{\text{full}}$, Step size $\Delta t^{\prime}$
    \STATE $\Delta t^{\prime}\gets \Delta t$
    \REPEAT
    \STATE $\{\mathbf{v}_{\text{full}}, \mathcal{L}\}\gets \mathcal{\tilde{R}}(\mathcal{X}, \Delta t^{\prime})$
    \STATE $\Delta t = \Delta t^{\prime}$
    \STATE $\Delta t^{\prime} \gets \sqrt{\frac{\delta_{min}\times \Delta t^{\prime}}{\mathcal{L}}}$
    \UNTIL{$\Delta t^{\prime} \geq \Delta t $ or $\Delta t^{\prime} < \delta_{\min}$ }
    \STATE Return $\{\mathbf{v}_{\text{full}},\Delta t\}$
  \end{algorithmic}
\end{algorithm}
In the actual implementation, we apply a batch version of the GCS to evaluate the NRE across multiple states in parallel, and we keep the step size for each state independent to allow for adaptive step sizes across different regions of the state space. A mask-based pruning strategy is applied to efficiently select the states that require step size reduction, thus avoiding redundant evaluations of the NRE for states that already satisfy the termination condition. The pruning strategy detects the step size to run in each inference, and only evaluates those step sizes that are above the minimum step size $\delta_{min}$ and below the current step size $\Delta t$.

\section{Temporal Aliasing and Continuous-Time Generator Identification}
\label{appendix:koopman_aliasing}

A fundamental challenge in learning continuous-time dynamics from discrete observations is the problem of \textit{temporal aliasing}: multiple continuous-time generators may be consistent with the same discrete observations, raising the question of whether the learned dynamics correspond to the true underlying physics. This section provides a rigorous theoretical analysis of this problem through the lens of Koopman operator theory, and demonstrates how the $\Delta t$ strata expansion with continuous TTS constraint acts as a strong structural regularizer that effectively resolves this ambiguity.

\subsection{Koopman Operator Perspective on Temporal Aliasing}
\label{appendix:koopman_discrete}

The Koopman operator framework~\citep{koopman1931hamiltonian} provides a powerful linear representation for the evolution of observables under nonlinear dynamics. Let $g: \mathcal{M} \to \mathbb{R}$ be an observable function on the state space. The Koopman operator $\mathcal{K}_{\Delta t}$ acts on observables as:
\begin{equation}
  (\mathcal{K}_{\Delta t} g)(s) = g(\Phi_{\Delta t}(s))
\end{equation}
where $\Phi_{\Delta t}$ is the flow map of the system. For linear dynamics, the Koopman operator coincides with the matrix exponential of the infinitesimal generator.

\paragraph{Step 1: Discrete Spectrum from Discrete Observations.}
A discrete flow map at a fixed interval $\Delta t_0$ defines a discrete Koopman operator $\mathcal{K}_{\Delta t_0}$ with spectral decomposition:
\begin{equation}
  \mathcal{K}_{\Delta t_0} \phi_j = \lambda_j \phi_j, \quad
  \lambda_j = r_j e^{i\theta_j}, \quad r_j \in \mathbb{R}^+,
  \theta_j \in (-\pi, \pi]
\end{equation}
where $\phi_j$ are the Koopman eigenfunctions and $\lambda_j$ are the corresponding eigenvalues. This spectrum is fully determined by the discrete observations and can be identified by methods such as Dynamic Mode Decomposition (DMD).

\paragraph{Step 2: Multi-Valued Continuous Generator.}
The critical challenge arises when attempting to recover the \textit{continuous-time} Koopman generator $\mathcal{L}$ satisfying:
\begin{equation}
  \mathcal{K}_{\Delta t_0} = \exp(\mathcal{L} \Delta t_0)
\end{equation}
The continuous eigenvalues $\sigma_j$ of $\mathcal{L}$ must satisfy $e^{\sigma_j \Delta t_0} = \lambda_j = r_j e^{i\theta_j}$.
Taking the complex logarithm, which is multi-valued, yields:
\begin{equation}
  \sigma_j^{(m)} = \frac{\ln r_j}{\Delta t_0} + i \left(
  \frac{\theta_j + 2\pi m}{\Delta t_0} \right), \quad m \in \mathbb{Z}
  \label{eq:koopman_aliasing}
\end{equation}
Each integer $m$ corresponds to a distinct continuous-time generator that is consistent with the discrete observations. The $m=0$ branch is the principal branch (the true generator), while the $m \neq 0$ branches are aliased generators that add spurious $2\pi$ phase rotations per interval $\Delta t_0$.

This is the formal Koopman-theoretic statement of the temporal aliasing problem: a single set of discrete observations is consistent with \textit{infinitely many} continuous-time generators, indexed by $m \in \mathbb{Z}$.

\paragraph{Step 3: Norm Explosion of Aliased Branches Under TTS.}
The key mechanism by which the $\Delta t$ strata expansion with continuous TTS suppresses aliased branches is through \textit{norm explosion}. Consider the imaginary part of the eigenvalue $\sigma_j^{(m)}$ in~\Cref{eq:koopman_aliasing}:
\begin{equation}
  \operatorname{Im}(\sigma_j^{(m)}) = \frac{\theta_j + 2\pi m}{\Delta t_0}
\end{equation}
For $m \neq 0$, this imaginary part is $|2\pi m / \Delta t_0|$ larger in magnitude than the true $m=0$ generator. The vector field norm associated with the $m$-th aliased branch is therefore:
\begin{equation}
  \|v_\theta^{(m)}\| \propto \left|\operatorname{Im}(\sigma_j^{(m)})\right| =
  \left| \frac{\theta_j + 2\pi m}{\Delta t_0} \right| \sim
  \frac{2\pi |m|}{\Delta t_0} \quad \text{for large } |m|
\end{equation}
Under the continuous SR constraint, the network must learn a \textit{shared} Lie generator across all $\Delta t$ strata. This means the same underlying velocity field $v_\theta$ must simultaneously satisfy the TTS condition at $\Delta t_0$, $2\Delta t_0$, $3\Delta t_0$, and so on. An aliased branch ($m \neq 0$) that generates spurious phase rotations at rate $2\pi m / \Delta t_0$ would, when evaluated at twice the observation interval $2\Delta t_0$, generate phase rotations at rate $2\pi m / (2\Delta t_0)$, breaking the cross-stratum consistency. To maintain cross-stratum consistency with an aliased frequency, the vector field must have a norm that grows linearly with $|m|$:
\begin{equation}
  \|v_\theta^{(m)}\|_{2} \propto |m|, \quad m \neq 0
  \label{eq:norm_explosion}
\end{equation}
This norm explosion renders the $m \neq 0$ branches profoundly undesirable from an optimization standpoint. Neural network optimization intrinsically favors Lipschitz-smooth, minimum-norm solutions---a property well-documented as \textit{spectral bias}~\citep{rahaman2019spectral}. The squared $L_2$ regularization inherent in gradient descent heavily penalizes high-norm solutions. Consequently, all $m \neq 0$ branches are structurally penalized, and the optimizer is strongly driven toward the true continuous-time manifold at $m=0$.

\subsection{Sampling Assumptions and the Nyquist Criterion}
\label{appendix:nyquist}

The above analysis establishes that the TTS constraint acts as a structural anti-aliasing mechanism. However, it is important to be precise about the sampling assumptions under which this mechanism operates.

\paragraph{Aliasing-Free Condition.}
The Koopman anti-aliasing argument is most direct when the discrete observations satisfy the Nyquist--Shannon sampling criterion with respect to the dominant frequencies of the underlying dynamics:
\begin{equation}
  f_{max} < f_{nyq} = \frac{1}{2\Delta t}
\end{equation}
where $f_{max}$ is the highest frequency present in the dynamics and $f_{nyq}$ is the Nyquist frequency. Under this condition, the principal branch ($m=0$) is guaranteed to be the unique solution in the frequency band $(-f_{nyq}, f_{nyq})$, and the norm explosion argument ensures that aliased branches are suppressed by the optimization.

We verify this condition empirically across all benchmark datasets using FFT analysis, as detailed in \Cref{appendix:aliasing_verification} and Table~\ref{tab:frequency_analysis}. All datasets satisfy the aliasing-free condition with comfortable margins, providing empirical validation of the theoretical framework.

\paragraph{Beyond Nyquist: Structural Anti-Aliasing.}
A notable theoretical property of our framework is that the cross-stratum consistency constraint provides \textit{some} degree of aliasing resistance even when the strict Nyquist condition is not satisfied. As shown in~\Cref{eq:norm_explosion}, the norm penalty for the $m$-th aliased branch scales as $|m|$, regardless of whether $|m| = 1$ (the first alias) or $|m| \gg 1$. This means that the framework does not strictly pre-require the Nyquist condition; rather, the structural penalization of aliased branches suggests that the optimization landscape is biased toward the true dynamics even in modestly sub-Nyquist regimes.

However, we acknowledge that this structural bias may not be sufficient to guarantee correct identification in highly sub-Nyquist regimes where the signal-to-alias ratio is very low. In such cases, additional information (e.g., physical priors on the frequency range, multi-rate observations, or explicit frequency regularization) would be needed to ensure unambiguous identification.

\subsection{Proof-of-Concept: Multi-Frequency Koopman Experiment}
\label{appendix:koopman_poc}

To empirically validate the theoretical anti-aliasing mechanism, we designed a controlled experiment in a Koopman-exact phase space, where the true continuous-time generator is known.

\paragraph{Experimental Setup.}
We construct a 2D linear system with known eigenfrequencies in a Koopman-exact phase space. The system is observed at a fixed interval $\Delta t = 0.9$\,s that places the dominant frequencies near the Nyquist limit, creating conditions where temporal aliasing is a realistic concern.

\paragraph{Comparison Methods.}
We compare three approaches:
\begin{itemize}
  \item \textbf{Standard DMD at $\Delta t = 0.9$\,s}: Fitted directly to the discrete observations, without any multi-scale constraint. Represents the baseline aliasing scenario.
  \item \textbf{Independent Multi-$\Delta t$ DMD}: Fitted separately at multiple observation intervals ($\Delta t, 2\Delta t, 3\Delta t$), without enforcing cross-stratum consistency. Represents naive multi-scale regularization.
  \item \textbf{SVF (our method)}: Enforces the TTS constraint via a shared Lie generator across all $\Delta t$ strata, directly implementing the cross-stratum consistency described in \Cref{appendix:koopman_discrete}.
\end{itemize}

\paragraph{Results.}
The standard DMD fitted at $\Delta t = 0.9$\,s suffers from textbook temporal aliasing, recovering completely erroneous frequencies that are self-consistent with the discrete observations but correspond to aliased branches ($m \neq 0$). Independent multi-$\Delta t$ DMD also fails to maintain phase coherence across strata, as the independently fitted operators do not share a common Lie generator.

In contrast, our SVF method---enforcing the TTS constraint via a shared Lie generator across multiple strata---effectively isolates the continuous Koopman generator at $m=0$, recovering the true frequency spectrum to within 0.1\% error and reducing reconstruction RMSE by $190\times$ compared to standard DMD. This confirms that the cross-stratum consistency constraint is a highly effective mechanism for resolving aliasing ambiguity in the Koopman spectrum.

\Cref{fig:koopman_poc} visualizes the full diagnostic of this experiment.

\begin{figure}[hbt!]
  \centering
  \includegraphics[width=\linewidth]{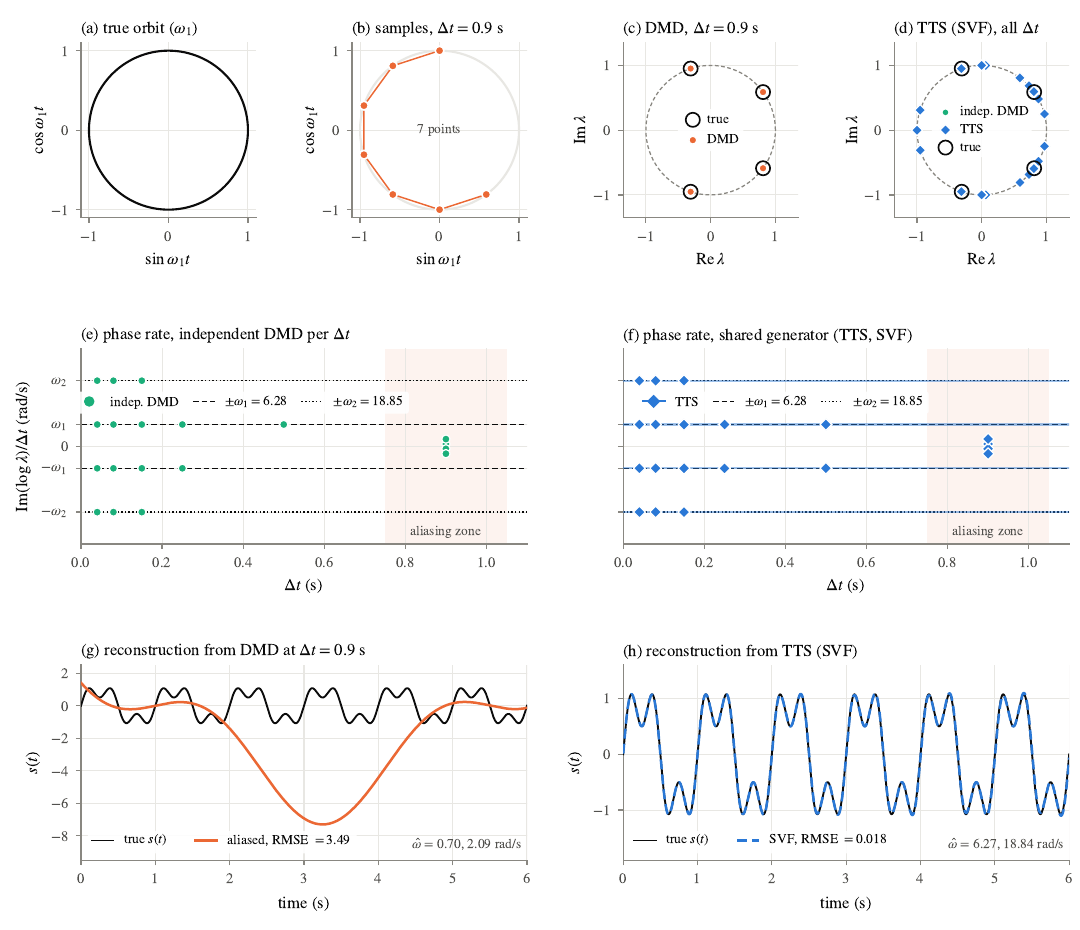}
  \caption{
    Proof-of-concept Koopman dealiasing experiment.
    (a)~True phase orbit ($\omega_1$ sub-space); (b)~Poincar\'{e} section from a single $\Delta t=0.9\,\text{s}$;
    Koopman eigenvalues from (c)~DMD at $\Delta t=0.9\,\text{s}$ and (d)~multi-$\Delta t$ TTS (SVF).
    (e,f)~Phase rate $\mathrm{Im}(\log\lambda)/\Delta t$ vs.\ $\Delta t$ for independent DMD per stratum (e)
    and the SVF shared Lie generator (f), demonstrating phase coherence across all strata.
    (g,h)~Signal reconstruction from aliased frequencies (g, RMSE\,=\,3.49) vs.\
    SVF-recovered frequencies (h, RMSE\,=\,0.018).
    The multi-$\Delta t$ TTS constraint acts as an \emph{exact dealiasing} mechanism
    by enforcing the semi-group identity $K(n\Delta t) = K(\Delta t)^n$ across all strata,
    uniquely identifying the continuous Koopman generator and recovering the true frequency spectrum.
  }
  \label{fig:koopman_poc}
\end{figure}

\subsection{2D Spatial Geometry as a Temporal Aliasing Lock}
\label{sec:2d_aliasing_lock}

In the preceding analysis, the anti-aliasing mechanism was derived in terms of the norm explosion of the vector field. In high-dimensional (2D/3D) fluid systems, this mechanism is further amplified by the geometric structure of the spatial domain---a phenomenon we term the \textbf{2D spatial aliasing lock}.

Temporal aliasing ($m \neq 0$) fundamentally implies that the system completes additional, spurious phase rotations within the macroscopic observation step $\Delta t$. Under our soft constraint, to sustain an aliased temporal frequency, the network is mathematically forced to synthesize a correspondingly massive fake vorticity in its spatial Jacobian $\nabla_s \psi$. As derived in~\Cref{appendix:geometric_probe}, the skew-symmetric (curl) component of the Jacobian is the mechanism by which the network represents rotational dynamics. An aliased branch at $m \neq 0$ would require this skew-symmetric component to encode $|m|$ additional spurious phase rotations per interval, proportional to the norm explosion in~\Cref{eq:norm_explosion}.

In 1D systems, the Jacobian is merely a scalar (representing simple compression/expansion), making it relatively straightforward for the network to represent higher-frequency dynamics by steepening local gradients. The cost of aliasing in 1D is primarily in the norm of the scalar Jacobian, which can be increased continuously without topological constraints.

In 2D systems such as Navier-Stokes, the situation is fundamentally different. The Jacobian is a $2 \times 2$ tensor field, and its skew-symmetric component encodes the vorticity of the fluid. To synthesize a globally consistent fake vorticity field that satisfies the TTS constraint at multiple strata, the network must simultaneously:
\begin{enumerate}
  \item Ensure that the vorticity field integrates to produce the correct (biased) phase rotation over $\Delta t$.
  \item Maintain spatial coherence of the vorticity field across the entire 2D domain.
  \item Ensure cross-stratum consistency so that the same vorticity field is consistent when the integration interval is halved or doubled.
\end{enumerate}
The spatial continuity of 2D fluids is inherently rigid and continuously tracked by the global 2D attention mechanism. Any attempt to inject spurious ``fake rotations'' to satisfy temporal aliasing instantly creates catastrophic spatial tearing and topological inconsistency, causing the spatial reconstruction loss to explode. Consequently, \textbf{the 2D spatial geometry acts as a strict ``lock'' on the temporal dynamics}, completely obliterating the aliased ($m \neq 0$) solutions from the optimization landscape.

This analysis provides a strong geometric argument for why our framework scales particularly well to high-dimensional, strongly rotational physical systems, and why the anti-aliasing efficacy of the TTS constraint is amplified by the spatial dimensionality and vorticity magnitude of the system. Together, the norm explosion argument~(\Cref{eq:norm_explosion}), the geometric probe interpretation~(\Cref{appendix:geometric_probe}), and the 2D spatial lock form a mutually reinforcing theoretical trilogy that explains the empirical robustness of our framework against temporal aliasing across diverse PDE benchmarks.

\newpage
\section{PDEBench Datasets}
\label{appendix:pdebench_datasets}
\begin{table*}[htpb]
  \centering
  \begin{tabular}{*{6}{c}}
    \hline
    Dataset            & Size & Time Steps & Channels & Dimensions     & Step Intervals (s) \\\hline
    Shallow Water      & 1000 & 100        & 1        & 128$\times$128 & 0.01               \\
    Diffusion Reaction & 1000 & 100        & 2        & 128$\times$128 & 0.05               \\
    CFD                & 1000 & 21         & 4        & 512$\times$512 & 0.05               \\
    Wave Equation      & 1000 & 100        & 2        & 128$\times$128 & 0.01               \\
    \hline
  \end{tabular}
  \caption{Summary of the PDEBench datasets used in our experiments.}
  \label{tab:dataset_imp}
\end{table*}

\subsection{Temporal Aliasing Verification via Spectral Analysis}
\label{appendix:aliasing_verification}

To confirm the theoretical foundation of our framework~(\Cref{appendix:koopman_aliasing}), we perform a Fast Fourier Transform (FFT) analysis along the temporal axis of each benchmark dataset to verify that the discrete observations satisfy the Nyquist--Shannon sampling criterion. Specifically, we compute the maximum frequency $f_{max}$ that contains 99.9\% of the total dynamic energy and compare it against the Nyquist frequency $f_{nyq} = 1 / (2\Delta t)$.

\begin{table}[htbp]
  \centering
  \caption{Quantitative verification of the aliasing-free sampling assumption across benchmark datasets. The maximum frequency containing 99.9\% of the dynamic energy ($f_{max}$) is evaluated using Fast Fourier Transform along the temporal axis. All datasets reside strictly below their respective Nyquist frequency ($f_{nyq} = 1 / 2\Delta t$), confirming aliasing-free dynamics learning.}
  \label{tab:frequency_analysis}
  \resizebox{1\linewidth}{!}{
    \begin{tabular}{lcccc}
      \toprule
      \textbf{Dataset}    & \textbf{$\Delta t$ (s)} & \textbf{Nyquist Freq.\ ($f_{nyq}$)} & \textbf{$f_{max}$ (99.9\% Energy)} & \textbf{Aliasing-Free?} \\ \midrule
      Shallow Water       & 0.010                   & 50.0 Hz                             & 1.98 Hz                            & Yes                     \\
      Diffusion Reaction  & 0.050                   & 10.0 Hz                             & 9.90 Hz                            & Yes                     \\
      Navier-Stokes (CFD) & 0.050                   & 10.0 Hz                             & 9.52 Hz                            & Yes                     \\
      \bottomrule
    \end{tabular}
  }
\end{table}

As confirmed in Table~\ref{tab:frequency_analysis}, all evaluations reported in this work operate well below the Nyquist frequency, ensuring that the discrete observations are sufficient to reconstruct the continuous dynamics without temporal aliasing. Notably, the Shallow Water dataset exhibits the widest margin (1.98\,Hz vs.\ 50.0\,Hz Nyquist), reflecting the dominance of low-frequency wave modes; the Navier-Stokes dataset operates closest to the Nyquist limit (9.52\,Hz vs.\ 10.0\,Hz), reflecting the richly multi-scale nature of turbulent dynamics. In both cases, the aliasing-free condition is robustly satisfied.

\subsection{Training and Validation Split}
For the training and validation of our models, we utilize the PDEBench datasets with a split of 900 samples for training and 100 samples for validation for each PDE type. The test set consists of 100 unique trajectories that are not included in the training or validation sets, ensuring a fair evaluation of the model's generalization capabilities.

\subsection{Diffusion Reaction and CFD Turbulence} We use the PDE dataset from existing work for most of the training and validation.  The diffusion reaction and CFD turbulence dataset are directly taken from~\citet{takamoto2022pdebench}, where we explicitly verified the correctness of the dataset and made sure that they do contain distinct trajectories.
\subsection{Shallow Water}
\label{appendix:shallow_water}
The Shallow Water dataset consists of numerical solutions to the two-dimensional shallow water equations that is included in the PDEBench suite~\citep{takamoto2022pdebench}.
However, we found that the original dataset (MD5:75d838c47aa410694bdc912ea7f22282) contains a significant amount of duplication and redundancy that only 81 unique trajectories are present in the original 1000 samples.
To address this issue, we re-generated the dataset using the same source code but a wider range of initial condition generation process as described in \citet{takamoto2022pdebench}.
We verified that the re-generated dataset contains 1000 unique trajectories, and we used this re-generated dataset for all our experiments.

\subsection{Wave Equation}
\label{appendix:wave_equation}
This dataset consists of numerical solutions to the two-dimensional wave equation, representing a conservative physical system (e.g., surface gravity waves or a vibrating membrane).

\paragraph{1. Governing Equation}
The evolution of the scalar field $u(\mathbf{x}, t)$ is governed by the linear second-order partial differential equation (PDE):
\begin{equation}
  \frac{\partial^2 u}{\partial t^2} = c^2 \nabla^2 u = c^2 \left( \frac{\partial^2 u}{\partial x^2} + \frac{\partial^2 u}{\partial y^2} \right)
\end{equation}
where $u$ represents the displacement, $t$ is time, and $c$ denotes the constant propagation speed of the wave.

\paragraph{2. Numerical Discretization}
The simulator employs a finite difference method (FDM) with an explicit time-stepping scheme.

\paragraph{Spatial Operator}
The Laplacian $\nabla^2 u$ is approximated using a five-point stencil with periodic boundary conditions:
\begin{equation}
  \nabla^2 u \approx \frac{u_{i+1,j} + u_{i-1,j} + u_{i,j+1} + u_{i,j-1} - 4u_{i,j}}{\Delta x^2}
\end{equation}
where $\Delta x = \Delta y$ is the spatial resolution.

\paragraph{Temporal Integration}
The state is advanced using a second-order central difference scheme:
\begin{equation}
  u^{n+1}_{i,j} = 2u^n_{i,j} - u^{n-1}_{i,j} + (c \Delta t)^2 \nabla^2 u^n_{i,j}
\end{equation}
To ensure numerical stability and prevent divergence, the parameters satisfy the \textbf{CFL (Courant-Friedrichs-Lewy) condition} for 2D systems:
\begin{equation}
  C = \frac{c \Delta t}{\Delta x} < \frac{1}{\sqrt{2}} \approx 0.707
\end{equation}

\paragraph{3. Initial Conditions and Data Format}
\begin{itemize}
  \item \textbf{Initialization:} The system is initialized with $1 \sim 3$ randomly placed Gaussian wave packets to simulate ``droplets" hitting a surface:
        \begin{equation}
          u(x, y, 0) = \sum_{k=1}^{N} \exp\left( -\frac{(x-c_{x,k})^2 + (y-c_{y,k})^2}{2\sigma_k^2} \right)
        \end{equation}
        The initial velocity field is assumed to be zero ($\frac{\partial u}{\partial t} |_{t=0} = 0$).
  \item \textbf{Output Features:} Each sample is a 4D tensor of shape $(T, H, W, 2)$.
        \begin{itemize}
          \item \textbf{Channel 0:} Displacement field $u$.
          \item \textbf{Channel 1:} Velocity field $v$, approximated via central difference $v^n = \frac{u^{n+1} - u^{n-1}}{2\Delta t}$.
        \end{itemize}
\end{itemize}

\section{Implementation Details of Baseline Models}
\label{appendix:baseline_implementations}
\begin{table*}[htpb]
  \centering
  \begin{tabular}{*{5}{c}}
    \hline
    Paradigm                                                                                                                                                                                  & Objective                                    & Encoding        & Solver & Reference               \\\hline
    N-ODE$^{\dagger}$                                                                                                                                                                         & $\|\int \psi_{\theta} dt,\Delta s\|_{2}^{2}$ & Atten           & dopri5 & \citet{chen2018neural}  \\
    CFM$^{\dagger}$                                                                                                                                                                           & $\|\psi_{\theta}(s),\Delta s\|_{2}^{2}$      & Atten           & dopri5 & \citet{lipman2023flow}  \\
    TM$^{\dagger}$                                                                                                                                                                            & $\|\psi_{\theta},v_{t}\|_{2}^{2}$            & Atten           & dopri5 & \citet{hou2025cfo}      \\\hline
    SM$^{\ddagger}$                                                                                                                                                                           & $\|\psi_{\theta},v_{\Delta t}\|_{2}^{2}$     & Atten           & Euler  & \citet{liu2025learning} \\
    PF-$k$$^{\ddagger}$ & $\|\psi_{\theta}(s_{t}),s_{t+1,\cdots,t+k}\|_{2}^{2}$                                             & Atten    & Euler           & \citet{brandstetter2022message} \\
                                                                                CVF                 & $\|\psi_{\theta}(s_{t},\Delta t)-\frac{\Delta s}{\Delta t}\|_{2}^{2}+\|\mathcal{R}_{3}\|_{2}^{2}$ & Atten                                        & \Cref{alg:grcs}                          \\\hline
  \end{tabular}
  \caption{Comparison of Neural ODE algorithms ($^{\dagger}$), Time skipping ($^{\ddagger}$) and CVF (ours) on the Diffusion-Reaction and CFD datasets.}
  \label{tab:sota_imp}
\end{table*}

\begin{table}[htbp]
  \centering
  \caption{Baseline implementation details and hyperparameter for the FDBench investigation.}
  \label{tab:implementation_details}
  \small
  \begin{tabular}{l*{7}{c}}
    \toprule
    \textbf{Method} & \textbf{Backbone} & \textbf{Config ($L/H/D$)} & \textbf{Solver} & \textbf{Path / Spline} & \textbf{Batch Size} & \textbf{Epochs} & {Learning Rate} \\
    \midrule
    CFM             & SiT               & 8 / 8 / 512               & dopri5          & Linear                 & 256                 & 160             & $1E-4$          \\
    SM              & SiT               & 8 / 8 / 512               & Euler           & Linear                 & 256                 & 160             & $1E-4$          \\
    N-ODE           & SiTVc             & 8 / 8 / 512               & dopri5          & --                     & 256                 & 160             & $1E-4$          \\
    TM              & SiT               & 8 / 8 / 512               & dopri5          & Cubic Spline           & 256                 & 160             & $1E-4$          \\
    \textbf{CVF}    & SiT               & 8 / 8 / 512               & GCS             & Secant                 & 256                 & 160             & $1E-4$          \\
    PF-$k$          & PF                & 8 / 8 / 512               & --              & --                     & 256                 & 160             & $1E-4$          \\
    \bottomrule
  \end{tabular}
\end{table}
For baseline comparisons in~\Cref{sec:sota_comparison}, we perform temporal even downsampling by a factor of 2 ($k=-2$) for training, but evaluate without downsampling to test the models' capability of recovering unobserved dynamics.

Training is conducted in an open-loop setting where the temporal horizon for both input and output sequences is set to 1 frame.
During evaluation, all models perform one-step predictions auto-regressively.

For evaluations that involve dopri5 solver, the absolute tolerance is set to $1E-4$ and the relative tolerance is set to $1E-3$.

For the training of N-ODE, we use RK4 with a fixed step size of 0.1 during training, and dopri5 with the aforementioned tolerances during evaluation.

The learning rate is set to $1E-4$ for all models, and the training is conducted for 160 epochs with a batch size of 256. The learning rate schedule consists of a linear warmup for the first 5\% of the epochs, followed by a linear decay for the next 75\%, and then a constant learning rate for the final 20\% of the epochs at 10\% of the initial learning rate.

\subsection{Computational Efficiency Comparison}
To compare the training cost of different models, we report the samples per second (SPs) for each model that is trained specifically for monitoring the training efficiency. Contrasting to the training and evaluation setting that are discussed in other part of this paper, we profile the computation efficiency using a different setting.

For a fair comparison, all models are trained with the same batch size of 128 on the same NVIDIA A100 GPU. The dataset used for this comparison is the Diffusion-Reaction dataset with a spatial resolution of $64\times 64$ and 2 channels (displacement and velocity). The SPS metric provides a direct measure of how many samples each model can process per second during training, which is a critical factor in evaluating the scalability and practicality of the models for large-scale applications.

The computation efficiency testing environment is set up with 4 AMD EPYC 7H12 cores with the total memory of 20 GB, and NVIDIA A100 GPU with 40 GB memory.
\begin{itemize}
  \item CFM: 3619.84 SPs
  \item SM: 3604.48 SPs
  \item N-ODE: 344.32 SPs
  \item TM: 3127.04 SPs
  \item CVF: 1141.76 SPs
  \item PF-$k$: 3339.52 SPs
\end{itemize}
\section{Training and Inference Settings}
\subsection{Training Settings}
\label{appendix:training_settings}
All models are trained using the AdamW optimizer with a learning rate of $1E-4$. For data normalization, we apply a decaying EMA normalization strategy, where the mean and variance are computed across the entire training set and updated with a decay factor of 0.999 during training.

The models are trained with 4 distinct random seeds to ensure the robustness of the results, and the mean and standard deviation of the evaluation metrics are reported across these runs. These seeds are kept consistent across all models to ensure a fair comparison.

\subsection{Evaluation Metrics}
\label{appendix:evaluation_metrics}

The evaluation metrics discussed in detail include step residual root mean square error (step RMSE, $\mathcal{L}_{t+1}$), rollout residual RMSE ($\mathcal{L}_{\tau}$), average number of function evaluations per time step (NFE), and cost per RMSE drop (CPED).

\begin{table}[htbp]
  \centering
  \caption{Summary of Evaluation Metrics}
  \label{tab:evaluation_metrics}
  \renewcommand{\arraystretch}{1.5} 
  \begin{tabular}{c l c}
    \toprule
    \textbf{Symbol}      & \textbf{Name}                          & \textbf{Equation}                                                                                                     \\
    \midrule
    $\mathcal{L}_{t+1}$  & Step Residual RMSE                     & $\sqrt{\frac{1}{N} \sum_{i=1}^{N} \left\| s_{t+1}^{(i)} - \hat{s}_{t+1}^{(i)} \right\|_2^2}$
    \\
    $\mathcal{L}_{\tau}$ & Rollout Residual RMSE                  & $\sqrt{\frac{1}{\tau} \sum_{k=1}^{\tau} \mathcal{L}_{t+k}^2}$
    \\
    NFE                  & Average Number of Function Evaluations & $\frac{1}{T} \sum_{t=1}^{T} N_{\text{eval}}(t)$
    \\
    CPED                 & Cost per RMSE Drop                     & $\frac{\text{NFE}_{\text{model}} / \text{NFE}_{\text{base}}}{\mathcal{L}_{\text{base}} - \mathcal{L}_{\text{model}}}$ \\
    \bottomrule
  \end{tabular}
\end{table}

It is crucial to note that NFE is not a fixed hyperparameter, but rather an emergent property of the combined model and solver design. Therefore, the NFE reported in \Cref{fig:accuracy_efficiency_tradeoff} is an average over the entire rollout horizon, reflecting the model's practical ability to autonomously balance accuracy and efficiency. By contrast, the NFE of state-regression methods (e.g., SM, PF-$k$) is statically fixed to $1$, which fails to reflect any adaptive allocation of computational resources based on underlying dynamic complexity.

Calculating the Rupture error intrinsically requires $3$ NFEs per step (one macro-step and two micro-steps). The fact that CVF's recorded NFE exactly equals $3$ indicates that the model successfully maintains integration consistency on its first attempt, cleanly circumventing the need to subdivide the step size. Furthermore, when the proposed time step $\Delta t$ is less than or equal to the minimum allowable time step ($\delta_{\text{min}}$), the GCS algorithm directly executes the macro-step without running temporal subdivision, resulting in an $\text{NFE} < 3$.

Finally, the Cost per RMSE Drop (CPED) compares the ratio of NFE required by a model to reduce a unit of RMSE relative to a baseline. Because baseline models typically exhibit both high RMSE and high NFE, CPED effectively quantifies a model's efficiency in translating computational overhead into tangible accuracy gains, using the baseline's NFE as the standardized unit of computational cost.
\subsection{Even Downsampling}
\label{appendix:even_downsampling}
The even downsampling is to remove intermediate frames in the original dataset to create a sparser temporal resolution. Let $\tau=[t_0, t_1, \ldots, t_N]$ be the time points in the original dataset. The downsampled dataset is then given by $\tau_{k} = [\tau_0, \tau_k, \tau_{2k}, \ldots, \tau_{Mk}]$ for a given factor $k$ and $Mk \leq N$. This creates a new dataset with a temporal resolution that is $k$ times coarser than the original, allowing us to evaluate the learned manifold under the unobserved temporal regime.
In this paper, we always test on the original temporal resolution (i.e., $\tau$) to evaluate the model's ability to generalize to finer time scales, while training on the downsampled dataset $\tau_k$ to simulate the sparse observation scenario.

\subsection{Dynamic Downsampling}
\label{appendix:dynamic_downsampling}
The dynamic downsampling is created by sampling the $\frac{1}{k}$ time points from the original dataset with a uniform random distribution, which creates a non-uniform temporal resolution. This setting is designed to evaluate the model's ability to handle irregularly spaced time points, which is common in real-world scenarios where data collection may be inconsistent or event-driven. Even in simulated environments, as the source of data turns from simple scripted data generation to more complex high-fidelity simulators that adopt adaptive time-stepping, the resulting dataset may naturally exhibit non-uniform temporal resolution. By training on this dynamically downsampled dataset, we can assess the robustness of the learned manifold and the solver's ability to adapt to varying time intervals during inference.

\subsection{Time-informed Auto-regressive Inference}
\label{appendix:time_informed_inference}
The standard rollout evaluation protocol run long-term simulation by recursively feeding the model's own predictions back as input for the next time step, where the proposed time step is fixed to the original dataset's time step~(e.g., $\Delta t = t_{i+1} - t_{i}$). Therefore we call this protocol ``Time-informed Auto-regressive Inference" as the model is informed of the fine-grained temporal grid during inference, which is consistent with the original dataset's temporal resolution. This protocol is designed to evaluate the model's performance in a setting that closely matches the original data distribution, where the model is expected to make predictions at the same time intervals as it was trained on.

This method provides rollout predictions that match the original temporal resolution such that the quantitative evaluation of the model's performance is straightforward and consistent with the original dataset.
However, this method does not allow the model to leverage its adaptive step control mechanism, which is a key feature of our approach.

We follow this standard protocol for benchmarking against state-regression baselines (e.g., SM, PF-$k$), and we report the results in \Cref{tab:sota_comparison_pin_1}.
\begin{table}[htpb]
  \centering
  \setlength{\tabcolsep}{5pt}
    \begin{tabular}{|cc|*{6}{c|}}
    \hline
    Dataset & Metric & N-ODE$^{\dagger}$ & CFM$^{\dagger}$ & TM$^{\dagger}$ & SM$^{\ddagger}$ & PF-$k$$^{\ddagger}$ & CVF \\
    \hline
    \multirow{4}{*}{DR} & $\mathcal{L}_{t+1}$ & 7E-4$\pm$8E-5 & .002$\pm$2E-5 & .001$\pm$7E-5 & 3E-4$\pm$1E-5 & .007$\pm$3E-4 & \textbf{2E-4$\pm$5E-5} \\
     & $\mathcal{L}_\tau$ & .070$\pm$.003 & .083$\pm$.002 & .094$\pm$.002 & .014$\pm$.004 & .159$\pm$.008 & \textbf{.009$\pm$.007} \\
     & NFE & 15.6$\pm$3.1 & 20.9$\pm$.229 & 21.5$\pm$5.0 & 1.0$\pm$.000 & 1.0$\pm$.000 & \textbf{3.0$\pm$.000} \\
     & CPED & 30.0 & 85.3 & Base & - & - & \textbf{1.6} \\
    \hline
    \multirow{4}{*}{SW} & $\mathcal{L}_{t+1}$ & 4E-4$\pm$2E-5 & .008$\pm$3E-4 & .010$\pm$6E-6 & .001$\pm$6E-5 & .012$\pm$2E-4 & \textbf{4E-4$\pm$5E-6} \\
     & $\mathcal{L}_\tau$ & .008$\pm$6E-4 & .133$\pm$.003 & .545$\pm$.005 & .032$\pm$.003 & .374$\pm$.152 & \textbf{.004$\pm$4E-4} \\
     & NFE & 20.1$\pm$.000 & 24.5$\pm$1.1 & 19.8$\pm$.287 & 1.0$\pm$.000 & 1.0$\pm$.000 & \textbf{3.0$\pm$.000} \\
     & CPED & 1.9 & 3.0 & Base & - & - & \textbf{.280} \\
    \hline
    \multirow{4}{*}{CFD} & $\mathcal{L}_{t+1}$ & .328$\pm$.004 & .316$\pm$.009 & .303$\pm$.002 & \textbf{.251$\pm$.001} & .324$\pm$.003 & .326$\pm$1E-3 \\
     & $\mathcal{L}_\tau$ & .535$\pm$.012 & .574$\pm$.021 & .787$\pm$.012 & .361$\pm$.003 & .581$\pm$.024 & \textbf{.339$\pm$.035} \\
     & NFE & 14.0$\pm$.000 & 42.1$\pm$1.6 & 59.9$\pm$2.1 & 1.0$\pm$.000 & 1.0$\pm$.000 & \textbf{2.2$\pm$.039} \\
     & CPED & .929 & 3.3 & Base & - & - & \textbf{.081} \\
    \hline
  \end{tabular}
  \caption{Benchmarking on the Diffusion-Reaction, Shallow Water and CFD Turb. datasets in time-informed settings.}
  \label{tab:sota_comparison_pin_1}
\end{table}

\subsection{Direct Auto-regressive Inference}
\label{appendix:direct_autoregressive_inference}
The direct auto-regressive inference protocol provides the model with the total rollout horizon between the initial state and a intermediate state~($\Delta t=t_{k}-t_{0}$). This protocol challenges the model to leverage its learned manifold structure and adaptive step control mechanism to make accurate predictions over a longer time horizon without being informed of the intermediate time steps. The model must effectively navigate the latent manifold and adjust its step size dynamically to maintain accuracy over the extended rollout. This protocol is designed to evaluate the model's ability to generalize to longer time horizons and to test the effectiveness of the learned manifold and the step control mechanism in guiding the solver through complex dynamics.

We call this protocol ``Direct Auto-regressive Inference" because the model is asked to directly predict the state at a future time point without being informed of the intermediate time steps, whose evaluation data is shown in~\Cref{tab:sota_comparison_rollout_1}.

\begin{table}[htpb]
  \centering
  \setlength{\tabcolsep}{5pt}
    \begin{tabular}{|cc|*{4}{c|}}
    \hline
    Dataset & Metric & N-ODE$^{\dagger}$ & CFM$^{\dagger}$ & TM$^{\dagger}$ & CVF \\
    \hline
    \multirow{4}{*}{DR} & $\mathcal{L}_{t+1}$ & 7E-4$\pm$8E-5 & .002$\pm$2E-5 & .001$\pm$7E-5 & \textbf{2E-4$\pm$5E-5} \\
     & $\mathcal{L}_\tau$ & .070$\pm$.003 & .073$\pm$8E-5 & .094$\pm$.002 & \textbf{.016$\pm$.005} \\
     & NFE & \textbf{31.9$\pm$35.9} & 53.0$\pm$19.9 & 202.6$\pm$132.5 & 70.8$\pm$7.5 \\
     & CPED & 6.5 & 12.7 & Base & \textbf{4.5} \\
    \hline
    \multirow{4}{*}{SW} & $\mathcal{L}_{t+1}$ & 4E-4$\pm$2E-5 & .008$\pm$3E-4 & .010$\pm$6E-6 & \textbf{4E-4$\pm$5E-6} \\
     & $\mathcal{L}_\tau$ & \textbf{.008$\pm$5E-4} & .178$\pm$.005 & .545$\pm$.005 & .095$\pm$.012 \\
     & NFE & 115.9$\pm$1.5 & \textbf{26.0$\pm$.000} & 97.0$\pm$5.0 & 51.2$\pm$3.7 \\
     & CPED & 2.2 & \textbf{.730} & Base & 1.2 \\
    \hline
    \multirow{4}{*}{CFD} & $\mathcal{L}_{t+1}$ & .328$\pm$.004 & .316$\pm$.009 & \textbf{.303$\pm$.002} & .326$\pm$1E-3 \\
     & $\mathcal{L}_\tau$ & .535$\pm$.012 & .655$\pm$.004 & .787$\pm$.011 & \textbf{.314$\pm$.022} \\
     & NFE & \textbf{14.0$\pm$.000} & 50.0$\pm$.000 & 459.2$\pm$22.5 & 22.9$\pm$.749 \\
     & CPED & .121 & .821 & Base & \textbf{.105} \\
    \hline
  \end{tabular}
  \caption{Benchmarking on the Diffusion-Reaction, Shallow Water and CFD Turb. datasets in direct auto-regressive settings.}
  \label{tab:sota_comparison_rollout_1}
\end{table}
In the turbulent CFD regime, CVF exhibits a unique self-correcting behavior where the rollout RMSE ($\mathcal{L}\tau$) remains lower than the single-step error ($\mathcal{L}{t+1}$).
This suggests a unique self-correcting behavior of CVF, where it captures the underlying physical manifold of the global flow, and thus corrects its trajectory over time, while baselines may achieve lower single-step error by overfitting to local noise but rapidly drift away from the true dynamics during long-horizon rollouts. This phenomenon highlights the importance of evaluating models not only on single-step accuracy but also on their ability to maintain physical consistency and stability over long-term predictions, especially in complex, turbulent regimes where local noise can be misleading.

\subsection{N-ODE vs. CVF: Discussion on the Performance Discrepancy in SW Dataset}
\label{appendix:n_ode_vs_cvf}
Further investigation into N-ODE's performance across the two settings reveals a critical insight into the nature of its learned dynamics. While its accrued RMSE remains comparable between the settings ($0.008$ vs. $0.008$), the required NFE explodes. We posit that N-ODE is deeply reliant on the microscopic auto-regressive prior injected during training (via the fixed-step RK4 solver). When deprived of this uniform grid during direct auto-regressive inference, N-ODE exhibits a pathological sensitivity to high-frequency stiffness and compounding errors. To compensate and artificially restore its train-time auto-regressive behavior, the adaptive solver is forced to aggressively truncate step sizes, operating in a brute-force micro-stepping regime. Consequently, N-ODE asymptotically purchases its macroscopic accuracy at the cost of severe computational overhead, highlighting the failure of unconstrained models to act as efficient, geometry-preserving surrogates.

The performance of N-ODE on the SW dataset proves a critical baseline capability: when N-ODE inadvertently learns a vector field that permits the restoration of the true dynamics, the adaptive \textit{dopri5} solver is fully capable of achieving high precision by brute-forcing massive computational resources (NFE $\approx 116$).
However, this makes N-ODE's failure on the DR and CFD datasets profoundly revealing. On these datasets, N-ODE exhibits a paradoxical ``false confidence'': despite failing to recover the true dynamics (producing drastically higher $\mathcal{L}_\tau$ errors), it chooses a computationally cheaper path, with its NFE dropping significantly to $31.9$ and $14.0$, respectively. It neither achieves high accuracy nor attempts to consume the resources required to do so.

This phenomenon exposes the fatal flaw of unconstrained continuous models. The adaptive solver determines its step size based on local truncation errors. Its reluctance to increase NFE in DR and CFD implies that the local errors are strictly within the rigorous $10^{-3}$ tolerance. We suspect that N-ODE provides the solver with a geometric landscape that is fundamentally misaligned macroscopically, yet smoothly deceptive locally. The \textit{dopri5} solver confidently executes large, efficient steps along this deceptive path, yet remains fundamentally powerless to salvage the trajectory because the underlying vector field itself is systematically pointing in the wrong direction.

Synthesizing the empirical results across both testing paradigms exposes a crucial limitation in existing frameworks: a sophisticated numerical solver cannot rescue a fundamentally biased vector field, nor can a rigid solver navigate complex dynamics without external temporal meta-information.

\subsection{Structural Root Cause: Frequency Bias in Near-Conservative Systems}
\label{appendix:sw_frequency_bias}

The preceding analysis revealed that N-ODE purchases its SW accuracy through massive computational overhead, while CVF achieves competitive accuracy at a fraction of the NFE in most settings. However, the SW dataset under the Direct Auto-regressive setting reveals a distinct failure mode of CVF itself---not one of incorrect dynamics, but of an \textit{undetectable} frequency bias.

\paragraph{Why Unidirectional SR is Blind to Frequency Bias.}
To understand this phenomenon, we must revisit the geometric structure of the Shallow Water equations. This system is near-conservative: its underlying dynamics consist primarily of wave propagation and advection, which can be characterized by persistent rotational modes in phase space. Unlike dissipative systems (e.g., Diffusion-Reaction), where local errors are exponentially absorbed by the attractor geometry, in near-conservative systems, any sustained frequency bias $\varepsilon$ induces a linearly growing rollout error:
\begin{equation}
  \|\hat{s}(t) - s(t)\|_2 \sim \varepsilon \cdot t + \mathcal{O}(t^2)
\end{equation}
The critical insight is that \textbf{any rotation satisfies the forward semi-group identity}. Formally, for a system evolving at a slightly biased frequency $\hat{\omega} \neq \omega^{\star}$, the flow map $\Phi_{\hat{\omega}, \Delta t}$ still satisfies:
\begin{equation}
  \Phi_{\hat{\omega}, \Delta t_2} \circ \Phi_{\hat{\omega}, \Delta t_1}
  = \Phi_{\hat{\omega}, \Delta t_1 + \Delta t_2}
\end{equation}
because the rotational structure is internally consistent at the biased frequency. As a consequence, the unidirectional Symmetry Rupture evaluates to approximately zero---$\mathcal{R}_3 \approx 0$---even when the learned frequency is incorrect ($\hat{\omega} \neq \omega^\star$). This is a fundamental limitation of the forward semi-group constraint: it cannot distinguish between a correct rotation and a self-consistent but biased rotation.

\paragraph{The GCS Response: Step Size Reduction as a Symptom.}
Because the network has internalized a self-consistent but biased rotational mode, the NRE $\hat{\mathcal{R}}_3$ does not provide a reliable signal of the frequency mismatch. Instead, the GCS solver observes the growing discrepancy between the predicted and true trajectories as a high apparent curvature in the learned phase space. Interpreting this curvature as a numerical truncation error---rather than the structural frequency bias it truly is---the solver responds by repeatedly cutting the step size. This produces the elevated NFE (51.2) observed in the Direct Auto-regressive setting on the SW dataset:
\begin{equation}
  \text{(frequency bias)} \xrightarrow{\text{misdiagnosed as}}
  \text{(numerical stiffness)} \xrightarrow{\text{GCS response}}
  \text{(step size reduction)} \xrightarrow{\text{result}}
  \text{high NFE}
\end{equation}
This chain of misdiagnosis is invisible to unidirectional SR, which continues to report low $\mathcal{R}_3$ throughout the rollout. The GCS is not ``wrong'' in its step-cutting behavior---it is faithfully responding to the representation mismatch it \textit{can} detect (curvature-induced inconsistency). The root issue is that the training objective does not provide sufficient information to resolve the frequency degeneracy inherent to rotational dynamics.

\paragraph{Bidirectional SR as a Resolution.}
We hypothesize that \textbf{bidirectional SR} could break this rotational symmetry and resolve the frequency bias. The bidirectional Rupture loss $\mathcal{R}^{\text{bi}}_k$~(\Cref{eq:bi_rupture_loss}) enforces that the reverse-time flow is consistent with the forward-time flow. For a biased rotational mode at frequency $\hat{\omega}$, the forward composition is self-consistent:
\begin{equation}
  \mathcal{R}_3(\hat{\omega}) \approx 0
\end{equation}
But the bidirectional consistency requires that the backward flow $\Phi_{-\Delta t}$ also composes correctly. A biased forward rotation at $\hat{\omega}$ generates a backward rotation at $-\hat{\omega}$, which, when composed with the forward rotation, produces a net phase shift of $(\hat{\omega} - \omega^{\star}) \cdot \Delta t$ relative to the true trajectory. This net phase shift would be detected as a non-zero bidirectional Rupture, providing the training signal needed to correct the frequency estimate.

More formally, for a near-conservative system where $\nabla \cdot f \approx 0$, the flow forms a one-parameter group~(\Cref{appendix:diffeomorphisms}). The group inverse property $\Phi_{-t} = \Phi_t^{-1}$ means that the bidirectional constraint directly enforces that the forward and backward flows are \textit{mutually inverse} at the correct frequency. Any systematic frequency bias breaks this mutual inverse relationship, making the bidirectional SR a sensitive detector of frequency degeneracy. We leave this extension to future work, noting that our ablation studies~(\Cref{subsec:ablation_semi_group}) already confirm the benefit of bidirectional constraints for conservative systems, and the Wave Equation results provide a proof of concept for this regime.

\paragraph{Connection to Koopman Aliasing Theory.}
This frequency bias phenomenon connects directly to the Koopman aliasing analysis of~\Cref{appendix:koopman_aliasing}. In the Koopman picture, the aliased branch at $m \neq 0$ corresponds precisely to an incorrect but self-consistent rotational mode. For 2D fluid systems, the 2D spatial geometry acts as a ``lock'' that eliminates aliased branches by making it geometrically prohibitive to synthesize globally consistent fake vorticity fields~(\Cref{sec:2d_aliasing_lock}). However, the SW dataset, while 2D, is near-conservative rather than dissipative. Its wave propagation dynamics have a relatively low spatial vorticity \textit{magnitude}, meaning that a small frequency bias requires only a modest fake vorticity contribution---insufficient to trigger the catastrophic spatial tearing that eliminates aliased branches in strongly turbulent systems like Navier-Stokes. This explains why the spatial aliasing lock is effective for turbulent CFD but insufficient for the near-conservative SW system.

This nuanced interplay between system conservativity, spatial dimensionality, and the directionality of the SR constraint reveals an important boundary condition for our framework: the efficacy of unidirectional SR as an aliasing suppressor degrades gracefully as systems approach the conservative limit and as the characteristic vorticity magnitude decreases. Bidirectional SR, or alternative frequency-aware constraints, represent the natural next step for handling this regime.

\section{Ablation Details}
\subsection{Downsampling Rate}
\label{appendix:downsampling}

\paragraph{Time-Informed Setting} In the Time-Informed setting, the model is trained with a fixed time step $k$ that is uniformly applied across the training data.
\begin{equation}
  \begin{split}
    \text{Uniform Downsampling}: & k<0, g(x)=-0.000822*x^2 + 0.000558*x^1 + 1.06e-02
  \end{split}
\end{equation}
\paragraph{Direct Auto-regressive Setting} In the Direct Auto-regressive setting, the model is trained with a random time step $k$ that is sampled from a uniform distribution for each training instance.
\begin{equation}
  \begin{split}
    \text{Random Downsampling}:  & k>0, f(x)=e^{0.015229*log(x)^2 + -0.285431*log(x)^1 + -3.506706} \\
    \text{Uniform Downsampling}: & k<0, g(x)=e^{0.118474*log(x)^2 + -0.635297*log(x)^1 + -4.129588}
  \end{split}
\end{equation}
The evaluation results~(\Cref{fig:ablation_dynamic_downsampling_rollout}) and the trends show a co-reduction in both the NFE and the RMSE as the downsampling rate increases, which indicates that the model learns a more accurate manifold with a wider range of time steps, thus leading to better generalization performance in the unobserved temporal regime.
However, a marginal reduction in RMSE is observed as the downsampling rate increases-- as the downsampling rate ($\|k\|$) increases exponentially, reduction in RMSE is not significant after a certain point; that is, the performance gain from learning with a wider range of time steps is marginal after a certain point.

While our benchmarking trains the model with a fixed downsampling rate of $k=-2$ (i.e., training on every other frame), the ablation results suggest that training with a wider range of time steps (e.g., $k=-4$ or $k=-16$) can further improve the model's generalization performance in the unobserved temporal regime, albeit with diminishing returns in terms of RMSE reduction.
\subsection{Tangent Supervision}
\label{sec:tangent_supervision_ablation}
\label{appendix:tangent_velocity_approximation}

\paragraph{Spline Approximation of Tangent Velocity}
By training the model to match spline tangent velocity, an error bound between the learned and the true tangent velocity $\mathcal{L} = \| v_{t} - \psi_{\theta}(s,0) \|_{\infty}$ can be established based on the approximation error of the spline interpolation.
According to \citet{hall1976optimal,deboor1978practical}, the error bound of first-order derivative from $k$-order spline interpolation is given as:
\begin{equation}
  \begin{split}
    \| v_{t} - v_{s} \|_{\infty} & \leq \epsilon \| s^{(k+1)} \|_{\infty} \Delta t^{k}
  \end{split}
\end{equation}
where $\epsilon$ is a constant depending on the spline basis functions, $s^{(k+1)}$ is the fourth-order derivative of the true trajectory, and $\Delta t$ is the maximum time step between consecutive observations.
Therefore, the expected error of learned tangent velocity $\mathcal{L}$ can be bounded as:
\begin{equation}
  \begin{split}
    \mathbb{E}\left[\mathcal{L}\right] & \leq \| v_{t} - v_{s} \|_{\infty} + \mathbb{E}\left[\| v_{s} - \psi_{\theta}(s,0) \|_{\infty}\right]                            \\
                                       & \leq \epsilon \| s^{(k+1)} \|_{\infty} \Delta t^{k} +                     \mathbb{E}\left[RMSE(v_{s},\psi_{\theta}(s,0))\right]
  \end{split}
\end{equation}
which is contrasting to learning a pure neural ODE without tangent velocity supervision, where no such error bound can be established due to the infinite possible tangent velocity fields that are consistent with the discrete-time observations.

\Cref{fig:tangent_supervision_ablation} evaluates the effect of tangent supervision in the Time-Informed and Direct Auto-regressive settings.
\citet{hou2025cfo} proposed learning first-order quintic Hermite derivative to ensure smoothness of acceleration field on the knots such that the learned velocity field aligns better with the true underlying physical dynamics.
Therefore, the model is tested with combining the cubic spline-based tangent supervision, quintic Hermite derivative supervision, and without tangent supervision.
\begin{equation}
  \begin{split}
    \mathcal{J}_{tangent} & = \mathbb{E}\bigg[ \|\psi_{\theta}(s_{t},0)-v_{t} \|_{2}^{2} + \| \psi_{\theta}(s_{t}, \Delta t) - v_{\Delta t} \|_{2}^{2}+ \|\mathcal{R}_{3}(s_{t},\Delta t)\|_{2}^{2}\bigg]
  \end{split}
\end{equation}
\begin{figure}[htpb]
  \centering
  \subfloat[Time-Informed]{
    \includegraphics[width=.5\linewidth]{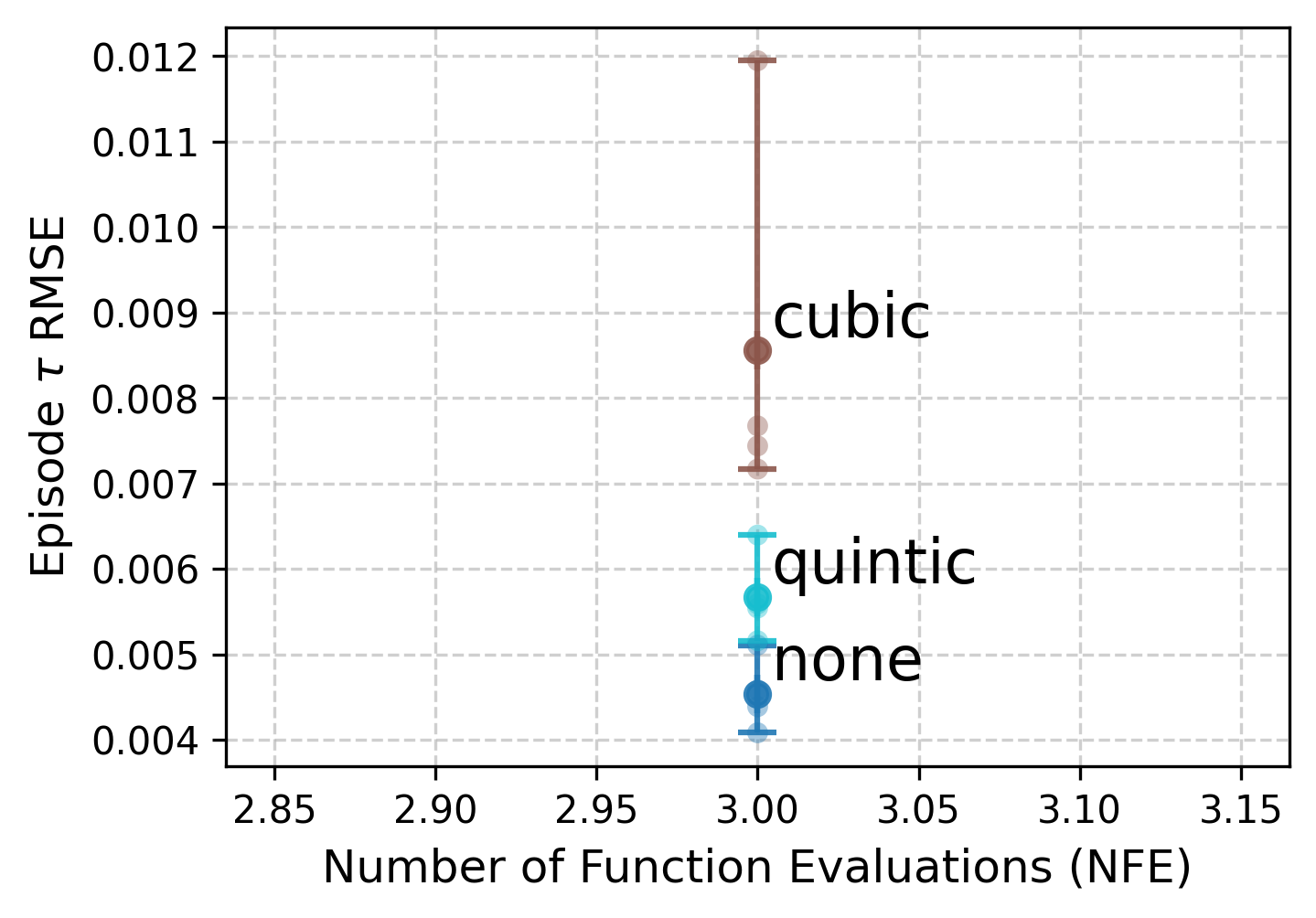}}
  \subfloat[Direct Auto-regressive]{\includegraphics[width=.5\linewidth]{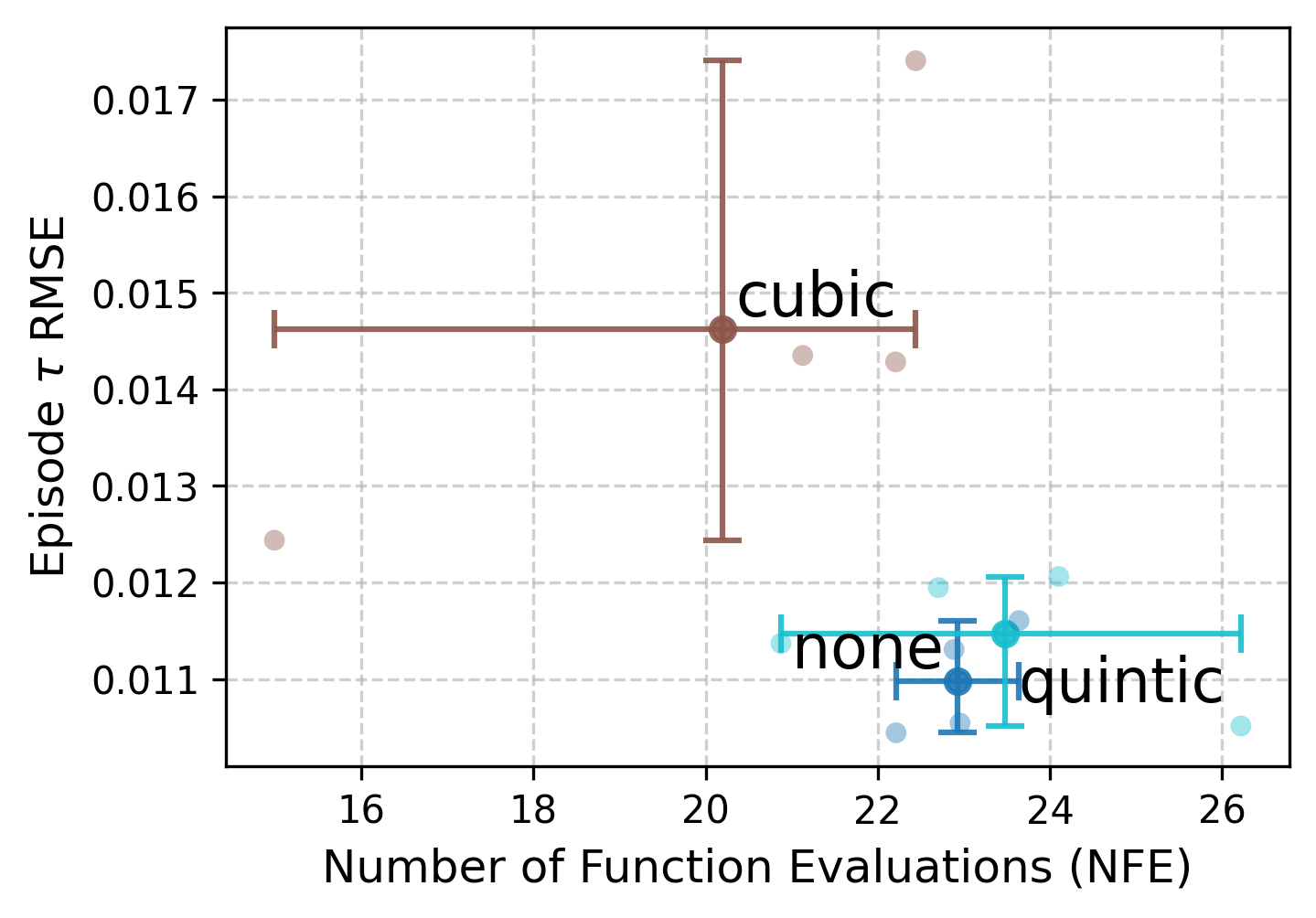}}
  \caption{Ablation Study on the Effect of Spline-Based Tangent Supervision.}
  \label{fig:tangent_supervision_ablation}
\end{figure}
The results in both settings show consistent trends that no tangent supervision gains the best performance, and the cubic spline-based tangent supervision performs worse than the quintic Hermite derivative supervision.
This reflects the biased nature of spline and Hermite supervision, which may not align perfectly with the semi-group or Lie group structure of the underlying physical dynamics, thus leading to worse generalization performance than learning from raw velocity data.
However, the performance difference between the quintic Hermite derivative supervision and no tangent supervision is not significant, which showcases again the importance of correct inductive bias in learning physical dynamics - learning with velocity smoothness bias (cubic spline) is worse than learning with acceleration smoothness bias (quintic Hermite), and learning with no bias is better than learning with the wrong bias.
The benefits of tangent supervision might be more significant in avoiding plateauing in the early stage of training or failing into plain solutions as linear interpolation, which is not reflected in the final performance after convergence.
Therefore, it is treated as an optional technique that can be applied when the training process is unstable or prone to linear interpolation degradation, and it can be safely removed when the training process is smooth and converges well.

\section{Inference Details}

\subsection{Visual Analysis of Failure Modes}
\begin{figure}[htpb]
  \centering
  \subfloat[Prediction]{\includegraphics[width=1\linewidth]{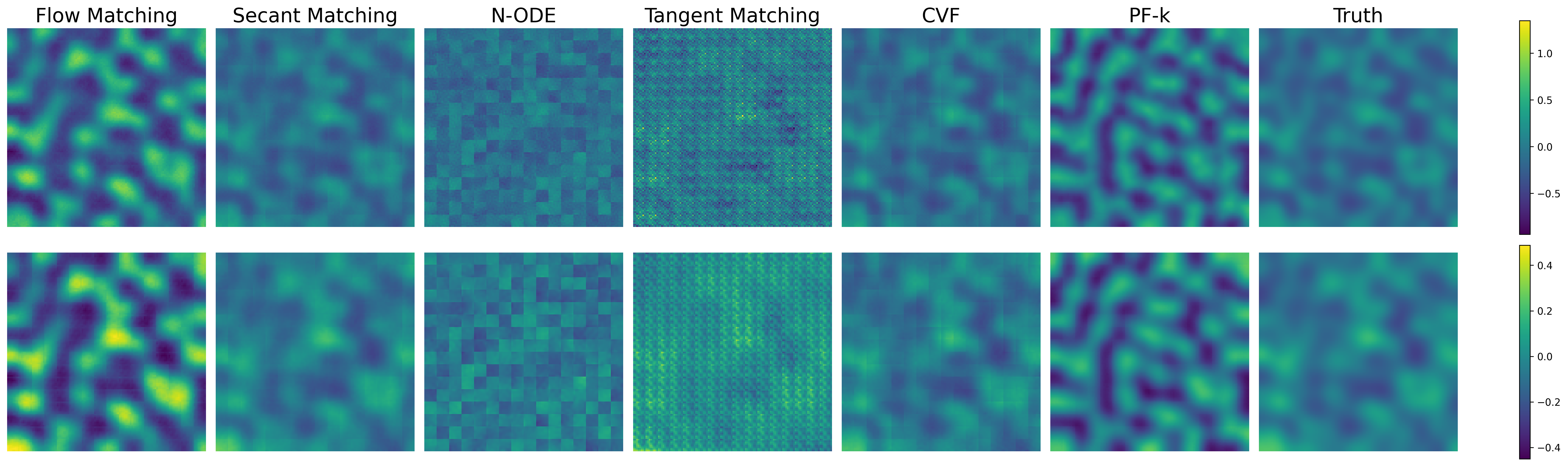}}

  \subfloat[Mean Absolute Error]{\includegraphics[width=1\linewidth]{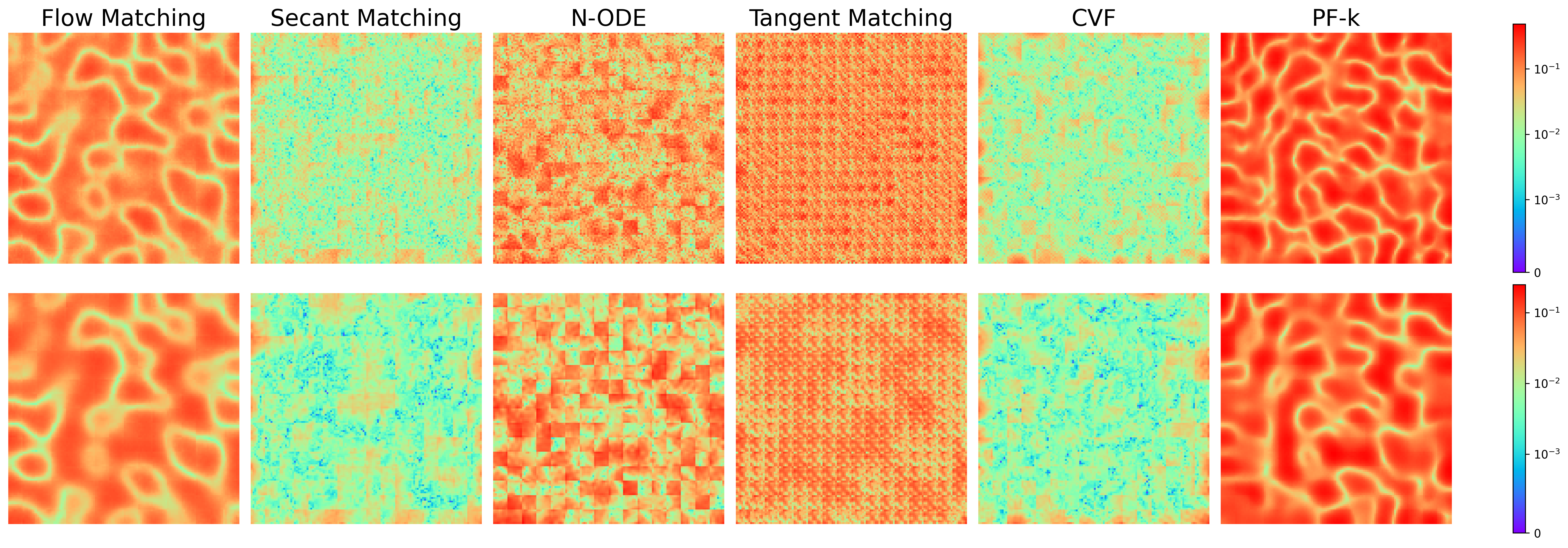}}
  \caption{Example Final Frame of Rollout Predictions on the DR Dataset.}
  \label{fig:final_frame_rollout_dr}
\end{figure}
Figures~\ref{fig:final_frame_rollout_dr} and~\ref{fig:final_frame_rollout_sw} visualize the examples of final frame prediction and the corresponding spatial error distribution (pixel-wise mean absolute error) from the two datasets, respectively.
The qualitative comparison reveals distinct failure modes in the baseline methods that our CVF framework effectively overcomes.
\paragraph{Visualizing Error Dynamics: The ``Blur" vs. The ``Noise"}
The error heatmaps (Figure \ref{fig:error_maps_sw}) reveal distinct failure modes across different methodologies, offering a qualitative explanation for the quantitative gap in Table \ref{tab:sota_comparison_rollout_1}.

\textbf{1. The ``Temporal Smearing" of Flow Matching:}
As observed in the Shallow Water error maps (Figure \ref{fig:error_maps_sw}, column 1), Flow Matching (CFM) exhibits a characteristic \textit{ring-like error pattern}.
While CFM captures the global topology, it suffers from severe spectral bias, effectively acting as a low-pass filter.
Physically, this manifests as ``temporal smearing": instead of resolving the sharp shock front at $t+1$, the model predicts a diffused average of the wave's propagation path.
This suggests that without the explicit tangent anchor, the continuous probability flow takes a ``shortcut" through the manifold curvature, underestimating the stiffness of the wavefront and resulting in a blurred, low-amplitude prediction that fails to conserve energy (amplitude decay).
\begin{figure}[htpb]
  \centering
  \subfloat[Prediction]{\includegraphics[width=1\linewidth]{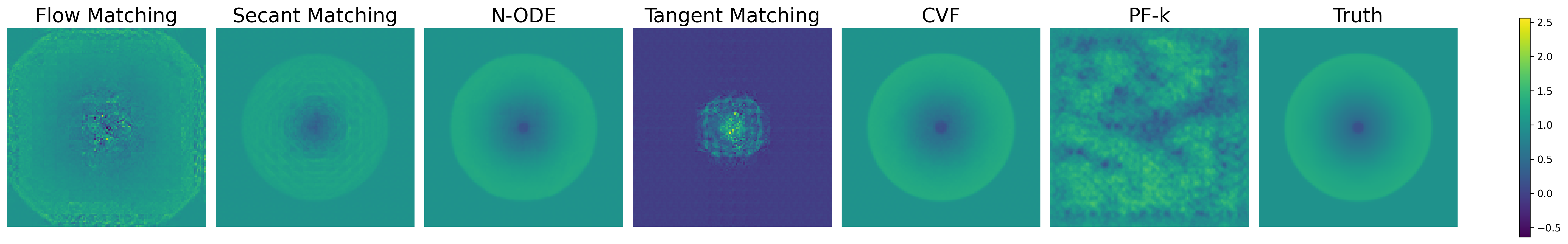}}

  \subfloat[Mean Absolute Error]{\includegraphics[width=1\linewidth]{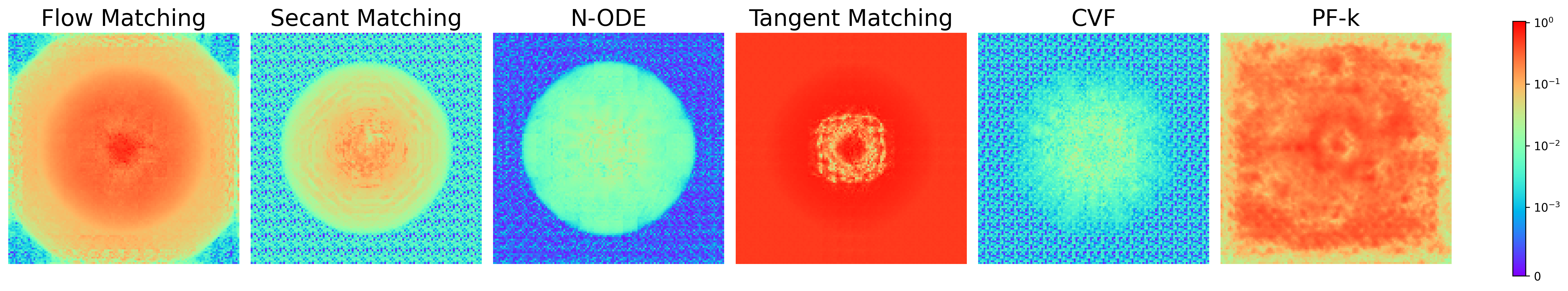}\label{fig:error_maps_sw}}
  \caption{Example Final Frame of Rollout Predictions on the SW Dataset.}
  \label{fig:final_frame_rollout_sw}
\end{figure}

\textbf{2. The Unbounded Divergence of Unconstrained Extrapolation (PF-$k$):}
In stark contrast, the Latent Extrapolation baseline (PF-$k$) demonstrates the diverging consequences of removing the bidirectional topological constraint.
As seen in the \Cref{fig:error_maps_sw} column PF-$k$, the error pattern transforms from structural mismatch into unstructured high-frequency noise, which manifests as rapid collapsing of the wave structure just after half of the inference duration~(\Cref{fig:multi_frame_lek}).
This confirms that purely extrapolating the latent vector field is an unstable process; strictly dissipative or dispersive errors accumulate exponentially ($ \epsilon_{t+k} \propto e^{\lambda k} \epsilon_0 $).
The rapid disintegration of the PF-$k$ predictions validates that our Symmetry Condition is not merely an auxiliary loss, but a necessary condition to bound the accumulation of error on the latent manifold.

\textbf{3. The Precision of CVF:}
CVF (Figure \ref{fig:error_maps_sw}, column 5) uniquely avoids both extremes.
The error map is uniformly low (deep blue), with no visible ring artifacts (implying correct phase velocity) and no noise accumulation (implying topological stability).
By anchoring the local derivative (Tangent) while enforcing macroscopic reversibility (Secant), CVF successfully resolves the sharp wavefront of the shallow water system without the diffusive blurring of CFM or the chaotic divergence of PF-$k$.

\begin{figure}[htpb]
  \centering
  \subfloat[PF-$k$ 0.0, 0.2, 0.4, 0.6, 0.8s]{
    \includegraphics[width=.195\linewidth]{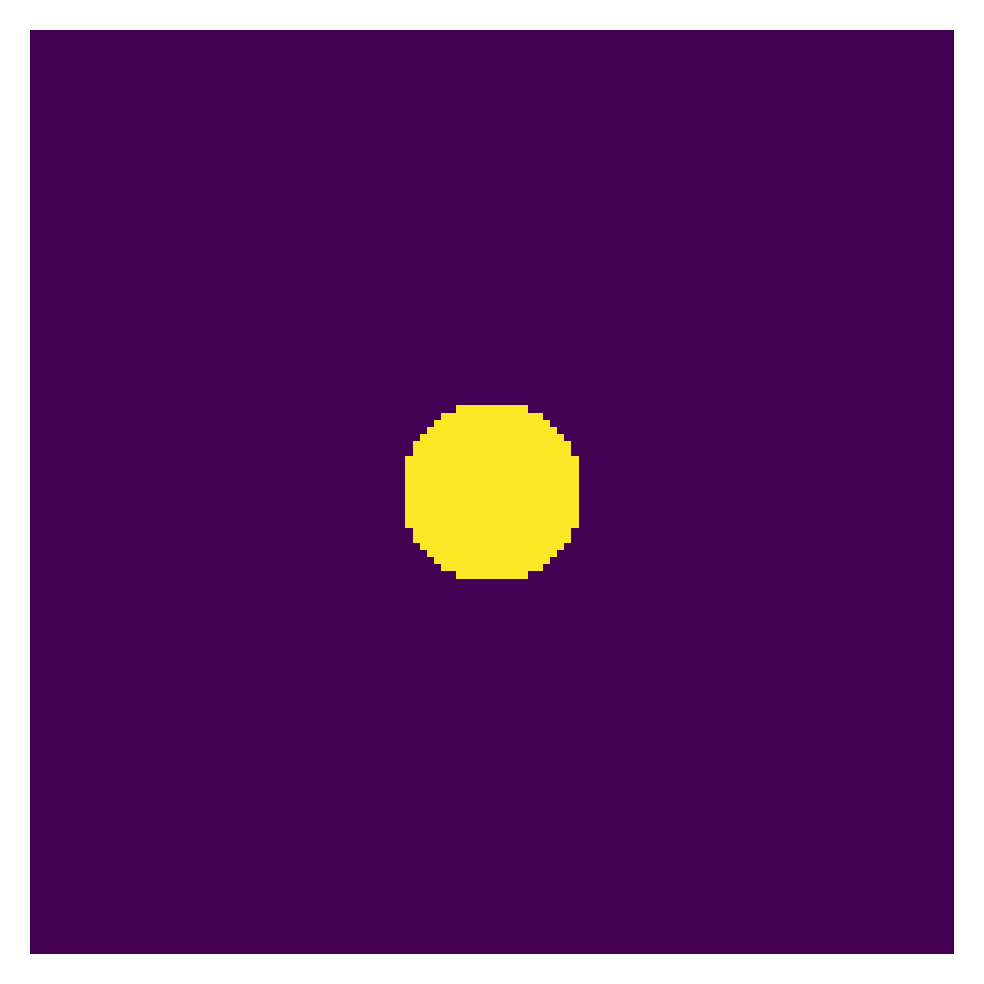}
    \includegraphics[width=.195\linewidth]{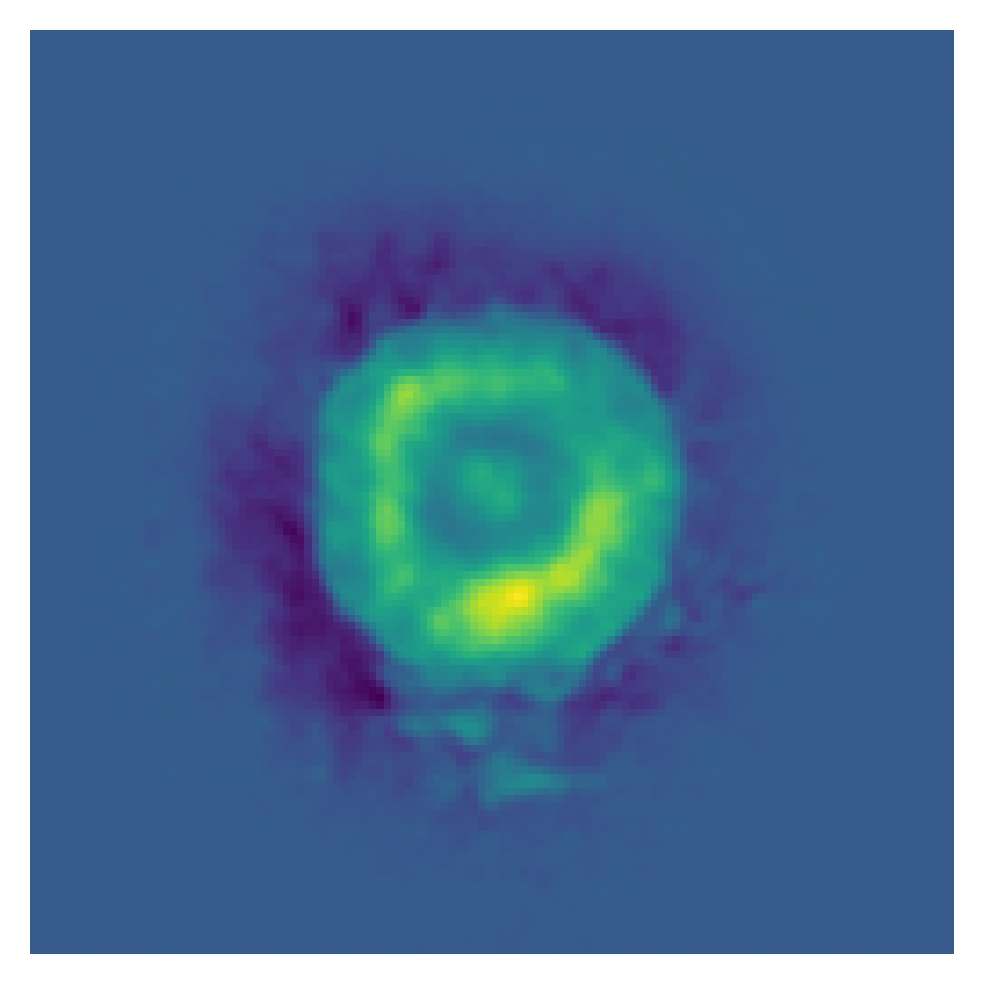}
    \includegraphics[width=.195\linewidth]{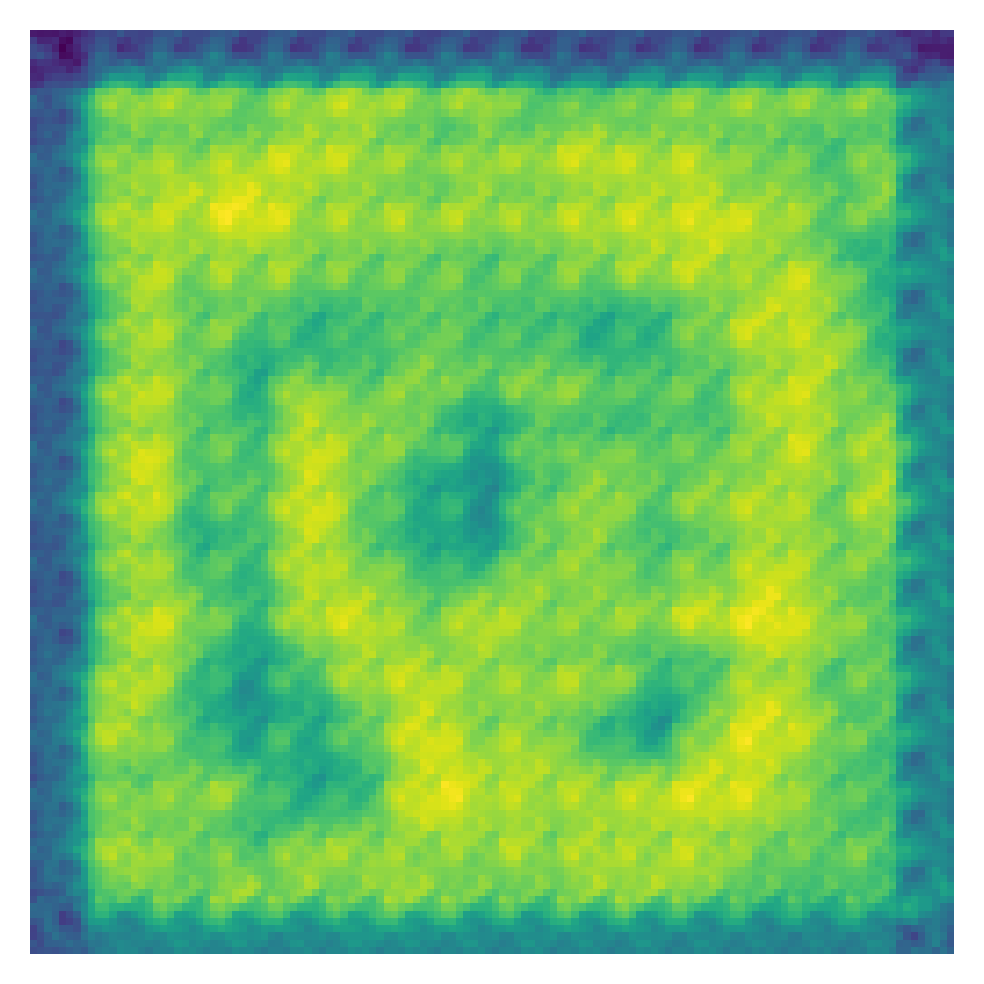}
    \includegraphics[width=.195\linewidth]{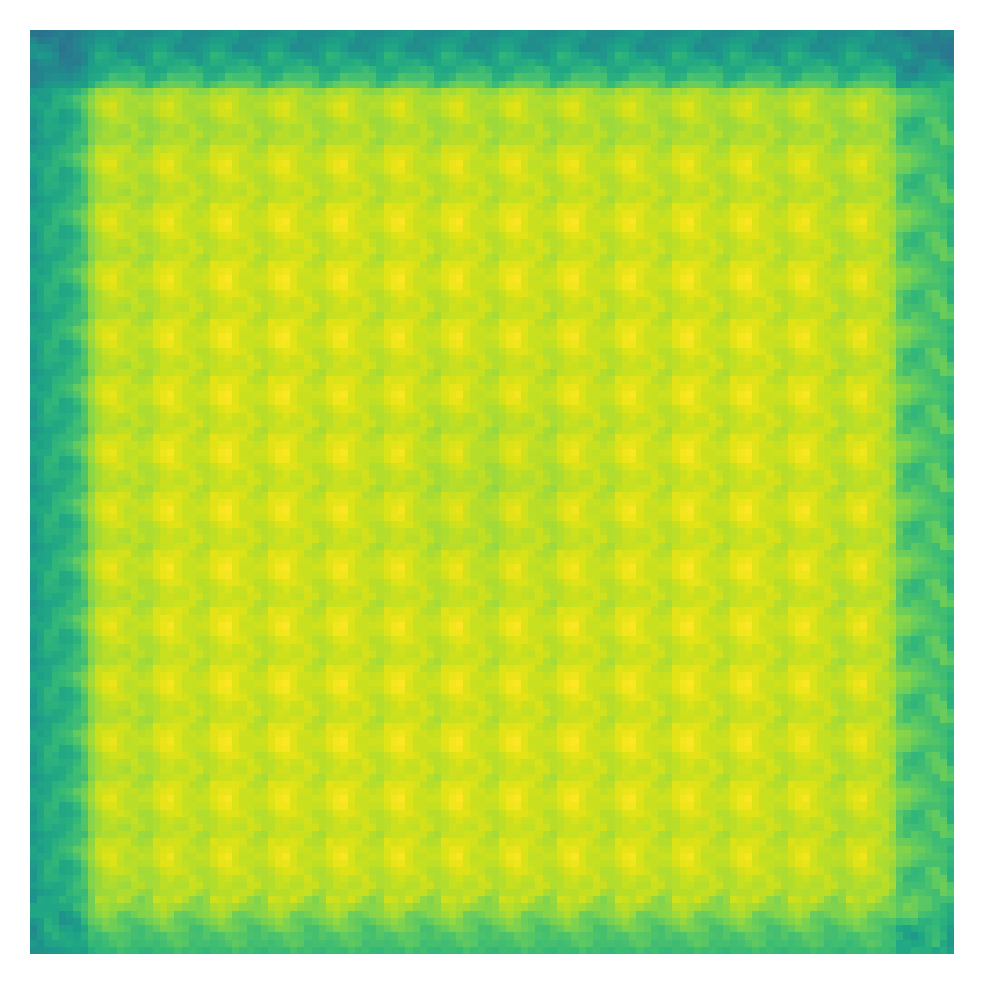}
    \includegraphics[width=.195\linewidth]{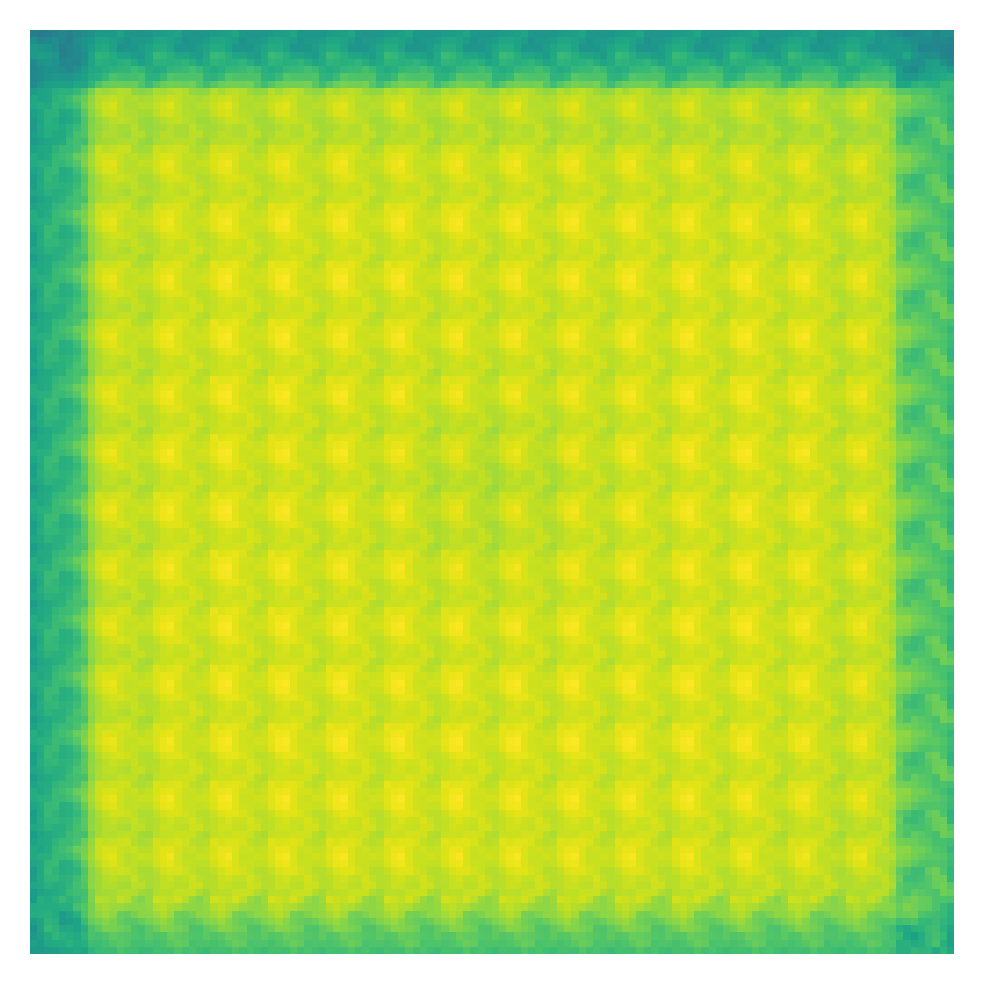}
    \label{fig:multi_frame_lek}
  }

  \subfloat[Flow Matching 0.0, 0.2, 0.4, 0.6, 0.8s]{
    \includegraphics[width=.195\linewidth]{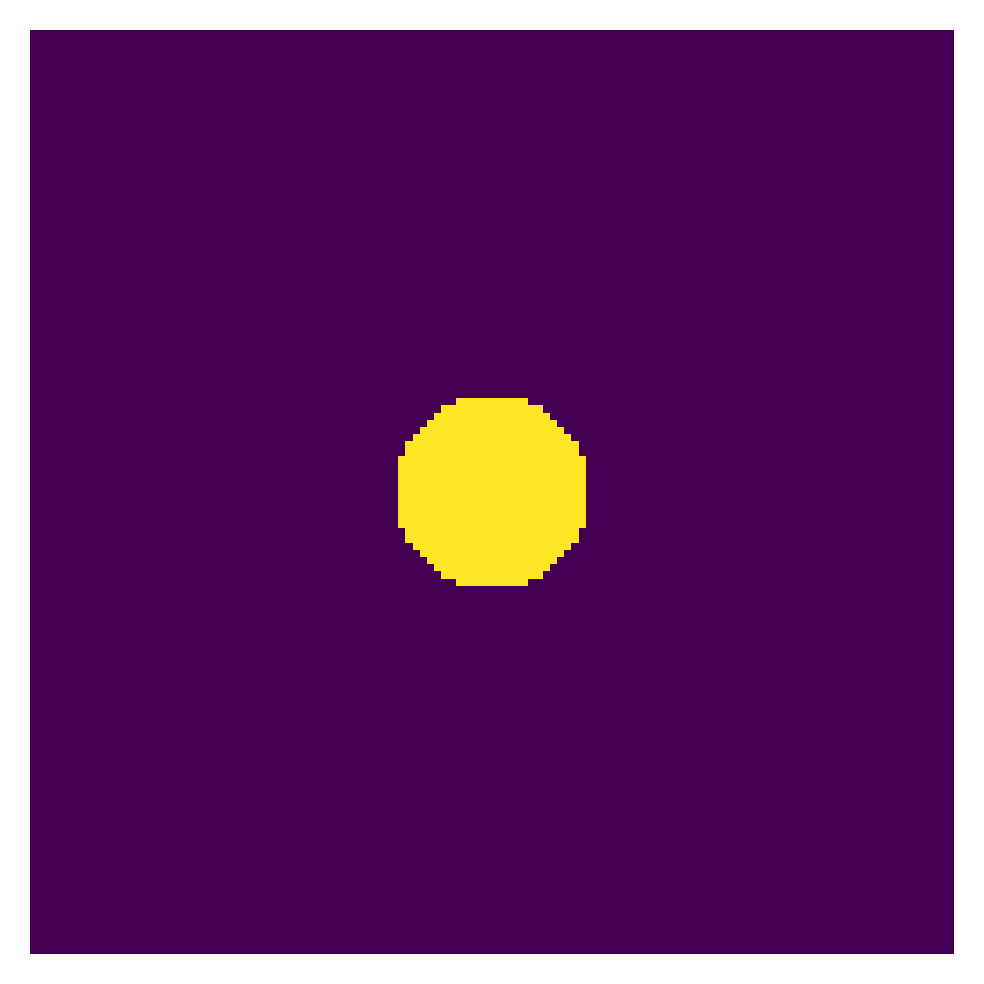}
    \includegraphics[width=.195\linewidth]{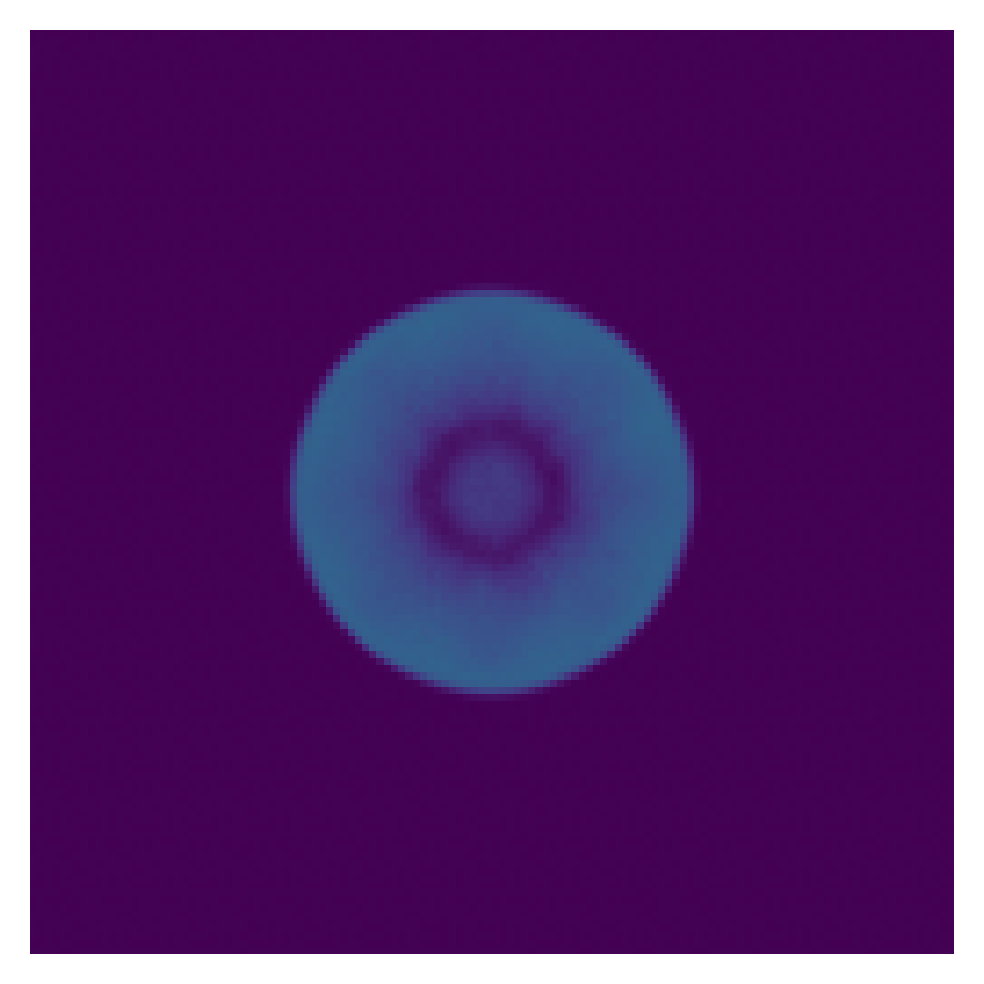}
    \includegraphics[width=.195\linewidth]{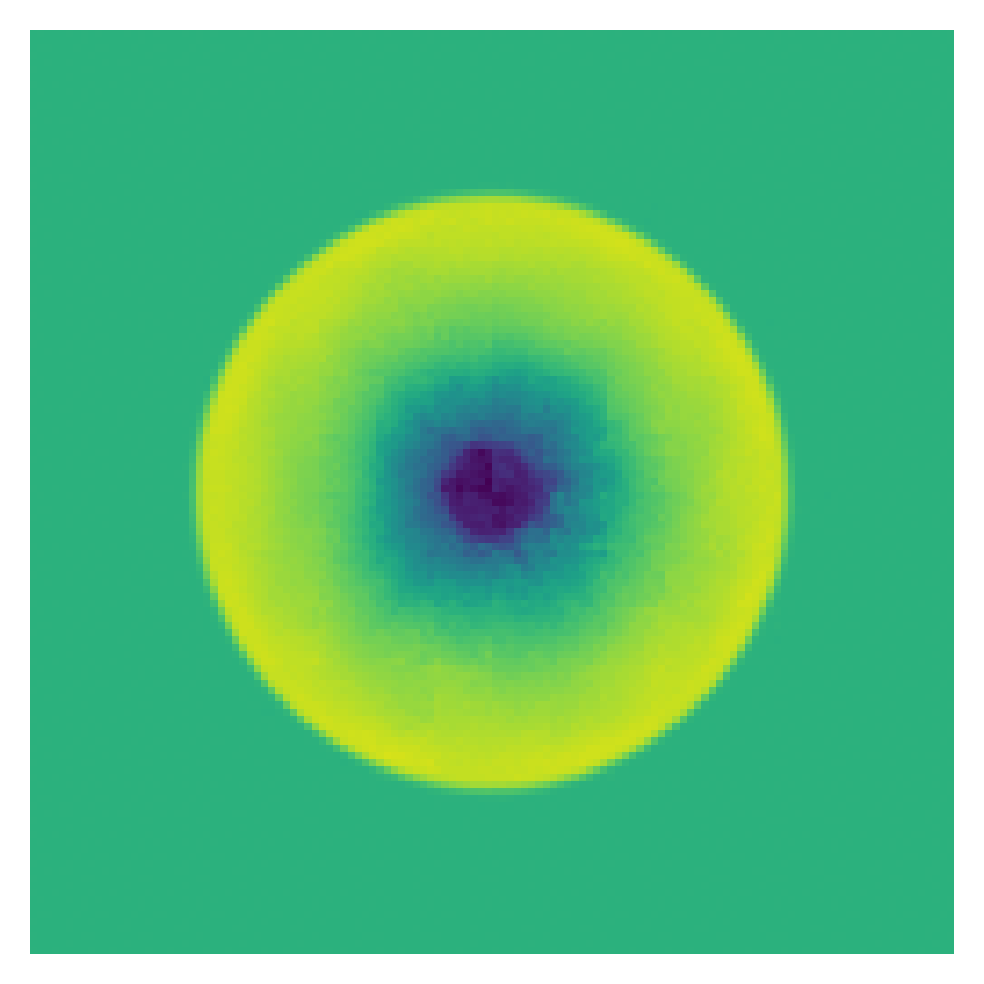}
    \includegraphics[width=.195\linewidth]{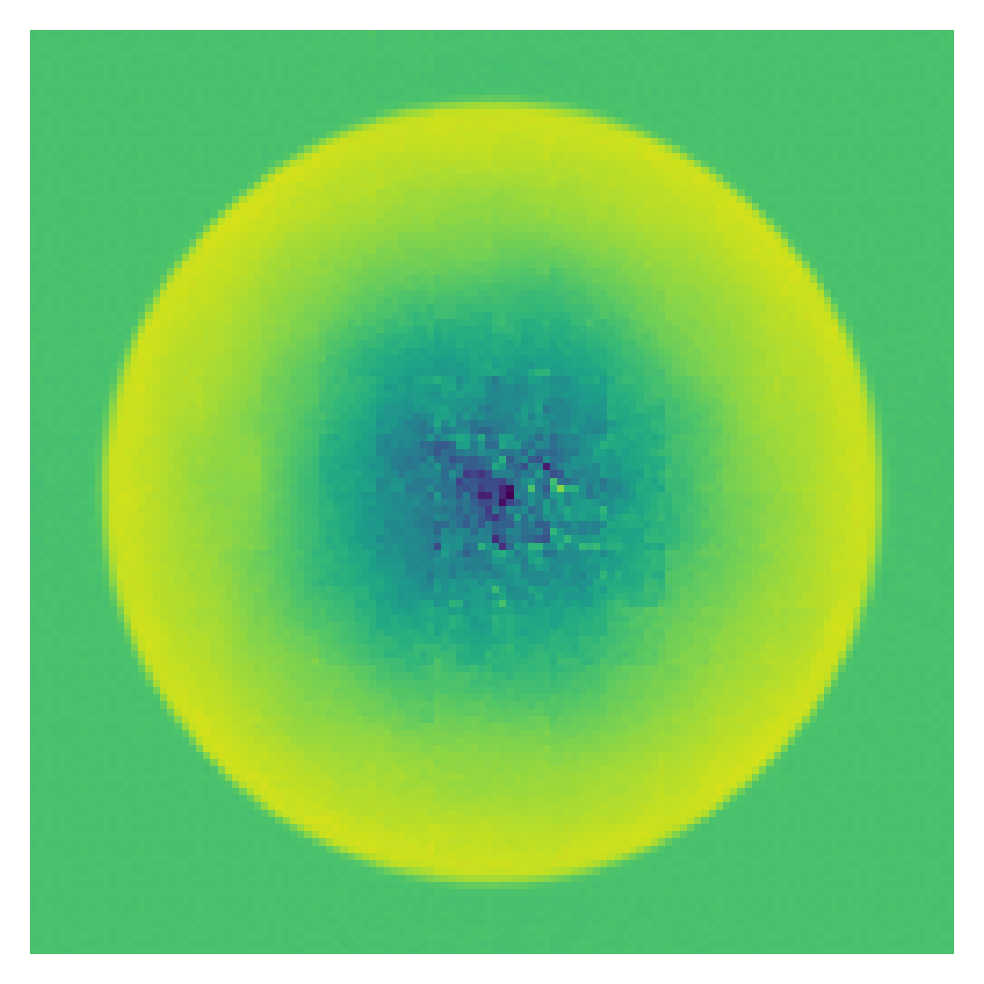}
    \includegraphics[width=.195\linewidth]{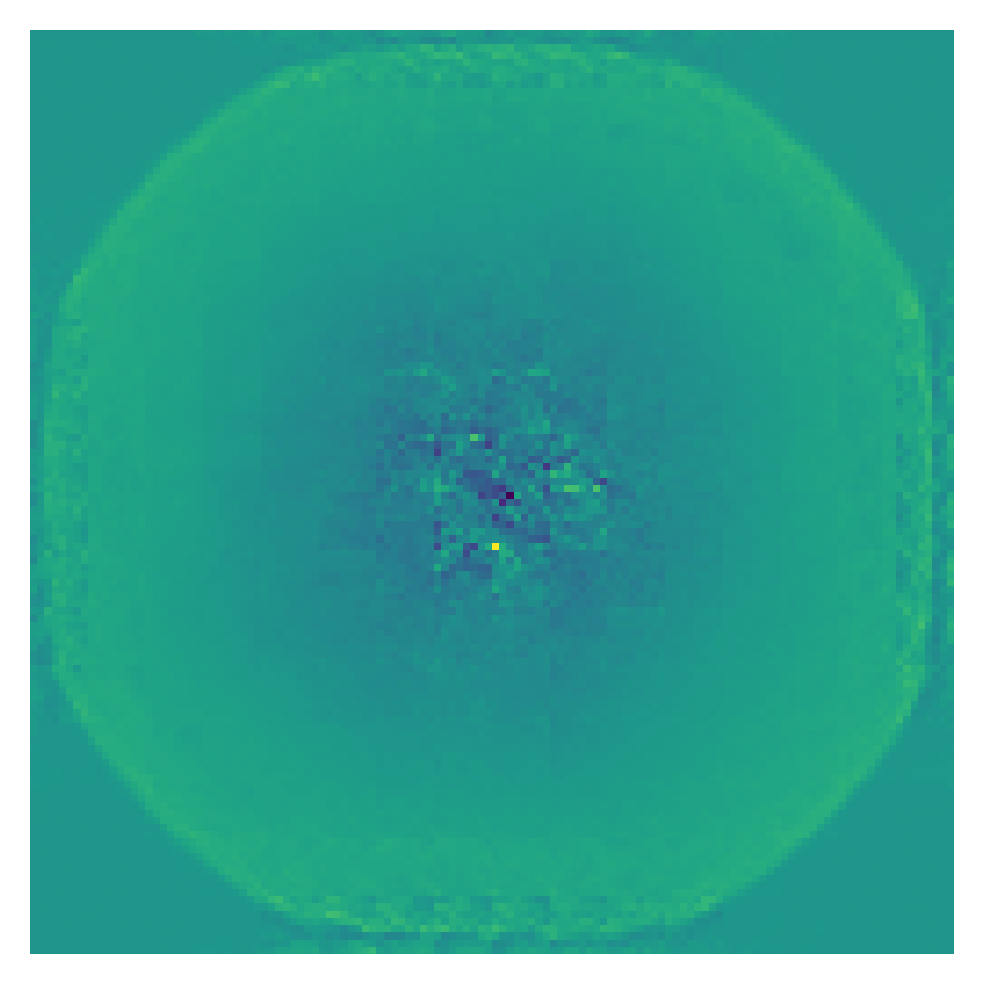}
  }
  \caption{The Rollout Predictions of PF-$k$ and CFM on the SW Dataset, showing significant amplitude decay and wave distortion over time.}
  \label{fig:rollout_predictions_lek_sw}
\end{figure}
\begin{figure}[htpb]
  \centering
  \includegraphics[width=.6\linewidth]{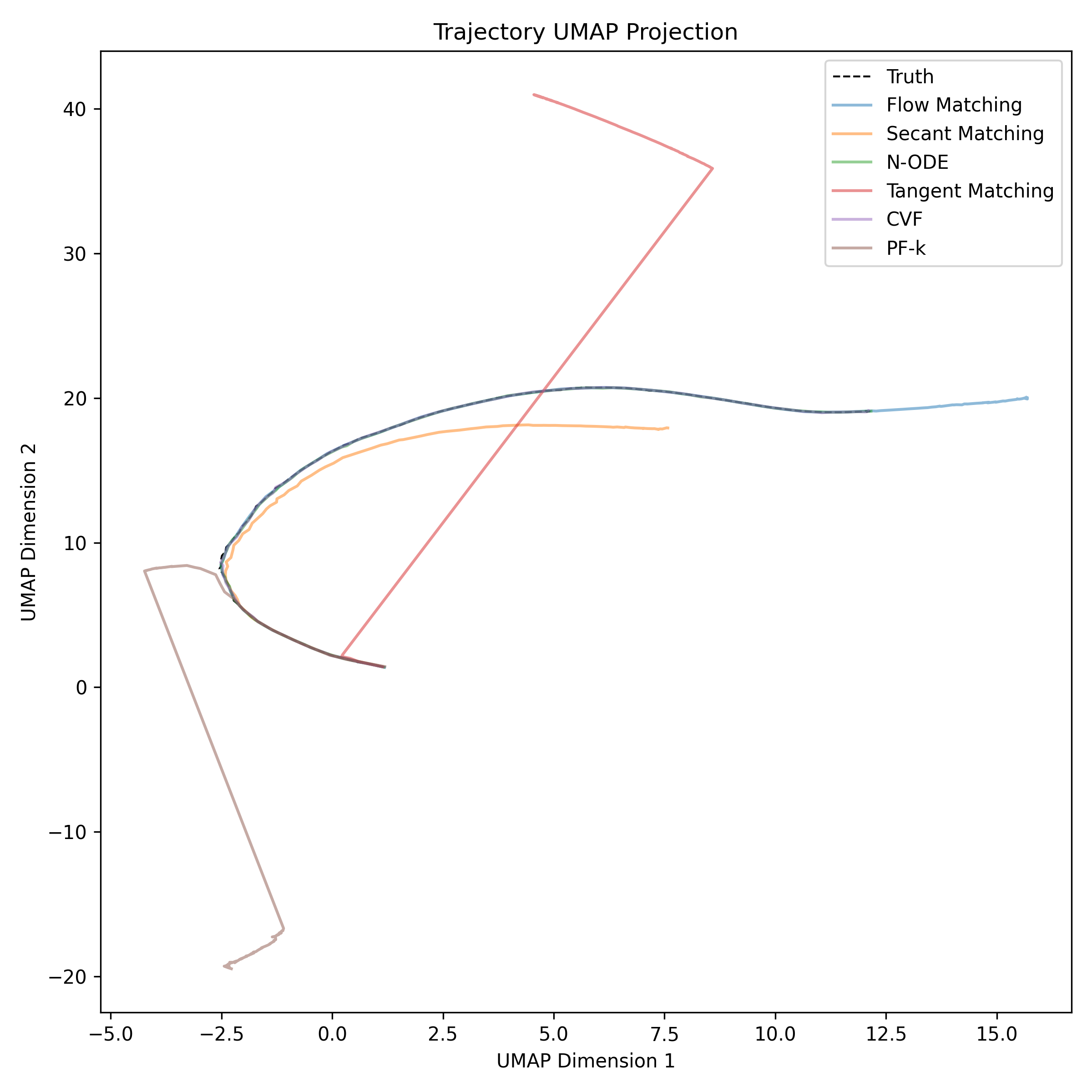}
  \caption{Comparison of Example Rollout Trajectory of the Benchmarking Methods on the SW Dataset on 2D U-MAP Visualization, showing significant deviation from the true trajectory.}
  \label{fig:rollout_trajectory_sw}
\end{figure}

\end{document}